\documentclass{article}
\usepackage{iclr2027_conference,times}
\iclrfinalcopy
\usepackage[utf8]{inputenc}
\usepackage[T1]{fontenc}
\usepackage{amsmath,amssymb,amsthm}
\usepackage{graphicx}
\usepackage{pgf}

\usepackage{booktabs}
\usepackage{enumitem}
\usepackage{microtype}
\usepackage{longtable}
\usepackage{placeins}
\usepackage{etoolbox}
\usepackage{array}
\usepackage{hyperref}
\makeatletter
\patchcmd{\LT@output}{\vss}
  {\vskip\z@\@plus1fil\@minus\normalbaselineskip}{}{}
\patchcmd{\LT@output}{\vss}
  {\vskip\z@\@plus1fil\@minus\normalbaselineskip}{}{}
\makeatother

\ifdefined\pdfsuppressptexinfo \pdfsuppressptexinfo=-1 \fi
\ifdefined\pdftrailerid \pdftrailerid{} \fi
\hypersetup{hidelinks, pdfauthor={Zhiyu Zhang, Yupeng Li}, pdftitle={Tracking States or Tracking Cosets? An Algebraic Account of Learned State Tracking},
            pdfsubject={}, pdfkeywords={}, pdfcreator={}, pdfproducer={}}

\newtheorem{lemma}{Lemma}
\newtheorem{theorem}{Theorem}
\newtheorem{proposition}{Proposition}
\newtheorem{corollary}{Corollary}

\title{Tracking States or Tracking Cosets? An\\
Algebraic Account of Learned State Tracking}
\author{Zhiyu Zhang \\
Conflux Labs Ltd \\
\texttt{zhiyu\_zhang1@alumni.brown.edu} \\
\And
Yupeng Li \\
Michigan State University \\
\texttt{yupengli@msu.edu}}
\date{}

\begin{document}
\maketitle
\lhead{Preprint.}
\suppressfloats[t]
\raggedbottom

\begin{abstract}
State tracking requires composing a sequence of updates, 
but accuracy alone does not reveal what a model has learned. 
We study neural networks trained to predict the running product 
of group elements. We identify quotient solutions in Transformers, 
where models recover the quotient class while predicting 
nearly uniformly among its members. The reciprocal of class size 
predicts partial accuracy without a fitted parameter, extending 
parity-based accounts to non-parity quotients. Our baseline 
Transformers’ predictions change little under prefix reordering 
beyond the exact-tracking frontier. We prove that, for finite 
groups under uniform i.i.d. full-group inputs, optimal order-blind 
exact accuracy converges to the reciprocal of abelianization 
class size as prefix length grows, consistent with the observed 
abelianization plateaus. Sequential updates permit more: 
any partition into right cosets of a subgroup, normal or not, 
survives sequential updates. In our census of standard Transformers, 
every recovered coset partition comes from a normal subgroup, whereas 
parameter-matched recurrent networks pass through both normal and 
non-normal right-coset stages during training. On $A_5$, we 
identify low-dimensional subspaces of the recurrent state that 
encode non-normal cosets. In the three-dimensional cases, coset mean vectors form approximate
dodecahedra, and swapping the state components 
in these subspaces transfers the donor's coset state through a shared
input suffix. Our results connect partial accuracy, learning
stages, and internal computation through the subgroup cosets that 
models learn to track.

\end{abstract}

\section{Introduction}
\label{sec:intro}

Tracking a changing state requires composing a sequence of updates. This problem
arises when language models follow entities in prompts \citep{kim2023entity},
execute programs \citep{nye2021scratchpads}, or predict game moves
\citep{li2023world}. Studies have found evidence of internal state representations
\citep{li2021implicit}. Yet strong predictive performance does not
imply faithful recovery of the underlying state space \citep{vafa2024world}.

\begin{figure}[t]
\centering
\includegraphics[width=\linewidth]{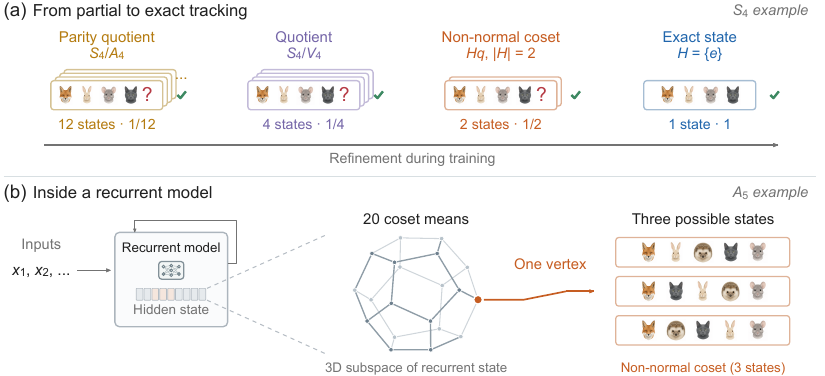}
\input{figures/fig1_overview_caption}
\label{fig:overview}
\end{figure}

Prior work establishes parallel shortcuts \citep{liu2023shortcut}, identifies
associative and parity-assisted scans \citep{li2025state}, and characterizes
the acquisition of group representations in two-layer networks
\citep{marchetti2026sgc}. We connect these accounts by asking which states the model 
can still tell apart. We characterize this information through 
group quotients and the uncertainty within their classes,
then show how a broader family of subgroup cosets describes recurrent
learning stages.

In our group composition task, the state is a running product \citep{liu2023shortcut}. 
Quotient tracking identifies a class of possible states without distinguishing its members
(Fig.~\ref{fig:overview}). For $S_4$, parity leaves 12 possible states and
the finer $S_3$ quotient leaves four. We find that class size predicts partial
accuracy when the model identifies the class but predicts nearly uniformly
within it. We prove that the product properties recoverable without 
input order are exactly those determined by the abelianization and show that, 
under uniform i.i.d. full-group inputs on finite groups, 
optimal order-blind accuracy converges to the reciprocal class size.

Sequential tracking permits more than quotient solutions. Any partition into
right cosets of a subgroup survives sequential updates, even when the subgroup
is not normal and its classes do not form a quotient group. We find these
non-normal coset stages in recurrent networks, including on $A_5$, which has
no nontrivial proper quotient group. We then test whether the network actually uses
these coset states. On $A_5$, a few directions of the recurrent state encode
the cosets, and exchanging those components transfers the donor's coset
state through a shared input suffix.

\begin{enumerate}[leftmargin=*,nosep]
\item \textbf{A quantitative account of partial tracking.}
We identify quotient solutions beyond parity \citep{li2025state}, account for their partial
accuracy through class size and output uncertainty, and establish both the
recoverable information and the asymptotic accuracy limit without input order
(\S\S\ref{sec:law}--\ref{sec:theory}).
\item \textbf{From persistent solutions to learning stages.}
We connect Transformer quotient solutions with normal and non-normal coset
stages in recurrent training, showing that refinement need not follow a
sequence of quotient groups (\S\ref{sec:levers}).
\item \textbf{An internal representation of coset states.}
We identify the geometry of coset mean vectors in hidden space and show through state
interventions that the corresponding subspaces carry subsequent coset
predictions (\S\ref{sec:coset-state}).
\end{enumerate}

\section{Background and preliminaries}
\label{sec:setup}
\label{sec:preliminaries}

\paragraph{Group composition as state tracking.} A group $G$ is a set of
invertible transformations with an associative composition operation and an
identity $e$. Each input $x_t\in G$ updates the state by $q_t=q_{t-1}x_t$,
starting from $q_0=e$. The task is to predict every running product
$q_t=x_1\cdots x_t$, a model of state tracking used in prior work
\citep{liu2023shortcut,li2025state}. A group is abelian if all elements commute.
For permutation groups, we compose left to right, $(pq)(i)=q(p(i))$.
We use $S_n$ for permutations of $n$ objects, $A_n$ for even permutations,
$C_n$ for the cyclic group of order $n$, and $V_4$ for the Klein four-group.

\paragraph{Cosets and quotient states.} A subgroup $H\le G$ partitions $G$
into \emph{right cosets} $Hq=\{hq:h\in H\}$, each containing $|H|$ states.
Tracking only the coset of $q_t$ leaves its members indistinguishable. For example,
if $H=\{e,h\}$, each class is the pair $\{q,hq\}$. A subgroup is normal when
$Hq=qH$ for every $q\in G$. We then write $N\trianglelefteq G$, and its classes form the quotient
group $G/N$, with $(qN)(rN)=(qr)N$. We also call a quotient class a \emph{fiber}.
Full tracking has $H=\{e\}$.

\paragraph{Updating and combining summaries.} Right cosets support sequential
updates because $(Hq)x=H(qx)$. Composing two classes by multiplying their
representatives is well-defined only for normal subgroups. More generally,
if state summaries can be combined to recover a summary of their product,
the states sharing $e$'s summary form a normal subgroup
(App.~\ref{app:algebra}). These constraints do not specify a neural implementation.

\paragraph{Abelianization.} The commutator subgroup $K=[G,G]$ is generated by
the commutators $[x,y]=xyx^{-1}y^{-1}$. The quotient $G^{\mathrm{ab}}=G/K$, called the
\emph{abelianization}, is commutative.
When $K$ is finite, we write $f=|K|$ for its class size (Fig.~\ref{fig:overview}a).

\paragraph{Input order.} The \emph{multiset} of an input sequence records
its elements and their counts, without their order. We call a model
\emph{order-blind} if its output depends only on these counts.
Section~\ref{sec:theory} establishes which product information such a model
can recover exactly on every input.

\section{Quotient tracking in Transformers}
\label{sec:law}

\subsection{Experimental setup}
\label{sec:experimental-setup}

We train models with cross-entropy loss at every position of length-100 sequences.
The finite-group experiments use independent uniform inputs from the full group
unless stated otherwise. We also test infinite groups with finite input alphabets.
Our group suite includes every non-abelian group of order at most 12,
selected larger groups through order 125, and the abelian control $C_8$.
Recipe A uses a 4-layer, width-256 GPT-NeoX transformer
\citep{black2022neox} with 3.16M parameters. Recipe B, adapted from
\citet{liu2023shortcut}, uses a 3-layer, width-512 GPT-2 transformer
\citep{radford2019gpt2}. The baseline uses three seeds per group; configurations
and exceptions are in Apps.~\ref{app:config} and~\ref{app:coverage}.

For the most probable output $\hat q_t$, we measure exact accuracy
$A_{\mathrm{exact}}:=\Pr(\hat q_t=q_t)$ and quotient accuracy 
$A_{\mathrm{quotient}}:= \Pr(\hat q_tN=q_tN)$.
Their ratio is exact accuracy conditional on a correct quotient,
$A_{\mathrm{exact}\mid Q}:=\Pr(\hat q_t=q_t\mid\hat q_tN=q_tN)$, and
$|G|A_{\mathrm{exact}}$ is exact accuracy in units of chance.
The exact and quotient \emph{frontiers} are the longest contiguous prefixes
with accuracy at least .75 and .90, respectively. Beyond the transition
region, we also measure uncertainty within the true quotient class.
Evaluation windows, sample counts, and threshold choices are detailed in
App.~\ref{app:readouts}.

\subsection{The class size predicts partial accuracy}
Under recipe A, we find that Transformers settle at group-dependent exact accuracies beyond
the exact frontier. Across the groups we examined, these accuracy levels often closely match the reciprocal abelianization class size, 
$1/f$ (Fig.~\ref{fig:quotient-output}a,d,e, Table~\ref{tab:law}). 
This match suggests that models retain the abelianization class while leaving its 
members unresolved. We test this interpretation without prescribing a subgroup. 
In a separate census across nine groups, we group states with similar
mean output distributions and test whether the recovered partitions form cosets.
The states that Transformers fail to distinguish are grouped 
exactly according to their abelianization class
(App.~\ref{app:blind-census}).

Without predictive information distinguishing $f$ equiprobable class members,
conditional exact accuracy is $1/f$. Allowing for class errors gives
$A_{\mathrm{exact}}\approx A_{\mathrm{quotient}}/f$, with no fitted parameter
(Table~\ref{tab:law}). Near-uniform probabilities alone do not imply this
argmax accuracy (we test output uniformity separately in \S\ref{sec:fiber}).
Comparing $A_4$ and $\mathrm{SL}(2,3)$ isolates class size from abelianization alone. 
Both have abelianization $C_3$, but $f$ equals four and eight, with conditional accuracies near
$1/4$ and $1/8$, respectively. The dihedral groups follow the same prediction.
We establish the special role of abelianization without input order in
\S\ref{sec:theory}.

This prediction requires finite classes, not a finite group. We compare finite
Heisenberg groups $H_3(\mathbb Z/p)$ with the infinite groups
$G_p=H_3(\mathbb Z)/\langle z^p\rangle$, where $z=(0,0,1)$.
In $G_p$, the coordinates $x,y$ are unbounded, but each abelianization class
contains only $p$ states. With full-group inputs for the finite groups and
a finite alphabet for $G_p$, both families accurately track the quotient
and approach $1/3$ exact accuracy for $p=3$
(Fig.~\ref{fig:quotient-output}d,e; Table~\ref{tab:law}).
Specific methods are given in App.~\ref{app:heisenberg-law}.

\begin{figure}[!t]
\centering
\includegraphics[width=\linewidth,trim={0 2.5bp 0 2.5bp},clip]{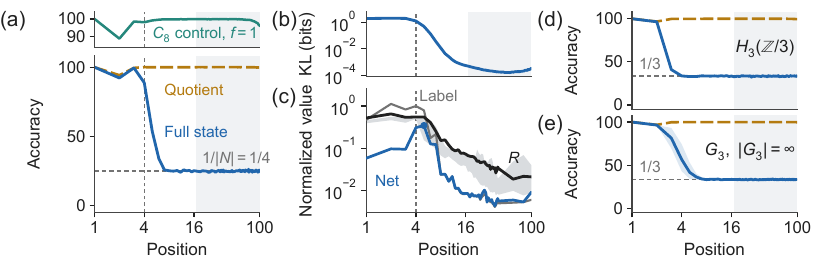}
\input{figures/fig2_a4_caption}
\label{fig:quotient-output}
\end{figure}

\subsection{Accurate classes, unresolved members}
\label{sec:fiber}

We now test the uniform-output assumption directly, using $A_4$ as an example.
Under recipe A, a model trained on $A_4$ identifies the correct class
in $A_4/V_4\cong C_3$ with 99.96\% accuracy across positions 17--100, 
while exact accuracy is only 24.98\% (Fig.~\ref{fig:quotient-output}a), 
and only 0.0527\% of its output mass lies
outside that four-state class. In contrast, with the same architecture and budget, we
obtain 99.26\% full-state accuracy on $C_8$, where every abelianization class contains
a single state. Repeated inputs account for the separate early dip at
position 2 (Fig.~\ref{fig:quotient-output}a, App.~\ref{app:early-dip}).

We next ask whether individual predictions distinguish states within the
correct quotient class. We restrict each prediction to its true class,
renormalize it, and compute its KL divergence from uniform before averaging.
Across 8,192 sequences and positions 17--100, the mean divergence is only
.000232 bits (Fig.~\ref{fig:quotient-output}b). Thus the within-class
distributions are close to uniform on average. The negligible log-likelihood gain likewise indicates no average advantage
for the true member (App.~\ref{app:readouts}). In hidden states, our linear
probes recover quotient identity but remain near chance for a state's index
within its class (App.~\ref{app:output-details}, 
Fig.~\ref{fig:linear-probes}).

\begin{table}[!htbp]
\centering
\caption{Quotient-class size and partial accuracy, min--max over the same runs.
Here $f=|[G,G]|$ is the size of an abelianization class.
Cond. is $A_{\mathrm{exact}}/A_{\mathrm{quotient}}$.
We report $\times$chance as $|G|A_{\mathrm{exact}}$. With accurate quotient
tracking, the class-size prediction gives $\times$chance approximately $|G|/f$.
$C_8$ has singleton classes and $A_5$ a single
class; we omit conditional and quotient accuracies for these controls.
$A_5$ and $S_5$ share $f$ but differ twofold in the predicted $\times$chance.
$^{*}$148,438 training steps; unmarked rows use 74,219. We use all available doubled-budget
runs where possible (App.~\ref{app:coverage}). $^{\dagger}$The $G_5$ window includes an early transition;
exact accuracy at positions 70--100 is .1995--.1997.}
\label{tab:law}
\footnotesize
\begingroup
\fontsize{8}{9.5}\selectfont
\setlength{\tabcolsep}{1pt}
\begin{minipage}[t]{.458\linewidth}
\centering
\begin{tabular}[t]{@{}lrrrrrr@{}}
\toprule
$G$ & $|G|$ & $f$ & $1/f$ & Cond. (\%) & Quot. (\%) & $\times$chance \\
\midrule
$C_8$ & 8 & 1 & 1 & -- & -- & 7.88--7.96 \\
$D_4^{*}$ & 8 & 2 & $1/2$ & 49.95--50.05 & 99.94--99.99 & 4.00 \\
$Q_8^{*}$ & 8 & 2 & $1/2$ & 49.91--49.96 & 99.92--99.93 & 3.99 \\
$S_3$ & 6 & 3 & $1/3$ & 33.32--33.39 & 99.89--99.92 & 2.00 \\
$D_6^{*}$ & 12 & 3 & $1/3$ & 33.31--33.47 & 99.86--99.97 & 3.99--4.01 \\
$\mathrm{Dic}_3$ & 12 & 3 & $1/3$ & 33.25--33.36 & 99.51--99.93 & 3.97--3.99 \\
$A_4^{*}$ & 12 & 4 & $1/4$ & 25.00--25.02 & 99.97 & 3.00 \\
$D_5$ & 10 & 5 & $1/5$ & 19.99--20.07 & 89.34--99.99 & 1.79--2.01 \\
$C_7\!\rtimes\!C_3$ & 21 & 7 & $1/7$ & 14.24--14.30 & 99.93--99.95 & 2.99--3.00 \\
\bottomrule
\end{tabular}
\end{minipage}\hfill
\begin{minipage}[t]{.525\linewidth}
\centering
\begin{tabular}[t]{@{}lrrrrrr@{}}
\toprule
$G$ & $|G|$ & $f$ & $1/f$ & Cond. (\%) & Quot. (\%) & $\times$chance \\
\midrule
$\mathrm{SL}(2,3)^{*}$ & 24 & 8 & $1/8$ & 12.45--12.49 & 99.95--99.98 & 2.99--3.00 \\
$S_4$ & 24 & 12 & $1/12$ & 8.30--8.36 & 99.77--99.98 & 1.99--2.00 \\
$S_5$ & 120 & 60 & $1/60$ & 1.65--1.66 & 79.36--99.99 & 1.57--1.99 \\
$A_5$ & 60 & 60 & $1/60$ & -- & -- & 0.99--1.01 \\
\midrule
$H_3(\mathbb Z/3)$ & 27 & 3 & $1/3$ & 33.32--33.43 & 99.82--99.96 & 8.98--9.02 \\
$H_3(\mathbb Z/5)^{*}$ & 125 & 5 & $1/5$ & 19.95--19.96 & 99.87--99.97 & 24.90--24.94 \\
\midrule
$G_3$ & $\infty$ & 3 & $1/3$ & 33.29--33.40 & 99.96 & -- \\
$G_4$ & $\infty$ & 4 & $1/4$ & 24.99--25.11 & 99.95--99.97 & -- \\
$G_5$ & $\infty$ & 5 & $1/5$ & 21.15--21.81$^{\dagger}$ & 99.92--99.93 & -- \\
\bottomrule
\end{tabular}
\end{minipage}
\endgroup

\end{table}

\subsection{Beyond the abelianization}
\label{sec:ladder}
Quotient tracking also extends beyond the abelianization. With recipe B and
odd-permutation inputs, we find a Transformer model tracking $S_4/V_4\cong S_3$ 
(Fig.~\ref{fig:decay}a).
This quotient refines the two parity classes into six classes of four states
and retains order-sensitive information. Quotient accuracy is near 100\%
through position 20 and falls to 72\% at position 100, while conditional
exact accuracy averages 25.01\% over positions 17--100. Within these classes,
mean per-example KL from uniform is .00272 bits. 
Across 21 runs on seven groups, we observe intermediate quotient behavior in
five runs, one on $S_4$, two on $D_{15}$, and two on $Q_{16}$.
Appendix~\ref{app:quotient-stages} gives coverage, selection thresholds,
distribution diagnostics, and reordering tests.
We also observe quotient stages before recurrent models reach exact tracking
(\S\ref{sec:lstm-stages} and App.~\ref{app:training-details}), although
their learning order varies across runs.

\subsection{Learning signals within quotient classes}
\label{sec:mechanism}

We ask whether the loss still supplies a signal for distinguishing members of
a learned class. We construct input variants with the same quotient target
but different full-state targets, covering every member of the class. For
each variant, we isolate the component of the cross-entropy gradient with
respect to all trainable parameters that distinguishes class members. Its \emph{Label} term contrasts the true target with a uniform label
on the class; its \emph{Model} term measures the model's departure from
within-class uniformity. Their sum is \emph{Net}. We average these vectors
before taking their norms (derivation in App.~\ref{app:gradient-details}).

For the $A_4$ model studied in \S\ref{sec:fiber}, the Net signal peaks
at position 5 and falls to 1.63\% of that peak at position 40
(Fig.~\ref{fig:quotient-output}c). To distinguish small individual signals
from cancellation between examples, we also measure
$R=\|\sum_i\ell_i\|/\sum_i\|\ell_i\|$, where $\ell_i$ is a variant's Label
gradient. A small $R$ 
means that contributions largely cancel when combined.
At position 40, $R$ is 0.8--3.3\% across nine $A_4$, $\mathrm{SL}(2,3)$, and
$S_4$ Transformers, covering class sizes \(f=4,8,12\).
Thus a nonzero signal on each example can leave little signal after averaging.

Cancellation suggests an obstacle to refinement. If logits and their parameter
sensitivities are identical across class members, averaging their distinct
targets produces a uniform class target. Uniform outputs alone do not
guarantee this condition, since their parameter sensitivities may differ.
We also modify backpropagated gradients while leaving the forward computation
unchanged. Generic gradient perturbations do not improve the frontier relative
to controls. An intervention aligned with a parameter-update direction obtained
from later training improves local accuracy but still does not advance the
frontier. These tests leave the causal role of cancellation in persistent
partial learning unresolved (App.~\ref{app:gradient-details}).

\section{The algebraic boundary without input order}
\label{sec:theory}

Next, we test whether models' full predictions depend on input order. Under recipe A
(uniform i.i.d. full-group inputs), prefix shuffling changes the full output
distribution little beyond the exact frontier in 28 runs across seven groups
(Fig.~\ref{fig:recipe-a-reordering}; App.~\ref{app:recipe-a-reordering}). 
What product information survives without order?

\begin{theorem}[State information available without order]
\label{thm:orderblind}
Fix $t\ge3$ and $\phi:G\to\mathbb R$. There exists an order-blind model that
outputs $\phi(x_1\cdots x_t)$ on every sequence in $G^t$ if and only if $\phi$
is constant on $[G,G]$-cosets.
\end{theorem}
The proof and length-two exception are in App.~\ref{app:algebra}.
We next ask how accurately a model can predict the full product from the multiset.

\begin{proposition}[Optimal order-blind accuracy]
\label{prop:ceiling}
Let $|[G,G]|=f<\infty$. For any distribution of length-$t$ words over a finite
alphabet in $G$, let $M$ be the input multiset and define
\begin{equation}
C_t=\mathbb E_M\!\left[\max_{g\in G}\Pr(q_t=g\mid M)\right].
\label{eq:orderblind-ceiling}
\end{equation}
The maximum exact accuracy of an order-blind model is $C_t\ge1/f$.
Equality holds precisely when the conditional product distribution is uniform
on its $[G,G]$-coset for almost every multiset.
\end{proposition}

This includes infinite groups with finite classes, such as $G_p$.
Can input counts favor particular members of the correct class and raise
accuracy above $1/f$? For finite groups under uniform full-group inputs, we show that
this advantage is bounded and vanishes with sequence length.

\begin{theorem}[Asymptotic order-blind accuracy]
\label{thm:orderblind-asymptotic}
For a finite group $G$ under uniform i.i.d. inputs from $G$, $C_t$ is
non-increasing in $t$ and converges to $1/f$, where $f=|[G,G]|$. For every $t\ge1$,
\[
0\le C_t-\frac1f\le\left(1-\frac1f\right)\sqrt{\frac{t+3}{2^{t+1}}}.
\]
\end{theorem}

Thus $1/f$ is the asymptotic limit of optimal order-blind accuracy, not just
a uniform-guessing baseline. Under these assumptions, an order-blind predictor's exact accuracy also
approaches its abelianization accuracy divided by $f$, without requiring
uniform outputs. We prove the bound and give sharper group-specific bounds
in App.~\ref{app:orderblind-asymptotic}.

At finite lengths, some multisets still favor particular products.
\begin{corollary}[Finite-length excess]
\label{cor:finite-excess}
For finite non-abelian $G$ under independent uniform inputs from $G$,
$C_t\ge 1/f+|G|^{-t}(1-1/f)>1/f$ at every finite $t\ge1$.
\end{corollary}
The all-identity multiset gives this lower bound. Since the convergence bound
can be loose, we also compute the full
excess $C_t-1/f$ for ten non-abelian groups through order 24. At
$t=17$, the start of our standard evaluation window, these calculations give
excesses of approximately $1.5\times10^{-4}$ to
$3.1\times10^{-3}$. For larger groups in this calculation, we sample multisets
and enumerate their orderings exactly. These are numerical estimates,
rather than rigorous bounds. Monotonicity ensures that any rigorous upper
bound on $C_{17}-1/f$ also holds at all later positions.
Proofs, enumeration details, and the relation to general state summaries are
in Apps.~\ref{app:algebra} and~\ref{app:ceiling}.

Reordering tests connect our plateaus to this limit, without proving strict
order-blindness or explaining training's choice. Non-abelian quotients remain order-sensitive
(App.~\ref{app:quotient-stages}), and restricted alphabets require separate analysis
(App.~\ref{sec:recovery}). 

\section{Persistence and refinement of state resolution}
\label{sec:levers}

We now ask how training changes the state distinctions models recover.

\subsection{Persistent quotient solutions}

We find that longer training extends exact prefixes without necessarily
refining resolution beyond them. Our five-million-update $S_4$ run ends at
frontier 9. Over positions 17--100, it reaches 99.67\% abelianization accuracy,
but exact accuracy conditional on a correct quotient remains 8.28\%, near
$1/12$. The finer $S_3$ quotient stays below 34.1\% throughout training
(Fig.~\ref{fig:long-training-stages}a,b). Further long runs and GPT-2/Llama
controls retain class-size baselines beyond their exact prefixes
(App.~\ref{app:training-details}).

For state-space models, we also find the same separation in 
parameter-matched Mamba-2 models \citep{dao2024mamba2}. 
On each of $D_4$, $A_4$, and $S_4$, two of three
two-layer models accurately recover the abelianization, 
with conditional exact accuracy remaining near the class-size 
baseline. A separate four-layer parameter-matched $D_4$ 
configuration improves quotient accuracy without
improving within-class resolution (App.~\ref{app:ssm}).

We also test whether equal feature-acquisition scores yield equal output
resolutions. We match the scores of \citet{marchetti2026sgc} in an
$S_3\times C_3$ control. Despite the matched scores, models recover the $C_3$
factor accurately while retaining only parity information about the $S_3$
factor (App.~\ref{app:score-control}), \textbf{suggesting that the
score alone cannot determine the resolution learned.}

\begin{figure}[!htbp]
\centering
\includegraphics[width=\linewidth,trim={0 8bp 0 0},clip]{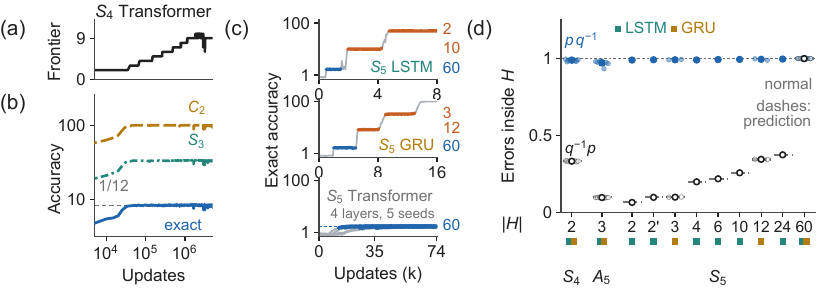}
\input{figures/fig4_long_training_caption}
\label{fig:long-training-stages}
\end{figure}

\subsection{From quotients to coset resolutions}
\label{sec:lstm-stages}
Our two-layer, parameter-matched LSTM and GRU models can also reach
resolutions that are not quotient groups (App.~\ref{app:nonsolvable-census}).
Three LSTM runs on $S_5$ exhibit eight sustained coset stages, two for the
normal subgroup $A_5$ and six for non-normal subgroups. One run reaches exact
tracking through subgroups of orders 24, 6, and 2, without an observed
intermediate $A_5$ stage. Three GRU runs instead pass through $A_5$ and then
non-normal subgroups of orders 12 and 3, and two reach exact tracking within the
budget (Fig.~\ref{fig:long-training-stages}c). A stage requires both a unique
fit of pooled prediction errors to a subgroup coset and accuracy near $1/|H|$, sustained
over successive evaluations. We give the thresholds and all trajectories in
Apps.~\ref{app:blind-census} and~\ref{app:nonsolvable-census}.

The distinction is clearest on $A_5$, which has no nontrivial proper normal
subgroup. All three LSTM and three GRU runs in our successful learning-rate
arms pass through cosets of a three-element subgroup $C_3$. These 20 classes
leave three possible states each, giving partial accuracy near $1/3$.
Thus an intermediate state resolution need not be a quotient at all.
All six runs reach near-exact tracking, although one LSTM later loses it
(App.~\ref{app:nonsolvable-census}).

We test the coset interpretation by asking where a model's incorrect
full-state predictions fall. Let $p$ be an incorrect prediction and $q$ the true state. If $p$ remains in the true coset $Hq$, then
$p=hq$ for some $h\in H$. Equivalently, $pq^{-1}=h\in H$. Reversing the order gives 
$q^{-1}p=q^{-1}hq$, which belongs to the conjugate subgroup
$q^{-1}Hq$. Since the true state $q$ varies across examples, these
reverse-order errors need not lie in one fixed copy of $H$. 
Fig.~\ref{fig:long-training-stages}d compares both fractions of 
errors inside $H$ with the algebraic prediction. We find
that incorrect predictions are indeed concentrated in the true coset, 
while the reverse-order fractions match our prediction 
(Fig.~\ref{fig:long-training-stages}d; 
full derivation in App.~\ref{app:nonsolvable-census}).
Their agreement hence supports right-coset structure beyond the accuracy plateau
alone.

In our census of standard Transformers, we recover no non-normal coset
partitions. All five four-layer $S_5$ runs end at the $A_5$ stage
(Fig.~\ref{fig:long-training-stages}c); broader surveys across groups and
depths are reported in App.~\ref{app:nonsolvable-census}.
This is an observed contrast, not an architectural prohibition. It's worth 
noting that with recirculation, which feeds representations from preceding positions
back into the model \citep{mozer2026recirculation}, 
we do observe a non-normal coset stage in one $S_4$ Transformer 
trajectory (App.~\ref{app:recirculation}).

\section{Internal representation of coset states}
\label{sec:coset-state}

We next ask whether the cosets in model predictions correspond to an
internal state that the network uses. We begin by transferring recurrent
activations between sequences and testing which coset the subsequent
predictions follow. Here, we study the $C_3$ stages of the six $A_5$ networks above,
at three times within each stage. For a subgroup 
\(H=\operatorname{Stab}(a,b)\cong C_3\), the three
states in a right coset \(Hq\) send objects \(a\) and \(b\) to the same ordered
pair of locations. We therefore label each coset by \((q(a),q(b))\). Other tests 
are in App.~\ref{app:coset-mechanism}.

\subsection{Locating and transferring the coset state}

We first replace a recurrent layer state with one from a different coset.
At position $t$, we insert a donor's recurrent layer state into a recipient
sequence and continue with the recipient's remaining inputs. We measure whether
subsequent predictions follow the donor's coset as it is updated by those
inputs, calling this fraction \emph{donor agreement}. \emph{Recipient agreement}
instead measures whether predictions follow the recipient's original coset
under the same subsequent inputs. This transfer motivates
a more specific question: which components of the layer state carry the coset?

To locate them, we group the recorded layer states by the true coset of
the running product and average within each group across sequences and
positions. This gives one mean vector for each of the 20 cosets. We subtract
the mean of these vectors and apply singular value decomposition (SVD)
to identify directions that distinguish the classes. First, we patch
only the recipient's components along the leading directions with those
of the donor, leaving the rest unchanged. These subspaces have three dimensions
in four models and seven in two LSTMs. Swapping them nearly reproduces full-layer
transfer (Fig.~\ref{fig:coset-state}c). Across the six models and three sampled times
per model, donor agreement is 90.0--99.9\% over the next 50 positions we measure,
within 2.5 percentage points of full-layer replacement. Swapping only the
remaining directions instead preserves the recipient's coset on
89.1--99.9\% of predictions. We thus locate a small subspace that transfers
the tracked class through subsequent inputs. The intervention formula,
all models, and controls are in App.~\ref{app:coset-interventions}. Separate
tests on $S_5$ also transfer coset predictions
(App.~\ref{app:causal-coset}).

\begin{figure}[!htbp]
\centering
\includegraphics[width=\linewidth,trim={0 9bp 0 0},clip]{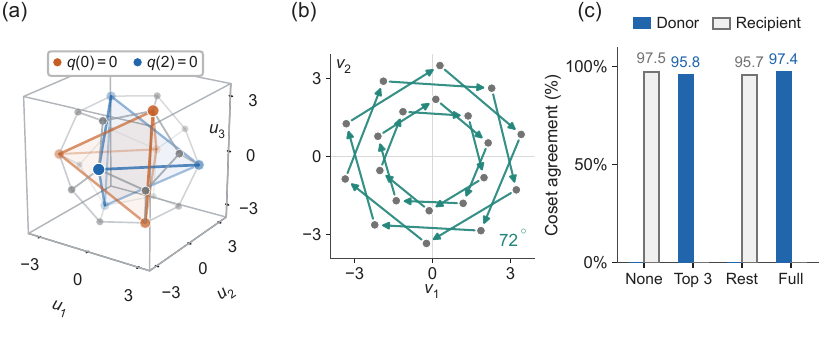}
\input{figures/fig_coset_state_caption}
\end{figure}

\subsection{The geometry and updates of the coset state}

This prompts us to ask what these causally effective directions represent. 
For the four models with a three-dimensional coset subspace, we plot each
centered coset mean along the three leading SVD directions. 
When plotted out in 3D, we find that 
\textbf{the 20 points lie almost exactly at the vertices of a 
regular dodecahedron}, a solid with 12 congruent
regular pentagonal faces and 20 vertices 
(Fig.~\ref{fig:coset-state}a). Each point represents a coset of
three possible group elements. We confirm the arrangement 
quantitatively by comparing normalized pairwise inner 
products with those of a regular
dodecahedron (App.~\ref{app:coset-geometry}). 
We also find internal geometries reflecting learned abelian quotients
in Transformers (App.~\ref{app:transformer-geometry}). 
This suggests that
the models learn an intrinsic geometric representation of the 
coset state in parallel to merely computing the group operations.

We then test whether updating inputs respect this geometry. 
Each input $x$ sends a coset $Hq$ to $Hqx$ as defined in 
\S\ref{sec:preliminaries}, specifying a permutation of the 20 means.
In the 3D cases, a fitted rotation $R_x$ closely matches each such
permutation of class means. Using row-vector coordinates,
$v_{Hq}R_x\approx v_{Hqx}$, the fitted maps satisfy 
\(R_xR_y\approx R_{xy}\) (App.~\ref{app:coset-geometry}).
Fig.~\ref{fig:coset-state}b shows the $72^\circ$ rotation for one 5-cycle.

\textbf{The measured rotations explain why the geometry is dodecahedral.}
Although these cosets do not form a quotient group, each input still permutes
them through $Hq\mapsto Hqx$. In an exact rotation-based code, 
their representation vectors take the form $v_{Hq}=wR_q$, 
where $wR_h=w$ for every $h\in H$, because inputs from
$H$ leave the reference vector unchanged. Three is the smallest dimension
of a nontrivial real linear action of $A_5$. In its three-dimensional rotation
representations, $C_3$ fixes a single axis, and rotating a nonzero vector on
this axis by $A_5$ produces the 20 vertices of a regular dodecahedron.
Our measurements approximately realize this minimal-dimensional code
(App.~\ref{app:coset-geometry}).

The two LSTMs with seven-dimensional coset subspaces combine three- and four-dimensional
components. Neither component alone transfers the coset reliably, but replacing
both does (App.~\ref{app:coset-interventions}). The fitted group action
describes class mean hidden vectors, while individual samples' activations 
vary within each class and do not undergo rigid rotations. 
Together, the interventions and geometry identify an internal 
representation of a $G$-set, a set permuted by group inputs, rather than 
a quotient group. Its class means approximately realize the sequential 
update $Hq\mapsto Hqx$ from \S\ref{sec:preliminaries} as a linear action.

\section{Related work}
\label{sec:related}

\paragraph{Shortcuts and learned information.} \citet{liu2023shortcut}
establish parallel shortcuts for state computation; \citet{li2025state}
identify associative and parity-assisted scans. We characterize the state
distinctions learned, extending parity accounts to non-parity and non-abelian
quotients and relating class size to partial accuracy. Expressivity bounds
\citep{barrington1986,hahn2020limitations,merrill2025depth} and
length-generalization tests \citep{deletang2023chomsky} address complementary limits.

\paragraph{Cosets and internal computation.} \citet{stander2024grokking}
identify coset circuits for group multiplication. \citet{wu2025unified}
connect them to irreducible representations and certify performance; their
stabilizer account includes non-normal subgroups. We connect these codes'
resolution to sequential tracking accuracy and training stages, and test
whether exchanging their components transfers coset predictions through a shared suffix.

\paragraph{Features and state resolution.} Fourier analyses
\citep{nanda2023progress,chughtai2023toy} and low-rank tensors
\citep{shutman2025words} explain arithmetic and group learning.
\citet{marchetti2026sgc} characterize encoding-dependent representation
acquisition in two-layer networks under alternating gradient flow
\citep{kunin2025agf}. A full representation matrix distinguishes elements
modulo its kernel; a vector orbit distinguishes cosets of its stabilizer.
We measure which resolution is expressed in predictions and how much
uncertainty remains within its classes.
Our matched-score control (App.~\ref{app:score-control}) examines transfer
to Transformers outside their theoretical setting.

\paragraph{Learning stages and recurrence.} Learning can precede accuracy gains
\citep{power2022grokking,barak2022hidden,nanda2023progress}.
\citet{forner2026cot} analyze staged retrieval and updates in a solvable model
of chain-of-thought tracking; \citet{sahasrabudhe2026marginals} separates plateau
uncertainty from duration through a marginal-to-conditional transition.
We identify acquired state distinctions and gradient cancellation that may
impede refinement, complementing feature-competition accounts \citep{pezeshki2021starvation}.
Our comparisons draw on recurrent expressivity and error bounds
\citep{merrill2024illusion,chung2026errorcontrol}, composition tests
\citep{lee2026falsifier}, and added recurrence \citep{mozer2026recirculation}.

\section{Conclusion}
\label{sec:discussion}
\label{sec:conclusion}

We characterize learned state tracking through the subgroup cosets that
models resolve. Class size and output uncertainty explain structured partial
accuracy, while normal and non-normal coset stages connect persistent
solutions with further learning. We go beyond output descriptions
to identify subspaces that causally carry the tracked coset. Together with
the exact order-blind boundary, these results give an algebraic account of
what models preserve, what remains unresolved, and how partial states
are represented inside models and support continued computation.

Training recipes can determine whether models reach exact tracking or remain
at a partial resolution. We hypothesize that input distributions shape both
the internal representations models learn and the gradients that refine them.
How these effects govern the selection and refinement of coset states remains open.

\clearpage
\subsection*{Reproducibility statement}
\label{sec:statements}
The appendices specify model configurations, data distributions, evaluation
windows, seed coverage, and the distinction between confirmatory and exploratory
analyses. Run-level manifests identify the Table~\ref{tab:law} cohorts and
retain all initial-budget results. The order-blind calculation includes a brute-force check at short
lengths. Mathematical assumptions and proofs are given in
App.~\ref{app:algebra}; measurement and intervention protocols are described
in Apps.~\ref{app:figure-methods}, \ref{app:coset-mechanism}, \ref{app:heisenberg},
and~\ref{app:gradient-details}.

\subsection*{Ethics statement}
The experiments use synthetic group-composition sequences and involve no
human subjects or personal data. The conclusions concern controlled algorithmic
tasks and do not establish performance or safety properties of deployed
language models.

\subsection*{AI use statement}

We set the research questions, chose which findings to pursue, fixed the claim
this paper is organized around, proposed and sketched out key proofs, 
proposed experiments that resulted in novel discoveries, 
and decided what entered the manuscript and what stayed out. 
Within that direction we used generative AI tools (Claude
Code and OpenAI Codex, across several model versions over the project period)
for the following tasks that require disclosure. They helped generate the synthetic
datasets, in the sense that they wrote the group-composition sampler scripts;
they helped check and complete our mathematical proofs, 
supplied ingredients for some proofs and drafted them, 
refined our hypotheses, designed experiments and controls and gave feedback on the
methodology, implemented the training, evaluation and analysis code,
aggregated and reformatted the stored run outputs into the reported tables, 
and translated between languages, since we discussed
the research largely in Chinese and drafted the manuscript in English. 
Every other required category involved AI assistance at some stage;
generating cartoon images for our illustrations is one example. 
The tools worked from our instructions and we kept or discarded what came
back.

We used the same tools for tasks with recommended disclosure: producing the
figures, suggesting experimental parameters, writing and editing code,
searching for and summarizing related work, identifying gaps, brainstorming,
proposing the structure of the paper, formatting references, proposing
candidate titles and keywords, and drafting and revising every section, this
statement included.

We have reviewed all AI-assisted work. Results that this paper reports as
confirmatory were registered before their analysis, the appendix labels the
exploratory ones, and the numbers in the tables and figures were reconciled
with the stored run logs rather than with a model's summary of them. Both authors
checked the mathematical statements and the proofs, read the data-generation
and analysis code, and re-ran spot checks of the reported quantities. 
Checks preceding our own review were also AI-assisted, including cross-model adversarial 
audits of the reported numbers and simulated reviews of the draft and codes, 
and we do not count that as independent human
verification. We take responsibility for the final content of this work,
including the text, claims, proofs, code and figures produced with the aid of
generative AI.

\bibliographystyle{plainnat}
\begingroup
\raggedright
\bibliography{refs}

\begin{thebibliography}{33}
\providecommand{\natexlab}[1]{#1}
\providecommand{\url}[1]{\texttt{#1}}
\expandafter\ifx\csname urlstyle\endcsname\relax
  \providecommand{\doi}[1]{doi: #1}\else
  \providecommand{\doi}{doi: \begingroup \urlstyle{rm}\Url}\fi

\bibitem[Alexeev and Zograf(2011)]{alexeev2011hultman}
Nikita Alexeev and Peter Zograf.
\newblock Hultman numbers, polygon gluings and matrix integrals.
\newblock \emph{arXiv preprint arXiv:1111.3061}, 2011.
\newblock URL \url{https://arxiv.org/abs/1111.3061}.

\bibitem[Barak et~al.(2022)Barak, Edelman, Goel, Kakade, Malach, and
  Zhang]{barak2022hidden}
Boaz Barak, Benjamin Edelman, Surbhi Goel, Sham Kakade, Eran Malach, and Cyril
  Zhang.
\newblock Hidden progress in deep learning: {SGD} learns parities near the
  computational limit.
\newblock In \emph{Advances in Neural Information Processing Systems
  (NeurIPS)}, volume~35, pages 21750--21764, 2022.
\newblock URL
  \url{https://papers.nips.cc/paper_files/paper/2022/hash/884baf65392170763b27c914087bde01-Abstract-Conference.html}.

\bibitem[Barrington(1989)]{barrington1986}
David~A. Barrington.
\newblock Bounded-width polynomial-size branching programs recognize exactly
  those languages in $\mathrm{NC}^1$.
\newblock \emph{Journal of Computer and System Sciences}, 38\penalty0
  (1):\penalty0 150--164, 1989.
\newblock Conference version: STOC 1986.

\bibitem[Black et~al.(2022)Black, Biderman, Hallahan, Anthony, Gao, Golding,
  He, Leahy, McDonell, Phang, Pieler, Prashanth, Purohit, Reynolds, Tow, Wang,
  and Weinbach]{black2022neox}
Sidney Black, Stella Biderman, Eric Hallahan, Quentin Anthony, Leo Gao,
  Laurence Golding, Horace He, Connor Leahy, Kyle McDonell, Jason Phang,
  Michael Pieler, Usvsn~Sai Prashanth, Shivanshu Purohit, Laria Reynolds,
  Jonathan Tow, Ben Wang, and Samuel Weinbach.
\newblock {GPT-NeoX-20B}: An open-source autoregressive language model.
\newblock In \emph{Proceedings of BigScience Episode \#5 -- Workshop on
  Challenges \& Perspectives in Creating Large Language Models}, 2022.
\newblock arXiv:2204.06745.

\bibitem[Chughtai et~al.(2023)Chughtai, Chan, and Nanda]{chughtai2023toy}
Bilal Chughtai, Lawrence Chan, and Neel Nanda.
\newblock A toy model of universality: Reverse engineering how networks learn
  group operations.
\newblock In \emph{International Conference on Machine Learning (ICML)}, 2023.
\newblock arXiv:2302.03025.

\bibitem[Chung et~al.(2026)Chung, Choi, and Kim]{chung2026errorcontrol}
Jiwan Chung, Heechan Choi, and Seon~Joo Kim.
\newblock Rethinking state tracking in recurrent models through error control
  dynamics.
\newblock \emph{arXiv preprint arXiv:2605.07755}, 2026.

\bibitem[Dao and Gu(2024)]{dao2024mamba2}
Tri Dao and Albert Gu.
\newblock Transformers are {SSM}s: Generalized models and efficient algorithms
  through structured state space duality.
\newblock In \emph{Proceedings of the 41st International Conference on Machine
  Learning}, volume 235 of \emph{Proceedings of Machine Learning Research},
  pages 10041--10071. PMLR, 2024.
\newblock URL \url{https://proceedings.mlr.press/v235/dao24a.html}.

\bibitem[Del{\'e}tang et~al.(2023)Del{\'e}tang, Ruoss, Grau-Moya, Genewein,
  Wenliang, Catt, Cundy, Hutter, Legg, Veness, and Ortega]{deletang2023chomsky}
Gr{\'e}goire Del{\'e}tang, Anian Ruoss, Jordi Grau-Moya, Tim Genewein, Li~Kevin
  Wenliang, Elliot Catt, Chris Cundy, Marcus Hutter, Shane Legg, Joel Veness,
  and Pedro~A. Ortega.
\newblock Neural networks and the chomsky hierarchy.
\newblock In \emph{International Conference on Learning Representations
  (ICLR)}, 2023.
\newblock URL \url{https://arxiv.org/abs/2207.02098}.

\bibitem[Forner et~al.(2026)Forner, K{\"u}hn, Thamm, and
  Rosenow]{forner2026cot}
Niklas Forner, Marcel K{\"u}hn, Matthias Thamm, and Bernd Rosenow.
\newblock Learning dynamics of chain-of-thought state tracking in a solvable
  transformer model.
\newblock \emph{arXiv preprint arXiv:2606.18164}, 2026.

\bibitem[Hahn(2020)]{hahn2020limitations}
Michael Hahn.
\newblock Theoretical limitations of self-attention in neural sequence models.
\newblock \emph{Transactions of the Association for Computational Linguistics},
  8:\penalty0 156--171, 2020.
\newblock \doi{10.1162/tacl_a_00306}.
\newblock URL \url{https://aclanthology.org/2020.tacl-1.11/}.

\bibitem[Kim and Schuster(2023)]{kim2023entity}
Najoung Kim and Sebastian Schuster.
\newblock Entity tracking in language models.
\newblock In \emph{Proceedings of the 61st Annual Meeting of the Association
  for Computational Linguistics (Volume 1: Long Papers)}, pages 3835--3855,
  2023.
\newblock \doi{10.18653/v1/2023.acl-long.213}.
\newblock URL \url{https://aclanthology.org/2023.acl-long.213/}.

\bibitem[Kunin et~al.(2025)Kunin, Marchetti, Chen, Karkada, Simon, DeWeese,
  Ganguli, and Miolane]{kunin2025agf}
Daniel Kunin, Giovanni~Luca Marchetti, Feng Chen, Dhruva Karkada, James~B.
  Simon, Michael~R. DeWeese, Surya Ganguli, and Nina Miolane.
\newblock Alternating gradient flows: A theory of feature learning in two-layer
  neural networks.
\newblock \emph{arXiv preprint arXiv:2506.06489}, 2025.

\bibitem[Lee(2026)]{lee2026falsifier}
Jeonghoon Lee.
\newblock A held-out transition-pair falsifier for long-horizon non-abelian
  state tracking.
\newblock \emph{arXiv preprint arXiv:2606.07254}, 2026.

\bibitem[Li et~al.(2021)Li, Nye, and Andreas]{li2021implicit}
Belinda~Z. Li, Maxwell Nye, and Jacob Andreas.
\newblock Implicit representations of meaning in neural language models.
\newblock In \emph{Proceedings of the 59th Annual Meeting of the Association
  for Computational Linguistics and the 11th International Joint Conference on
  Natural Language Processing (Volume 1: Long Papers)}, pages 1813--1827, 2021.
\newblock \doi{10.18653/v1/2021.acl-long.143}.
\newblock URL \url{https://aclanthology.org/2021.acl-long.143/}.

\bibitem[Li et~al.(2025)Li, Guo, and Andreas]{li2025state}
Belinda~Z. Li, Zifan~Carl Guo, and Jacob Andreas.
\newblock (how) do language models track state?
\newblock In \emph{International Conference on Machine Learning (ICML)}, 2025.
\newblock arXiv:2503.02854.

\bibitem[Li et~al.(2023)Li, Hopkins, Bau, Vi{\'e}gas, Pfister, and
  Wattenberg]{li2023world}
Kenneth Li, Aspen~K. Hopkins, David Bau, Fernanda Vi{\'e}gas, Hanspeter
  Pfister, and Martin Wattenberg.
\newblock Emergent world representations: Exploring a sequence model trained on
  a synthetic task.
\newblock In \emph{International Conference on Learning Representations
  (ICLR)}, 2023.
\newblock URL \url{https://openreview.net/forum?id=DeG07_TcZvT}.

\bibitem[Liu et~al.(2023)Liu, Ash, Goel, Krishnamurthy, and
  Zhang]{liu2023shortcut}
Bingbin Liu, Jordan~T. Ash, Surbhi Goel, Akshay Krishnamurthy, and Cyril Zhang.
\newblock Transformers learn shortcuts to automata.
\newblock In \emph{International Conference on Learning Representations
  (ICLR)}, 2023.
\newblock arXiv:2210.10749.

\bibitem[Loshchilov and Hutter(2019)]{loshchilov2019adamw}
Ilya Loshchilov and Frank Hutter.
\newblock Decoupled weight decay regularization.
\newblock In \emph{International Conference on Learning Representations
  (ICLR)}, 2019.
\newblock arXiv:1711.05101.

\bibitem[Marchetti et~al.(2026)Marchetti, Kunin, Myers, Acosta, and
  Miolane]{marchetti2026sgc}
Giovanni~Luca Marchetti, Daniel Kunin, Adele Myers, Francisco Acosta, and Nina
  Miolane.
\newblock Sequential group composition: A window into the mechanics of deep
  learning.
\newblock \emph{arXiv preprint arXiv:2602.03655}, 2026.

\bibitem[Merrill and Sabharwal(2025)]{merrill2025depth}
William Merrill and Ashish Sabharwal.
\newblock A little depth goes a long way: The expressive power of log-depth
  transformers.
\newblock \emph{arXiv preprint arXiv:2503.03961}, 2025.

\bibitem[Merrill et~al.(2024)Merrill, Petty, and
  Sabharwal]{merrill2024illusion}
William Merrill, Jackson Petty, and Ashish Sabharwal.
\newblock The illusion of state in state-space models.
\newblock In \emph{International Conference on Machine Learning (ICML)}, 2024.
\newblock arXiv:2404.08819.

\bibitem[Mozer et~al.(2026)Mozer, Siddiqui, Sawyer, Sanyal, and
  Liu]{mozer2026recirculation}
Michael~C. Mozer, Shoaib~Ahmed Siddiqui, Danny Sawyer, Sunny Sanyal, and
  Rosanne Liu.
\newblock Recirculation.
\newblock \emph{arXiv preprint arXiv:2608.17981}, 2026.

\bibitem[Nanda et~al.(2023)Nanda, Chan, Lieberum, Smith, and
  Steinhardt]{nanda2023progress}
Neel Nanda, Lawrence Chan, Tom Lieberum, Jess Smith, and Jacob Steinhardt.
\newblock Progress measures for grokking via mechanistic interpretability.
\newblock In \emph{International Conference on Learning Representations
  (ICLR)}, 2023.
\newblock arXiv:2301.05217.

\bibitem[Nye et~al.(2021)Nye, Andreassen, Gur-Ari, Michalewski, Austin, Bieber,
  Dohan, Lewkowycz, Bosma, Luan, Sutton, and Odena]{nye2021scratchpads}
Maxwell Nye, Anders~Johan Andreassen, Guy Gur-Ari, Henryk Michalewski, Jacob
  Austin, David Bieber, David Dohan, Aitor Lewkowycz, Maarten Bosma, David
  Luan, Charles Sutton, and Augustus Odena.
\newblock Show your work: Scratchpads for intermediate computation with
  language models.
\newblock \emph{arXiv preprint arXiv:2112.00114}, 2021.
\newblock URL \url{https://arxiv.org/abs/2112.00114}.

\bibitem[Pezeshki et~al.(2021)Pezeshki, Kaba, Bengio, Courville, Precup, and
  Lajoie]{pezeshki2021starvation}
Mohammad Pezeshki, Sekou-Oumar Kaba, Yoshua Bengio, Aaron Courville, Doina
  Precup, and Guillaume Lajoie.
\newblock Gradient starvation: A learning proclivity in neural networks.
\newblock In \emph{Advances in Neural Information Processing Systems
  (NeurIPS)}, 2021.

\bibitem[Power et~al.(2022)Power, Burda, Edwards, Babuschkin, and
  Misra]{power2022grokking}
Alethea Power, Yuri Burda, Harri Edwards, Igor Babuschkin, and Vedant Misra.
\newblock Grokking: Generalization beyond overfitting on small algorithmic
  datasets.
\newblock \emph{arXiv preprint arXiv:2201.02177}, 2022.
\newblock URL \url{https://arxiv.org/abs/2201.02177}.

\bibitem[Radford et~al.(2019)Radford, Wu, Child, Luan, Amodei, and
  Sutskever]{radford2019gpt2}
Alec Radford, Jeffrey Wu, Rewon Child, David Luan, Dario Amodei, and Ilya
  Sutskever.
\newblock Language models are unsupervised multitask learners.
\newblock Technical report, OpenAI, 2019.

\bibitem[Sahasrabudhe(2026)]{sahasrabudhe2026marginals}
Mihir Sahasrabudhe.
\newblock Marginals before conditionals.
\newblock \emph{arXiv preprint arXiv:2603.10074}, 2026.

\bibitem[Shutman et~al.(2025)Shutman, Louidor, and Tessler]{shutman2025words}
Maor Shutman, Oren Louidor, and Ran~J. Tessler.
\newblock Learning words in groups: fusion algebras, tensor ranks and grokking.
\newblock \emph{arXiv preprint arXiv:2509.06931}, 2025.

\bibitem[Stander et~al.(2024)Stander, Yu, Fan, and
  Biderman]{stander2024grokking}
Dashiell Stander, Qinan Yu, Honglu Fan, and Stella Biderman.
\newblock Grokking group multiplication with cosets.
\newblock In \emph{International Conference on Machine Learning (ICML)}, 2024.
\newblock arXiv:2312.06581.

\bibitem[Su et~al.(2024)Su, Lu, Pan, Murtadha, Wen, and Liu]{su2024rope}
Jianlin Su, Yu~Lu, Shengfeng Pan, Ahmed Murtadha, Bo~Wen, and Yunfeng Liu.
\newblock {RoFormer}: Enhanced transformer with rotary position embedding.
\newblock \emph{Neurocomputing}, 568, 2024.
\newblock arXiv:2104.09864.

\bibitem[Vafa et~al.(2024)Vafa, Chen, Rambachan, Kleinberg, and
  Mullainathan]{vafa2024world}
Keyon Vafa, Justin~Y. Chen, Ashesh Rambachan, Jon Kleinberg, and Sendhil
  Mullainathan.
\newblock Evaluating the world model implicit in a generative model.
\newblock In \emph{Advances in Neural Information Processing Systems
  (NeurIPS)}, volume~37, 2024.
\newblock URL
  \url{https://proceedings.neurips.cc/paper_files/paper/2024/hash/2f6a6317bada76b26a4f61bb70a7db59-Abstract-Conference.html}.

\bibitem[Wu et~al.(2025)Wu, Jaburi, Drori, and Gross]{wu2025unified}
Wilson Wu, Louis Jaburi, Jacob Drori, and Jason Gross.
\newblock Towards a unified and verified understanding of group-operation
  networks.
\newblock In \emph{International Conference on Learning Representations
  (ICLR)}, 2025.
\newblock arXiv:2410.07476.

\end{thebibliography}
\endgroup

\clearpage
\appendix
\raggedbottom
\makeatletter
\setlength{\@fptop}{0pt}
\setlength{\@fpsep}{20pt}
\makeatother
\setcounter{figure}{0}
\renewcommand{\thefigure}{S\arabic{figure}}
\renewcommand{\theHfigure}{supp.\arabic{figure}}

\section{Algebraic statements and proofs}
\label{app:algebra}

\paragraph{Composable summaries.} Suppose a summary $s:G\to S$ admits an
operation $M$ such that $s(gh)=M(s(g),s(h))$ for all $g,h\in G$. Restrict $S$
to the image of $s$. Associativity, an identity, and inverses are inherited
from $G$, so this image is a group and $s$ is a homomorphism. Its kernel
$N=\{g:s(g)=s(e)\}$ is normal, and two products have the same summary exactly
when they lie in the same $N$-coset. Thus a summary that composes arbitrary
segments has the resolution of a group quotient. This statement does not
assume that a neural model has learned such an exact summary.

\paragraph{Sequential summaries.} Suppose instead that
$s(gx)=U(s(g),x)$ for every $g,x\in G$. Define $g\sim g'$ when
$s(g)=s(g')$. The update rule makes this equivalence relation invariant under
right multiplication. Its identity class $H=\{h:s(h)=s(e)\}$ is a subgroup.
Indeed, $a\sim e$ and $b\sim e$ imply $ab\sim b\sim e$; multiplying
$a\sim e$ by $a^{-1}$ gives $e\sim a^{-1}$. Moreover,
$g\sim q$ if and only if $gq^{-1}\sim e$, so the class of $q$ is exactly
$Hq$. Conversely, every right-coset partition defines a sequential update
$U(Hq,x)=H(qx)$. Normality is unnecessary. If updates are initially specified
only on a finite alphabet that generates $G$ as a monoid, composition of
those updates gives the same conclusion.

These statements classify exact, state-dependent summaries, not arbitrary
history-dependent hidden states. The distinction motivates the measurements
in App.~\ref{app:blind-census}; observing a partition does not establish that
a neural network implements either exact summary rule.

\begin{lemma}[Reordering preserves the abelianized product]
\label{lem:reorder}
If two sequences have the same multiset of entries, their products have the
same image in $G/[G,G]$.
\end{lemma}
\begin{proof}
The projection to $G/[G,G]$ is a homomorphism from $G$ to an abelian group.
Products of the images therefore do not depend on their order.
\end{proof}

\paragraph{Proof of Theorem~\ref{thm:orderblind}.}
Any function constant on $[G,G]$-cosets is recoverable from the multiset,
since the projected elements commute. Conversely, choosing
$(x,y,z,e,\ldots,e)$ and swapping the second and third entries gives
$\phi(xyz)=\phi(xzy)$. Since $xyz=(xzy)[y^{-1},z^{-1}]$, varying $x$ gives
$\phi(g[y^{-1},z^{-1}])=\phi(g)$ for every $g,y,z$. These commutators
generate $[G,G]$, so $\phi$ is constant on its cosets.

\paragraph{The length-two exception.} For $t=2$, order invariance requires
$\phi(xy)=\phi(yx)$. The two products are conjugate; conversely, this identity
for all $x,y$ implies conjugation invariance. Thus class functions, including
functions finer than the abelianization, are recoverable at length two.
The free third factor in Theorem~\ref{thm:orderblind} strengthens conjugation
invariance to invariance under multiplication by every commutator.

\paragraph{Proof of Proposition~\ref{prop:ceiling}.}
Fix a multiset $M$ with positive probability. Since an order-blind model
has access only to $M$, its optimal prediction is a most probable value of
$q_t$ given $M$. Its conditional accuracy is therefore
\[
\max_{g\in G}\Pr(q_t=g\mid M),
\]
and averaging over $M$ gives $C_t$. By Lemma~\ref{lem:reorder}, all products
obtained by reordering $M$ have the same image in $G/[G,G]$. The conditional
product distribution is thus supported on one $[G,G]$-coset with $f$ elements.
Its largest probability is at least $1/f$, with equality exactly when it is
uniform on that coset. Taking expectations gives $C_t\ge1/f$; equality holds
if and only if this uniformity condition holds for almost every $M$.

For finite non-abelian $G$ with independent uniform full-group inputs, the
all-identity multiset $M_0=\{e,\ldots,e\}$ occurs with probability $|G|^{-t}$.
Conditional on $M_0$, the product is deterministically $e$, so the optimal
accuracy is $1$. Every other multiset contributes at least $1/f$, giving
\[
C_t\ge |G|^{-t}+\frac{1-|G|^{-t}}{f}>\frac{1}{f}.
\]
This proves Corollary~\ref{cor:finite-excess}.

\paragraph{What a conditional ratio can establish.} Consider a randomized
prediction that is uniform on $Hq_t$. Its exact accuracy is $1/|H|$, and the
probability of the correct $[G,G]$-coset is $|H\cap[G,G]|/|H|$. Their ratio is
$1/|H\cap[G,G]|$. It equals $1/f$ only if $[G,G]\subseteq H$; every subgroup
containing $[G,G]$ is normal. Otherwise the ratio is at least $2/f$. These facts
apply to the specified uniform-coset model. For learned outputs, we also
check class accuracy and the per-example distribution; a ratio alone neither
establishes uniformity nor excludes arbitrary hidden encodings.

\subsection{Asymptotic order-blind accuracy}
\label{app:orderblind-asymptotic}

We prove Theorem~\ref{thm:orderblind-asymptotic} by showing that the product
distribution conditional on the input multiset approaches the uniform
distribution on its abelianization class, on average over multisets.
Throughout, $G$ is fixed and finite, $K=[G,G]$, $f=|K|$, and the inputs
$X_1,\ldots,X_t$ are independent and uniform on $G$. Write
$q_t=X_1\cdots X_t$ and let $M_t$ be their multiset. For each possible $M$, set
\[
D_M(g)=\Pr(q_t=g\mid M_t=M).
\]
By Lemma~\ref{lem:reorder}, $D_M$ is supported on a unique $K$-coset
$F(M)$, although its support need not fill that coset. Let $U_M$ be uniform
on $F(M)$ and therefore we have $U_M(g)  = 1/f$ for $g\in F(M)$ and $0$ otherwise. We define
\[
\operatorname{TV}(\mu,\nu)=\frac12\sum_{g\in G}|\mu(g)-\nu(g)|,
\qquad
\Delta_t=\mathbb E_M\operatorname{TV}(D_M,U_M).
\]
\subsubsection{An explicit finite-length bound}
\label{app:orderblind-l2}
We measure the remaining within-class bias by its mean squared distance
from uniform,
\begin{equation}
\Phi_t=\mathbb E_M\sum_{g\in G}\bigl(D_M(g)-U_M(g)\bigr)^2 =\mathbb E_M\sum_{g\in F(M)}\bigl(D_M(g)-U_M(g)\bigr)^2.
\label{eq:orderblind-l2-definition}
\end{equation}
We first relate this quantity to prediction accuracy, then compute it
exactly using the dimensions of the irreducible representations of $G$.
This yields both group-specific bounds and the group-independent bound
in Theorem~\ref{thm:orderblind-asymptotic}.

\paragraph{From squared bias to accuracy.}
For $f>1$, the following bounds hold:
\begin{equation}
\Phi_t\le C_t-\frac1f\le\sqrt{(1-1/f)\Phi_t},
\qquad
\Delta_t\le\frac{\sqrt f}{2}\sqrt{\Phi_t}.
\label{eq:orderblind-l2-sandwich}
\end{equation}
To prove the lower bound, write $p_g=D_M(g)$ on the $f$ elements of $F(M)$.
Then $\sum_g(p_g-1/f)^2=\sum_gp_g^2-1/f\le\max_gp_g-1/f$.
For the upper bound, let $v_g = D_M(g) - 1/f$. We use $v = \langle v_g \mid g\in F(M)\rangle$ and $u = \langle 1/f, \dots, 1/f\rangle$ to represent vectors in $\mathbb{R}^f$. We choose element $j$ which maximizes $v_j$. Since $\sum_gv_g=0$, we have
\[
\max_gp_g-1/f=\langle v,e_j-u\rangle
\le\sqrt{1-1/f}\,\|v\|_2.
\]
Here $e_j$ is the point mass at element $j$. Cauchy--Schwarz also gives
$\operatorname{TV}(D_M,U_M)\le\sqrt f\,\|v\|_2/2$.
Averaging and applying Jensen's inequality proves
Eq.~\eqref{eq:orderblind-l2-sandwich}. If $f=1$, all three errors are zero.

\paragraph{An exact expression for the squared bias.}
Let $\widehat G$ be a complete set of inequivalent irreducible unitary
complex representations of $G$. For $\pi\in\widehat G$, write $d_\pi$ for
its dimension and $\chi_\pi(g)=\operatorname{tr}\pi(g)$ for its character.
We will show that, for every $t\ge1$,
\begin{equation}
\Phi_t=\frac1{|G|}\sum_{\substack{\pi\in\widehat G\\d_\pi>1}}
d_\pi^{\,1-t}
\left[\binom{d_\pi+t+1}{t+2}-\binom{d_\pi}{t+2}\right],
\label{eq:orderblind-l2-closed}
\end{equation}
with $\binom{n}{k}=0$ for integers $k>n\ge0$. This is an exact expression
for $\Phi_t$, not for the optimal accuracy $C_t$. Together with
Eq.~\eqref{eq:orderblind-l2-sandwich}, it gives computable bounds on $C_t$.

Define $\widehat\mu(\pi)=\sum_g\mu(g)\pi(g)$ and
$\|A\|_{\mathrm{HS}}^2=\operatorname{tr}(AA^*)$.
The finite-group Plancherel identity gives
\[
\sum_g|D_M(g)-U_M(g)|^2
=\frac1{|G|}\sum_{\pi\in\widehat G}d_\pi
\|\widehat D_M(\pi)-\widehat U_M(\pi)\|_{\mathrm{HS}}^2.
\]
The one-dimensional representations are trivial on $K$ and constant on
$F(M)$, so their terms vanish. Conversely, an irreducible representation
trivial on $K$ factors through the abelian group $G/K$ and is
one-dimensional. For each remaining representation, averaging $\pi(k)$
over $k\in K$, given by $\frac{1}{f}\sum_{k\in K}\pi(k)$, projects vectors to $K$-fixed vectors. The subspace spanned by all $K$-fixed vectors is
$G$-invariant because $K$ is normal, and must be either zero of the whole space by irreducibility. If it is the whole space, then the representation is trivial on $K$ and hence one dimensional.
Thus $\widehat U_M(\pi)=0$ when $d_\pi>1$, and
\begin{equation}
\Phi_t=\frac1{|G|}\sum_{d_\pi>1}d_\pi
\mathbb E_M\|\widehat D_M(\pi)\|_{\mathrm{HS}}^2.
\label{eq:orderblind-l2-fourier}
\end{equation}

Let $q$ and $q'$ be products of independent uniform orderings conditional
on the same multiset. Since $\widehat D_M(\pi)=\mathbb E[\pi(q)\mid M]$, we have
\[
\mathbb E_M\|\widehat D_M(\pi)\|_{\mathrm{HS}}^2
=\mathbb E\chi_\pi(qq'^{-1}).
\]
Equivalently, draw independent uniform $x_1,\ldots,x_t\in G$ and a uniform
permutation $\theta\in S_t$, and set
$q=x_1\cdots x_t$ and $q'=x_{\theta(1)}\cdots x_{\theta(t)}$.
Repeated values cause no bias because each distinct ordering of a multiset
has the same number of labeled permutations.

For a fixed $\theta$ and an irreducible representation of dimension $d$,
expand the trace of
$\pi(x_1\cdots x_t x_{\theta(t)}^{-1}\cdots x_{\theta(1)}^{-1})$.
Each independent input appears once as a matrix and once as its adjoint.
Schur orthogonality gives
\[
\mathbb E_x\bigl[\pi(x)_{ij}\overline{\pi(x)_{k\ell}}\bigr]
=d^{-1}\delta_{ik}\delta_{j\ell}.
\]
The $t$ averages contribute $d^{-t}$ and identify matrix indices. If
$L(\theta)$ is the number of free index classes after these identifications,
summing over them gives
\begin{equation}
\mathbb E_{x_1,\ldots,x_t}\chi_\pi(qq'^{-1})=d^{L(\theta)-t}.
\label{eq:orderblind-index-loops}
\end{equation}

We make the combinatorial identification explicit. Place the $2t$ factors
around a polygon, with cyclic indices $i_1,\ldots,i_{2t}$ at its vertices.
The inverse of $x_j$ occupies edge
$\alpha_j=2t+1-\theta^{-1}(j)$. Its orthogonality constraints are
$i_j=i_{\alpha_j+1}$ and $i_{j+1}=i_{\alpha_j}$, with indices read cyclically.
These are exactly the endpoint identifications obtained by gluing each
positive edge to its oppositely oriented inverse edge. The positive edges
form one contiguous half of the polygon and the inverse edges the other.
Consequently, $L(\theta)$ is the number of vertices of the polygon gluing
enumerated by the Hultman numbers
\citep[Theorem~1]{alexeev2011hultman}; the matching runs over all permutations
as $\theta$ does. In particular, $L(\theta)\le t+1$.

Writing $c(n,k)$ for the number of permutations of $n$ objects with $k$
cycles, that enumeration gives
\[
\text{Pr}_\theta(L(\theta)=k)=
\begin{cases}
2c(t+2,k)/(t+2)!, & k\equiv t+1\pmod2,\\
0, & \text{otherwise}.
\end{cases}
\]
Using $\sum_kc(n,k)d^k=d(d+1)\cdots(d+n-1)$ and selecting the indicated
parity yields
\begin{equation}
\mathbb E_\theta d^{L(\theta)-t}
=d^{-t}\left[\binom{d+t+1}{t+2}-\binom{d}{t+2}\right].
\label{eq:orderblind-loop-average}
\end{equation}
This is also the coefficient formula in
\citet[Theorem~3(i)]{alexeev2011hultman}. Substituting into
Eq.~\eqref{eq:orderblind-l2-fourier} proves Eq.~\eqref{eq:orderblind-l2-closed}.

\paragraph{A bound requiring only class size.}
Since $L(\theta)-t-1\le0$ and $d_\pi\ge2$ in the remaining sum,
\begin{align*}
\Phi_t
&=\frac1{|G|}\sum_{d_\pi>1}d_\pi^2
  \mathbb E_\theta d_\pi^{L(\theta)-t-1}\\
&\le\frac1{|G|}\sum_{d_\pi>1}d_\pi^2
  \mathbb E_\theta 2^{L(\theta)-t-1}
 =\left(1-\frac1f\right)\frac{t+3}{2^{t+1}}.
\end{align*}
Here $\sum_{\pi\in\widehat G}d_\pi^2=|G|$, exactly $|G/K|=|G|/f$
representations are one-dimensional, and the last expectation follows
from Eq.~\eqref{eq:orderblind-loop-average} at $d=2$.
Combining this with Eq.~\eqref{eq:orderblind-l2-sandwich} proves the explicit
bound and convergence in Theorem~\ref{thm:orderblind-asymptotic}.

\paragraph{Why the optimal accuracy cannot increase with length.}
If $m_x$ is the count of $x$ in $M$, conditioning on the last input gives
\[
D_M(g)=\sum_{x:m_x>0}\frac{m_x}{t}D_{M-\{x\}}(gx^{-1}).
\]
Taking the maximum and using its convexity bounds $\max_gD_M(g)$ by
the same weighted average of $\max_hD_{M-\{x\}}(h)$.
Deleting a uniformly chosen occurrence from an i.i.d. multiset of length
$t$ leaves an i.i.d. multiset of length $t-1$. Averaging therefore gives
$C_t\le C_{t-1}$ for $t\ge2$, completing the theorem.
The same argument proves $\Delta_t\le\Delta_{t-1}$, since translating
$U_{M-\{x\}}$ on the right by $x$ gives $U_M$, and total variation is convex.

\subsubsection{An elementary proof of convergence}
\label{app:orderblind-block-proof}
We give an independent proof of convergence that avoids representation theory. Under the standing assumption of uniform i.i.d. inputs from \(G\), we establish constants \(L=L(G)\ge1\) and \(\rho=\rho(G)\in(0,1)\) such that
\begin{equation}
0\le C_t-\frac1f
\le\Delta_t
\le\left(1-\frac1f\right)\rho^{\lfloor t/L\rfloor}.
\label{eq:orderblind-tv-bound}
\end{equation}
The same argument extends to any fixed i.i.d. input law with full support on \(G\), although \(\rho\) then also depends on that law. We give this extension at the end of the proof.

For abelian $G$, the multiset determines the product, so $C_t=1$ and
$\Delta_t=0$; take $L=1$ and $\rho=1/2$. Below assume $f>1$.

\paragraph{A multiset whose products cover the commutator subgroup.}
We first construct a fixed multiset $M_*$ whose reorderings attain every
element of $K$. Use the convention $[a,b]=aba^{-1}b^{-1}$. Every $k\in K$
is a finite product of such commutators, since $[a,b]^{-1}=[b,a]$.
Choose a word $W_k$ formed by concatenating the corresponding quadruples
$(a,b,a^{-1},b^{-1})$; take $W_e$ to be empty. Reordering each quadruple as
$(a,a^{-1},b,b^{-1})$ gives a word $W_k^0$ with the same multiset and
product $e$.

Concatenate all $W_k$ in a fixed order and call the resulting multiset
$M_*$, of length $L>0$. To obtain product $k$, retain its segment $W_k$
and replace every other segment $W_h$ by $W_h^0$. Conversely, every
reordering has trivial image in $G/K$. Thus its possible products are
exactly $K$. Conditional on $M_*$, every distinct word has positive, equal
probability: counting labeled permutations assigns each word the same
multiplicity $\prod_x m_x!$, where $m_x$ is the count of $x$.

Let $D_*$ be this conditional product distribution. Since $D_*(k)>0$ for
every $k\in K$, we can set
\[
\varepsilon=f\min_{k\in K}D_*(k)\in(0,1].
\]
If $\varepsilon<1$, then
$D_*=\varepsilon U_K+(1-\varepsilon)R$ for a probability distribution
$R$ on $K$. If $\varepsilon=1$, then $D_*=U_K$.

\paragraph{A special block contracts the error.}
For probability distributions on $G$, define the ordered convolution
\[
(\mu*\nu)(g)=\sum_{h\in G}\mu(h)\nu(h^{-1}g),
\]
the law of the product of independent draws from $\mu$ and $\nu$ in that
order. If $\mu$ is supported on a coset $F=cK$ and $\nu$ on $F'=c'K$,
then $\mu*U_K=U_F$ and $U_F*\nu=U_{FF'}$. The second identity uses
normality of $K$, which gives $U_K*\delta_{c'}=\delta_{c'}*U_K$.
Right convolution is a Markov kernel and contracts total variation, so
\[
\operatorname{TV}(\mu*\nu,U_{FF'})
=\operatorname{TV}(\mu*\nu,U_F*\nu)
\le\operatorname{TV}(\mu,U_F).
\]
For a special block with law $D_*$ and $\varepsilon<1$,
\[
\mu*D_*-U_F=(1-\varepsilon)(\mu-U_F)*R,
\]
and hence
\[
\operatorname{TV}(\mu*D_*,U_F)
\le(1-\varepsilon)\operatorname{TV}(\mu,U_F).
\]
If $\varepsilon=1$, that block makes the error zero immediately. Subsequent
blocks preserve uniformity on the resulting coset.

\paragraph{Conditioning on blocks and then forgetting their boundaries.}
Partition the inputs into \(n=\lfloor t/L\rfloor\) complete blocks of length \(L\) and a remainder of length \(r=t-nL\). Let \(M^{(1)},\ldots,M^{(n)}\) be the multisets of the complete blocks, listed in block order, and let \(M^{(\mathrm{rem})}\) be the multiset of the remainder, empty when \(r=0\). Define
\[
Z=\bigl(M^{(1)},\ldots,M^{(n)},M^{(\mathrm{rem})}\bigr).
\]
Thus \(Z\) determines the total input multiset \(M_t\). Given $Z$, the blocks remain independent:
each conditioning event involves only that block's independent input
coordinates. Within each block, the distinct orderings are equiprobable.
Thus the conditional law $D_Z$ of $q_t$ is the ordered convolution of
the block laws. An empty remainder has law $\delta_e$.

Let $N$ count the complete blocks with multiset $M_*$. The initial
distance $\operatorname{TV}(\delta_e,U_K)$ is $1-1/f$. Each such special
block contracts it by at most $1-\varepsilon$; other blocks, including
the remainder, cannot increase it. Therefore
\[
\operatorname{TV}(D_Z,U_{M_t})\le(1-1/f)(1-\varepsilon)^N.
\]
When $\varepsilon=1$, the factor is interpreted as $1$ if $N=0$ and as
$0$ otherwise. We have used independence conditional on the block
multisets, not independence conditional on the total multiset.

A complete block has multiset $M_*$ with probability
\[
p=\frac{L!}{\prod_x m_x!}|G|^{-L}>0.
\]
Since $M_*$ contains a nonidentity element, it differs from the
all-identity multiset, so $p\le1-|G|^{-L}<1$. The complete blocks are
independent before conditioning, giving $N\sim\operatorname{Binomial}(n,p)$.
It follows that
\[
\mathbb E_Z\operatorname{TV}(D_Z,U_{M_t})
\le(1-1/f)\mathbb E(1-\varepsilon)^N
=(1-1/f)(1-p\varepsilon)^n.
\]
Since $Z$ determines $M_t$, by the tower property,
$D_{M_t}(g)=\mathbb E[D_Z(g)\mid M_t]$. Since $U_{M_t}$ is fixed given
$M_t$, convexity of total variation gives
\[
\Delta_t\le\mathbb E_Z\operatorname{TV}(D_Z,U_{M_t})
\le(1-1/f)(1-p\varepsilon)^{\lfloor t/L\rfloor}.
\]
Set $\rho=1-p\varepsilon\in(0,1)$. If $t<L$, this reduces to the
general bound $1-1/f$ for a distribution supported on $f$ points.
Finally, for every $M$ the largest atom of $D_M$ is at least $1/f$,
and the single-event bound for total variation yields
\[
0\le\max_g D_M(g)-1/f\le\operatorname{TV}(D_M,U_M).
\]
Averaging proves Eq.~\eqref{eq:orderblind-tv-bound}, providing an
independent proof of convergence.

For nonuniform inputs, suppose the inputs are instead i.i.d. with a fixed law \(\nu\) satisfying \(\nu(x)>0\) for every \(x\in G\). Every ordering of a multiset with counts \((m_x)_{x\in G}\) has the same probability \(\prod_{x\in G}\nu(x)^{m_x}\). Consequently, conditional on a multiset, its distinct orderings remain equiprobable. The construction of \(M_*\), its conditional product law \(D_*\), and the contraction factor \(1-\varepsilon\) are therefore unchanged.

The probability that a complete block has multiset \(M_*\) becomes
\[
p_\nu=\frac{L!}{\prod_{x\in G}m_x!}
\prod_{x\in G}\nu(x)^{m_x}>0.
\]
Since \(M_*\) differs from the all-identity multiset,
\[
p_\nu\le1-\nu(e)^L<1.
\]
Repeating the block argument proves Eq. (8) with \(C_t\) and \(\Delta_t\) evaluated under \(\nu\), and with
\[
\rho_\nu=1-p_\nu\varepsilon\in(0,1).
\]
Thus \(L\) and \(\varepsilon\) can be chosen to depend only on \(G\), while the convergence rate \(\rho_\nu\) also depends on the input law.

\paragraph{Consequences for an order-blind predictor.}
Let $\widehat q_t=\psi_t(M_t,\omega_t)$, where the external randomness
$\omega_t$ is independent of all test inputs; deterministic predictors
are included. Define its exact and abelianization accuracies by
\[
A_t=\Pr(\widehat q_t=q_t),\qquad
Q_t=\Pr(\widehat q_tK=q_tK).
\]
Conditional on $M_t$, the prediction is independent of $q_t$. Writing
$r_t(g\mid M)=\Pr(\widehat q_t=g\mid M_t=M)$, we obtain
\[
A_t=\mathbb E_M\sum_g D_M(g)r_t(g\mid M),\qquad
Q_t/f=\mathbb E_M\sum_g U_M(g)r_t(g\mid M).
\]
Since $0\le r_t(g\mid M)\le1$ and $D_M-U_M$ has total mass zero,
\begin{equation}
|A_t-Q_t/f|\le\Delta_t
\le\frac{\sqrt f}{2}\sqrt{\Phi_t}.
\label{eq:orderblind-predictor-bound}
\end{equation}
If $Q_t\to1$, then $A_t\to1/f$ and $C_t-A_t\to0$. This conclusion
does not require uniform predictor outputs. It is the conditional law of
the true product that becomes uniform on average.

\paragraph{Scope of the limit.}
Uniformization holds in expectation over multisets, not for every multiset:
the all-identity multiset determines the product at every length. The group
is finite and inputs are uniform on the full group for the explicit bound
in Theorem~\ref{thm:orderblind-asymptotic}; its excess bound is uniform over
finite groups. The elementary block proof additionally applies to a fixed
nonuniform i.i.d. input law with full support on $G$, with $\rho$ depending
on that law. Its constants may be loose. Neither argument applies
unconditionally to restricted alphabets or infinite groups: sampling only
$e$ in a non-abelian group gives $C_t=1$, not $1/f$.
App.~\ref{app:ceiling} retains the finite-length calculations.
Finally, finite-length reordering tests do not prove strict order-blindness
at arbitrary lengths or $Q_t\to1$ for a learned model.

\section{Non-normal coset stages and their internal representation}
\label{app:nonsolvable-census}

\subsection{Training cohorts and stage identification}

We extend the recurrent census of App.~\ref{app:blind-census} to $A_5$ and
$S_5$. All inputs are independent and uniform over the full group, training
length is 100, and the reported accuracies average positions 17--100. We
use two-layer LSTM and GRU models matched to the four-layer recipe-A
Transformer's parameter count, in fp32. On $A_5$, embedding and hidden
widths are 442 for LSTM and 510 for GRU, giving 3,185,996 and 3,188,580
parameters against the Transformer's 3,190,272. On $S_5$, the GRU width is
508, with 3,224,904 parameters against 3,220,992. We use AdamW, batch size
256, zero weight decay, gradient clipping at 1, and linear decay scheduled
over 74,219 updates without warmup; the runs below stop earlier.
The $A_5$ LSTMs use learning rate $10^{-3}$,
8,000 updates, and evaluations every 50 updates. The GRU arms that show the
stages below use $3\times10^{-4}$, 16,000 updates, and evaluations every 100
updates. Each arm has seeds 42--44. These are descriptive experiments.
The lower GRU learning rate was added after inspecting the $10^{-3}$ arm;
it is not a preregistered architecture comparison.

We identify stages from prediction errors rather than from accuracy alone.
For each incorrect argmax prediction $p$ of state $q$, define
$h_R=pq^{-1}$ and $h_L=q^{-1}p$. We normalize the counts of $h_R$ over
incorrect predictions to obtain a distribution $D_R$. For each candidate
subgroup $H$, we compute
\[
 d_H=\frac12\sum_{g\in G}
 \left|D_R(g)-\frac{\mathbf 1\{g\in H\setminus\{e\}\}}{|H|-1}\right|.
\]
We enumerate all 59 subgroups of $A_5$ and all 156 of $S_5$. A partial stage
requires a best-fitting nontrivial proper subgroup with $d_H<.15$, exact
accuracy within 15\% of $1/|H|$, and an interval lasting at least 300
updates. We also check the competing subgroup fits. The thresholds were
chosen during the earlier descriptive analysis, and are not presented as
preregistered. A stage in this sense describes the pooled errors and
accuracy; it does not by itself imply uniform output probabilities on
every sequence. The blind row-partition tests in App.~\ref{app:blind-census}
provide a separate diagnostic.

\begin{table}[!htbp]
\centering\small
\begin{tabular}{llcll}
\toprule
Group & Model & Seed & Sustained subgroup orders & Final exact accuracy \\
\midrule
$A_5$ & LSTM & 42 & $3$ & .9996 \\
 & LSTM & 43 & $3$ & .9923 \\
 & LSTM & 44 & $3$ (two intervals) & .0877$^{*}$ \\
 & GRU & 42 & $3$ & .9938 \\
 & GRU & 43 & $3$ & .9961 \\
 & GRU & 44 & $3$ & .9950 \\
$S_5$ & LSTM & 42 & $60\to10\to2$ & .4997 \\
 & LSTM & 43 & $24\to6\to2$ & 1.0000 \\
 & LSTM & 44 & $60\to4$ & .4868 \\
 & GRU & 42 & $60\to12\to3$ & .9973 \\
 & GRU & 43 & $60\to12\to3$ & .9996 \\
 & GRU & 44 & $60\to12\to3$ & .3317 \\
\bottomrule
\end{tabular}
\caption{Recurrent stages in the reported learning-rate arms. Orders refer
to the unresolved subgroup $H$, not the number of cosets. Only the order-60
stages are nontrivial normal quotients. $^{*}$The third $A_5$ LSTM reaches
approximately .99 accuracy and subsequently loses it; its final accuracy
must not be read as failure to reach an accurate state earlier.}
\label{tab:nonsolvable-stages}
\end{table}

On $A_5$, the three-element subgroups fix two objects pointwise; each has
20 right cosets. The sustained intervals are updates 1,450--2,350,
1,250--1,900, and 1,150--1,900 plus 2,100--2,650 for the three LSTMs, and
4,500--6,400, 4,100--9,200, and 3,600--7,400 for the three GRUs. The largest
accepted $d_H$ is .099 for LSTMs and .149 for GRUs; the next-best subgroup
is at least .600 away within these intervals. The third LSTM reaches
approximately .99 accuracy at updates 4,000--5,000 before becoming unstable.
On $S_5$, all three lower-learning-rate GRUs first recover the $A_5$ quotient,
then the right cosets of an order-12 subgroup fixing one object and an
order-3 subgroup fixing two. Two reach near-exact tracking within 16,000
updates; the third remains at the order-3 stage.

The other learning-rate arms matter for interpreting this result. At
$10^{-3}$, the three $A_5$ GRUs have final accuracy .034--.036 even after
40,000 updates, without a fitting partial coset stage; the three $S_5$ GRUs
retain the $A_5$ quotient with exact accuracy .02--.04. Thus the observed
normal/non-normal distinction is not a claim that every recurrent model
necessarily leaves a quotient stage.

\paragraph{Transformer coverage.}
Figure~\ref{fig:long-training-stages}c shows five four-layer recipe-A
$S_5$ models with width 256, four heads, and 3,220,992 parameters, matching
the configuration in Table~\ref{tab:law}. Inputs are independent and uniform
over all 120 group elements. Seeds 42--46 use AdamW at $5\times10^{-5}$,
zero weight decay, linear decay over 74,219 updates without warmup, batch
size 256, training length 100, and bf16 autocast. Using FP32 census readouts,
we evaluate 2,048 sequences
at positions 17--100 every 50 updates and at the final update, giving 1,486
evaluations per run. Scoring all 156 subgroups with the fit and accuracy
criteria above detects no non-normal coset at any of the 7,430 evaluations;
all five runs end in an accepted $A_5$ stage. Stage acceptance uses the same
300-update duration as the recurrent census. This dense census addresses
the possibility of missing short stages on the older 1,000-update grid.
The first 20,000 updates form the shared budget window with the lower-rate
recurrent controls; only two of these five Transformers have entered an
accepted $A_5$ stage by then. Results at 74,219 updates are budget extensions,
not a matched-time comparison. Seeds 42--44 and 45--46 run on different GPU
models; these are five new trajectories, not bitwise replays of Table~\ref{tab:law}.

The 41-model cross-group census is reported separately in
App.~\ref{app:blind-census}. Here we add three 12-layer $S_5$ trajectories
with width 256, four heads, and 9,539,072 parameters. They use AdamW at
$5\times10^{-5}$, zero weight decay, linear learning-rate decay over 74,219
updates, batch size 256, and bf16 autocast. We evaluate every 1,000 updates
and at the final update, giving 75 sampled times per run and 225 in total.
The sustained-stage duration is 3,000 updates at this coarser sampling rate.
All detected nontrivial coset stages correspond to the normal subgroup
$A_5$. One run's exact accuracy rises to .029 near the end, with increasing
preference for the true member but no finer subgroup partition. None of the
three reaches exact tracking. These are new trajectories with the original
seeds, not bitwise replays of the earlier training histories.

A separate census covers 108 final $S_5$ evaluations, 36 each at depths 4,
8, and 12. Of these, 89 uniquely fit the $A_5$ stage under the subgroup and
accuracy criteria, and 19 fit no partial subgroup stage. We recover no
non-normal coset partitions. Predictions that fit no stage include
unstructured or increasingly concentrated outputs; they must not be counted
as quotient solutions. The 225 trajectory samples and 108 final evaluations
are different coverage summaries, not 333 independently trained models.

\paragraph{Reading Fig.~\ref{fig:long-training-stages}.}
Panels (a,b) use the same five-million-update $S_4$ run and the exact and
quotient definitions of App.~\ref{app:readouts}. Panel (c) displays one $S_5$
LSTM and one GRU, selected first by the number of distinct accepted subgroup
orders, then by total duration of stages that end before the last evaluation,
and finally by seed. This picks seeds 42 and 43, respectively, without using
whether the run solves the task. All five four-layer Transformer trajectories
from the dense census are shown. The recurrent and Transformer budgets and evaluation intervals differ
as specified above.

Panel (d) groups stages by group and conjugacy type of $H$. We first average
each reading over all evaluations in a stage, retaining its first and last
evaluation. Small dots show these stage means, with horizontal offsets for
visibility. Large dots give their equally weighted mean; bars show one sample
standard deviation across stage means, using denominator $m-1$ for $m$ stages.
We omit the standard deviation when $m=1$ and do not truncate bars at zero or
one. Stages can come from the same run, so this spread is descriptive, not a
confidence interval or an estimate of variation across independent seeds.
Table~\ref{tab:coset-stage-spread} retains the full minimum--maximum range over
all evaluations, including boundaries. In particular, the $A_5$ reading reaches
.8510 in one accepted stage even though its stage average is higher.
We report the fractions of
incorrect predictions with $h_R\in H$ and $h_L\in H$. If $p=hq$, then
$h_L=q^{-1}hq$. If, conditional on an error, $q$ and $h$ are independently
uniform on $G$ and $H\setminus\{e\}$, the predicted second fraction is
\[
 \sum_{\mathcal C}
 \frac{|(H\setminus\{e\})\cap\mathcal C|}{|H|-1}
 \frac{|H\cap\mathcal C|}{|\mathcal C|},
\]
where the sum runs over conjugacy classes $\mathcal C$ of nonidentity elements.
This idealized reference does not fit an additional parameter. A normal $H$
gives one, whereas a non-normal $H$ generally gives a smaller value. For an
order-two subgroup generated by a double transposition, it gives $1/3$ in
$S_4$ and $1/15$ in $S_5$; a transposition in $S_5$ gives $1/10$.

\begin{table}[t]
\centering
\small
\setlength{\tabcolsep}{4pt}
\caption{Stage summaries and full evaluation ranges for Fig.~\ref{fig:long-training-stages}d. $R$ is the fraction of errors with $pq^{-1}\in H$; $L$ uses $q^{-1}p\in H$. Means and sample SDs weight stages equally. Ranges include every evaluation within accepted stages. A dash indicates a single stage, for which SD is undefined. Primes match the subgroup types in the figure.}
\label{tab:coset-stage-spread}
\begin{tabular}{lrrrrrr}
\toprule
Group & $|H|$ & Stages & $R$: mean $\pm$ SD & $R$: range & $L$: mean $\pm$ SD & $L$: range \\
\midrule
$S_4$ & $2$ & 9 & $0.9896\pm0.0068$ & $0.8935$--$1.0000$ & $0.3337\pm0.0023$ & $0.3055$--$0.3382$ \\
$A_5$ & $3$ & 7 & $0.9718\pm0.0178$ & $0.8510$--$0.9998$ & $0.0979\pm0.0024$ & $0.0856$--$0.1025$ \\
$S_5$ & $2$ & 1 & $0.9903$ (---) & $0.9589$--$0.9999$ & $0.0660$ (---) & $0.0638$--$0.0676$ \\
$S_5$ & $2'$ & 1 & $0.9922$ (---) & $0.9215$--$0.9997$ & $0.0989$ (---) & $0.0951$--$0.1010$ \\
$S_5$ & $3$ & 3 & $0.9920\pm0.0029$ & $0.8800$--$1.0000$ & $0.1003\pm0.0003$ & $0.0962$--$0.1016$ \\
$S_5$ & $4$ & 1 & $0.9912$ (---) & $0.8880$--$0.9994$ & $0.1997$ (---) & $0.1832$--$0.2018$ \\
$S_5$ & $6$ & 1 & $0.9936$ (---) & $0.9735$--$0.9994$ & $0.2195$ (---) & $0.2172$--$0.2214$ \\
$S_5$ & $10$ & 1 & $0.9947$ (---) & $0.9739$--$1.0000$ & $0.2584$ (---) & $0.2559$--$0.2600$ \\
$S_5$ & $12$ & 3 & $0.9951\pm0.0049$ & $0.9774$--$1.0000$ & $0.3459\pm0.0009$ & $0.3433$--$0.3476$ \\
$S_5$ & $24$ & 1 & $0.9966$ (---) & $0.9908$--$0.9995$ & $0.3751$ (---) & $0.3745$--$0.3757$ \\
$S_5$ & $60$ & 5 & $1.0000\pm0.0000$ & $0.9983$--$1.0000$ & $1.0000\pm0.0000$ & $0.9983$--$1.0000$ \\
\bottomrule
\end{tabular}
\end{table}

\subsection{Coset tracking with recirculation}
\label{app:recirculation}

We test whether non-normal coset stages also occur in a Transformer with
recurrent feedback. Following the recirculation construction of
\citet{mozer2026recirculation}, we feed the preceding position's final-block
residual stream into the current position before its first block. The source
vector is rescaled to the destination's norm, then mixed with weights
$\alpha=.25$ and $\beta=.75$. Each of $r$ feedback rounds uses the preceding
round's representations, for $r+1$ forward passes per training update.
Specifically, let $d_t^{(k)}$ and $s_t^{(k)}$ be the input to the first block
and output of the last block at position $t$ in pass $k$. After an ordinary
pass $k=0$, each pass $k=1,\ldots,r$ uses
\[
d_t^{(k)}=\alpha\,
\frac{\|d_t^{(k-1)}\|_2}{\max(\|s_{t-1}^{(k-1)}\|_2,10^{-6})}
s_{t-1}^{(k-1)}+\beta d_t^{(k-1)},\qquad t\ge2,
\]
then recomputes the blocks with shared weights. Position 1 remains
$d_1^{(k)}=d_1^{(0)}$. The $\beta=1$ control retains $\alpha=.25$.
Feedback is active during both training and evaluation, with gradients
through it. This increases computation and effective depth; the comparison
does not isolate recurrence from those changes.

We use recipe A on $S_4$, including its full-group i.i.d. inputs, batch size
256, learning rate $5\times10^{-5}$ with linear decay, and 74,219-update
budget. We train in bf16 and evaluate the coset census in fp32 on 2,048
sequences over positions 17--100 every 250 updates. We score all 30 subgroups
using the pooled-error criterion above. A stage must last at least 300
updates, so two observations 250 updates apart alone do not suffice.
These are descriptive experiments, separate from the standard-Transformer
census. Table~\ref{tab:recirculation} reports all 11 completed runs in this
comparison, including three follow-ups, and the original interrupted trajectory.
The four-round seed-42 follow-up restarts from initialization with the same
seed; it is a repeated run, not a new independent seed.
The shuffled-source control was added after observing the non-normal stage.
An earlier attempted one-round shuffle used a different, serial update schedule
and is excluded; the reported control matches the four-round forward-pass count.
It is therefore a post-result control, not an independent confirmation.

\begin{table}[!htbp]
\centering\small
\begin{tabular}{@{}llcc@{}}
\toprule
Configuration & Seed & Sustained subgroup & Last census exact \\
\midrule
No feedback & 42 / 43 / 44 & $A_4$ & .0825 / .0836 / .0808 \\
One round & 42 / 43 & $A_4$ & .0849 / .0840 \\
One round & 44 & $V_4$ & .2466 \\
Four rounds, restart & 42 & $D_4$ then $V_4$ & .2709 \\
Four rounds & 43 & none & .1398 \\
Four rounds & 44 & $V_4$ & .5412 \\
Four rounds, $\beta=1$ & 42 & $A_4$ & .0841 \\
Four rounds, shuffled source & 42 & $A_4$ & .0842 \\
Four rounds, interrupted & 42 & $D_4$ (non-normal) & .1986 \\
\bottomrule
\end{tabular}
\caption{Recirculation census on $S_4$. Completed runs reach 74,219 updates;
the interrupted run's last census is at 38,750. A listed stage can precede
the last evaluation. The shuffled control cyclically shifts source vectors
across the batch, preserving the number of feedback rounds.}
\label{tab:recirculation}
\end{table}

In the interrupted four-round trajectory, an order-eight non-normal subgroup
$H\cong D_4$ fits 59 consecutive evaluations from updates 14,500 to 29,000.
It defines three right cosets with eight states each. Throughout this stage,
exact accuracy is .1088--.1420, the pooled-error fit has total variation at
most .1368, and the next-best subgroup fit has total variation at least
.5589. At update 19,000, exact accuracy is 12.39\% and correct-coset accuracy
is 99.60\%. The run later leaves the $1/8$ stage and was accidentally
terminated after its 38,750-update evaluation. We retain the observed
interval as a descriptive result, but exclude this run from completed-run
counts. The complete same-seed restart also passes through this non-normal
$D_4$ stage, from updates 13,750 to 17,500 (16 evaluations). Exact accuracy
is .1087--.1241, total variation is at most .1362, and the next-best fit has
total variation at least .5595. It then enters a normal $V_4$ stage from
18,500 to 74,219; final exact and correct-$V_4$-coset accuracies are 27.09\%
and 99.92\%. Neither of the other two completed four-round seeds exhibits
a non-normal stage; seed 44 instead passes through $V_4$ from updates
7,250 to 30,000. The non-normal stage therefore repeats under the same
initialization seed but is not replicated across seeds.

The one-round follow-ups also differ. Seed 43 ends with an $A_4$ stage,
with exact accuracy 8.40\% and correct-coset accuracy 99.90\%. Seed 44
instead sustains a $V_4$ stage from updates 38,000 to 74,219, ending at
24.66\% exact and 98.94\% correct-coset accuracy. This contradicts our
initial prediction that one feedback round would remain at the
abelianization plateau. More feedback rounds are therefore not required
to recover a finer quotient in this setup.

The feedback remains part of the evaluated model. Removing it from the
interrupted run at updates 20,000 and 35,000 reduces exact accuracy to
.0413 and .0419, near the $1/24$ chance level. Conversely, adding four-round
feedback only at inference to a model trained without it gives .0425.
The observed coset structure therefore belongs to the jointly trained
weights and feedback configuration. It does not establish the same
structure in an ordinary Transformer without feedback.

\subsection{Locating the recurrent coset state}
\label{app:coset-mechanism}

We analyze the $A_5$ order-three stages in three LSTMs and three GRUs, at
three saved training times each. This gives 18 evaluations of six trained
models, not 18 independent runs. The replay trajectories recover the same
subgroups as the original census but need not match its training history
bitwise. In particular, changing the evaluation interval consumes random
numbers differently, and the recurrent kernels are not bitwise deterministic.
The geometry and intervention analyses were added after inspecting the stages.

For each model, we generate 4,096 fresh independent uniform sequences of
length 100. We fit class means, subspaces, and readouts on the first 3,072
sequences and evaluate on the remaining 1,024. The three sampled times within
a run use the same sequence pool. We take the first GRU layer's hidden state
and the second LSTM layer's concatenated hidden and cell states, following
layer-replacement diagnostics. Layers are numbered from one. We denote the
resulting row vector by $s_t$ to distinguish it from the group state $q_t$.

Let $c_t=Hq_t$ be the true coset. For each of the 20 values of $c_t$, we
average $s_t$ over fit-split positions 17--100 to obtain $\mu_c$. We subtract
the unweighted mean of these 20 vectors and stack them as the rows of $M$.
The right singular vectors of $M$ give orthonormal hidden-state directions.
We select $k$ by the largest adjacent singular-value ratio among the first
12 singular values and let $U_k$ contain the leading $k$ directions. This
rule was chosen after inspecting the spectra. It returns $k=3$ for all
three GRUs and LSTM seed 43, and $k=7$ for LSTM seeds 42 and 44, at all three
sampled times. It measures variation between class means, not the fraction
of total hidden-state variance explained.

\begin{table}[!htbp]
\centering\footnotesize
\begin{tabular}{llrrrrrr}
\toprule
Model & Step & $A_{\rm coset}$ & $k$ & Top $k$ & Rest & Full & Residual \\
\midrule
LSTM 42 & 1,500 & 0.987 & 7 & 0.987 & 0.985 & 0.988 & 0.00172 \\
LSTM 42 & 1,900 & 0.995 & 7 & 0.994 & 0.992 & 0.994 & 0.00155 \\
LSTM 42 & 2,100 & 0.993 & 7 & 0.993 & 0.989 & 0.992 & 0.00106 \\
LSTM 43 & 1,450 & 0.993 & 3 & 0.985 & 0.987 & 0.989 & 0.00056 \\
LSTM 43 & 1,600 & 0.999 & 3 & 0.999 & 0.997 & 0.999 & 0.00051 \\
LSTM 43 & 1,700 & 0.999 & 3 & 0.999 & 0.997 & 0.999 & 0.00047 \\
LSTM 44 & 1,400 & 0.998 & 7 & 0.998 & 0.999 & 0.998 & 0.00226 \\
LSTM 44 & 1,600 & 0.998 & 7 & 0.997 & 0.998 & 0.997 & 0.00208 \\
LSTM 44 & 1,850 & 0.992 & 7 & 0.984 & 0.988 & 0.985 & 0.00204 \\
GRU 42 & 5,000 & 0.932 & 3 & 0.900 & 0.891 & 0.916 & 0.00094 \\
GRU 42 & 5,500 & 0.971 & 3 & 0.942 & 0.943 & 0.960 & 0.00098 \\
GRU 42 & 6,000 & 0.979 & 3 & 0.958 & 0.957 & 0.974 & 0.00099 \\
GRU 43 & 4,800 & 0.961 & 3 & 0.930 & 0.935 & 0.953 & 0.00131 \\
GRU 43 & 5,600 & 0.985 & 3 & 0.974 & 0.974 & 0.984 & 0.00120 \\
GRU 43 & 6,300 & 0.987 & 3 & 0.977 & 0.976 & 0.982 & 0.00121 \\
GRU 44 & 4,600 & 0.974 & 3 & 0.949 & 0.949 & 0.960 & 0.00108 \\
GRU 44 & 5,600 & 0.987 & 3 & 0.981 & 0.974 & 0.988 & 0.00108 \\
GRU 44 & 6,500 & 0.994 & 3 & 0.978 & 0.979 & 0.989 & 0.00108 \\
\bottomrule
\end{tabular}

\caption{All 18 $A_5$ mechanism evaluations. $A_{\rm coset}$ is the fraction
of full-state argmax predictions in the correct right coset. $k$ is the
number of fitted directions. ``Top $k$'' reports donor agreement after
replacing those components; ``Rest'' reports recipient agreement after
replacing their complement; ``Full'' reports donor agreement after replacing
the full selected layer state. Residual is the mean relative squared error
of orthogonal fits to the class-mean action.}
\label{tab:coset-mechanism}
\end{table}

\subsection{Geometry and its scope}
\label{app:coset-geometry}

For an input $x$, let $a_x(c)$ be the successor coset under right
multiplication. Writing the projected centered means as row vectors $v_c$,
we fit an orthogonal matrix $R_x$ to minimize
$\sum_c\|v_cR_x-v_{a_x(c)}\|^2$. We divide by
$\sum_c\|v_{a_x(c)}\|^2$ to report relative squared error. The action of all
60 inputs has mean residual .00047--.00226 across the 18 evaluations.
At each evaluation, 200 arbitrary permutations of successor means give a
minimum residual at least .649; randomly relabeling the cosets gives at
least .847. These controls test whether a comparably good fit could be
obtained without the true group action. We also check class means computed
only from held-out sequences in the fitted basis.

In the three-dimensional cases, normalized pairwise inner products lie
near the dodecahedral values $\{\pm\sqrt5/3,\pm1/3,-1\}$. For 20 vertices,
the corresponding ordered-pair counts are $60,120,120,60,20$ in descending
order. This tests the geometry numerically rather than judging a projection
by eye. The traces of the fitted group matrices identify a three-dimensional
irreducible representation of $A_5$, meaning one with no proper nonzero
invariant linear subspace. Specifically, we compare the character
$\chi(x)=\operatorname{tr}(R_x)$ on conjugacy classes with the $A_5$
character table. The two inequivalent three-dimensional representations
give the same geometry with different group-element labels; we make no
claim that training prefers one. The seven-dimensional cases split into
three- and four-dimensional invariant subspaces. The latter carries the
geometry of tracking one object, and the two components together encode
the coset. We obtain each component from the eigenspace near one of
$(d/60)\sum_{x\in A_5}\chi(x^{-1})R_x$, where $\chi$ is its character and
$d=\chi(e)$ its dimension. This separates the components even when their
singular values overlap. We verify the multiplication law of the fitted mean-action
matrices and report both components in the intervention controls below.

The group action describes the class means much more accurately than it
describes every individual transition. On held-out
transitions ending at positions 17--100 at one sampled time per model, rotation of the previous
state has relative squared error .016--.085. Predicting the next class mean
directly has error .015--.066. Orthogonal maps fitted to the individual
transitions have mean group-composition defects of .033--.118, compared with
at most .0014 for the mean-action matrices. Here the defect is
$\|R_xR_y-R_{xy}\|_F/\sqrt{k}$, averaged over 300 sampled input pairs;
$\|\cdot\|_F$ is the square root of the sum of squared matrix entries.
For within-class deviations $\epsilon_t=v_t-v_{c_t}$, we fit a scalar
$\alpha$ in $\epsilon_t\approx\alpha\epsilon_{t-1}R_{x_t}$ by least squares.
The fitted factors .36--.60 describe contraction rather than rigid rotation. These measurements support an organized code of
class means, but not literal rotation of each individual recurrent state.

\paragraph{Why the three-dimensional code is dodecahedral.}
We give the exact algebraic statement underlying the approximate fits above.
Let $\rho:G\to O(d)$ be an orthogonal representation, so
$\rho(x)\rho(y)=\rho(xy)$, and let $v_{Hq}$ be row vectors satisfying
$v_{Hqx}=v_{Hq}\rho(x)$. Setting $w=v_H$ gives
\[
 v_{Hq}=w\rho(q),\qquad w\rho(h)=w\quad(h\in H).
\]
Conversely, any vector $w$ fixed by $H$ defines such a code: if $q'=hq$,
then $w\rho(q')=w\rho(h)\rho(q)=w\rho(q)$. Thus the code is the
\emph{orbit} of $w$, the set of its images under the group. Its vectors
distinguish all cosets exactly when the subgroup fixing $w$ is precisely $H$.
This construction requires a group action on the cosets, not a quotient
group structure.

For $G=A_5$, the irreducible representation dimensions are $1,3,3,4,5$;
the only one-dimensional representation is trivial. Hence three is the
smallest dimension of a nontrivial real linear action. In either
three-dimensional irreducible representation, the character
$\chi_3(g)=\operatorname{tr}\rho(g)$ is 3 at the identity and 0 on
3-cycles. For $H\cong C_3$, the subspace $V^H$ of vectors fixed by every
$h\in H$ therefore has dimension
\[
 \dim V^H=\frac{1}{|H|}\sum_{h\in H}\chi_3(h)
 =\frac{3+0+0}{3}=1.
\]
Indeed, the average of $\rho(h)$ over $H$ is the projection onto $V^H$,
so its trace counts that dimension. Both three-dimensional orthogonal
actions identify $A_5$ with the rotations of a regular dodecahedron,
with different labels for the rotations. A subgroup $C_3$ fixes the axis
through two opposite vertices. A nonzero vector on that axis has
stabilizer exactly $C_3$, and its orbit consists of the 20 vertices,
up to a common scale and an orthogonal change of coordinates. This both
constructs a three-dimensional code distinguishing all 20 cosets and
shows that no lower-dimensional linear action can do so.

Orthogonality matters for the Euclidean shape: a general invertible linear
change of coordinates can stretch the same orbit into a non-regular
polyhedron. Our empirical claim combines a three-dimensional subspace,
approximately orthogonal maps on class means, and the independent
pairwise-inner-product check. The exact argument does not establish an
exact update law for individual hidden states or explain why training
selects this representation.

\paragraph{The four-dimensional component and one-object tracking.}
In the two seven-dimensional LSTMs, the four-dimensional component
approximately groups cosets according to the location of one object.
The corresponding five group means form an approximate regular
4-simplex, a set of five equidistant vertices in four dimensions.
For example, in LSTM seed 42 at update 1,900, the normalized pairwise
inner products of these five means range from $-.30$ to $-.21$,
compared with $-1/4$ for a regular 4-simplex.

This shape also has an algebraic explanation, conditional on the
one-object grouping. Let $K\cong A_4$ fix that object. The
four-dimensional representation of $A_5$ acts by permuting coordinates
in $\{u\in\mathbb R^5:\sum_i u_i=0\}$. Its $K$-fixed vectors are multiples
of $e_j-\mathbf1/5$, whose orbit is a regular 4-simplex. Here $e_j$ is
the coordinate vector for the fixed object and $\mathbf1$ is the all-ones
vector. Equivalently, $K$ contains the identity, three double
transpositions, and eight 3-cycles, giving
\[
 \dim V_4^K=\frac{4+3\cdot0+8\cdot1}{12}=1,
 \qquad
 \dim V_3^K=\frac{3+3(-1)+8\cdot0}{12}=0.
\]
This additional grouping is substantive: $\dim V_4^{C_3}=2$, so
$C_3$-equivariance alone does not force a simplex in the four-dimensional
component. Both components are needed for the measured coset transfer,
as shown by the interventions below.

\subsection{State interventions and controls}
\label{app:coset-interventions}

At position 50, we form 2,000 ordered donor--recipient pairs with different
true cosets, sampled from the held-out sequence pool. Pairs can reuse
sequences. Let $s_r,s_d$ be the selected layer states and
$P_k=U_kU_k^\top$. Replacing the leading components means setting
\[
 s'_r=s_r+(s_d-s_r)P_k.
\]
The complementary intervention uses $I-P_k$ instead. Full-layer replacement
sets $s'_r=s_d$, while the no-intervention control leaves $s_r$ unchanged.
For LSTMs, the vector concatenates that layer's hidden and cell states, which
are split back after the edit; other layers remain the recipient's.

We then feed the recipient's remaining 50 inputs to the edited network.
For suffix product $w_j$ at offset $j$, the donor target is $H(q_dw_j)$ and
the recipient target is $H(q_rw_j)$. We compare the coset of the network's
full-state argmax with each target, averaging first over the 50 positions
and then over pairs. The two targets remain distinct under their common
suffix. The reported bars describe these paired measurements; repeated
positions and reused sequences are not independent experimental units.
Table~\ref{tab:coset-mechanism} gives every model and sampled time.

Replacing only the leading $k$ components gives donor agreement
.8997--.9990. The complementary replacement gives recipient agreement
.8908--.9990, while full-layer replacement gives donor agreement
.9158--.9995. The leading-component result is within .025 of full-layer
replacement in every evaluation. For the GRU in Fig.~\ref{fig:coset-state} (seed 42, update 6,000, fixing objects 0 and 2),
the four conditions give 97.5\% recipient agreement without an edit,
95.8\% donor agreement for the leading three directions, 95.7\% recipient
agreement for the complement, and 97.4\% donor agreement for full replacement.

We separately replace the two invariant components in the seven-dimensional
LSTMs at their middle sampled times. The three-dimensional component alone
gives donor agreement .162 and .167; the four-dimensional component alone
gives .205 and .206. Replacing both gives .994 and .997. Thus the
four-dimensional component is not an incidental addition to a sufficient
three-dimensional state. Additional controls on these same middle samples
compare raw-feature and standardized linear classifiers. Both decode cosets
accurately, but their weight subspaces have different intervention effects.
Raw-feature classifier directions transfer the coset, whereas unitwise
standardization can select low-variance directions with little causal effect.
We therefore do not use this comparison to claim a general separation
between linear probes and causal computation.

Two further checks use one LSTM and one GRU at their middle sampled times.
We construct prefixes ending in every one of the 60 group elements by
setting the last input to $q_{49}^{-1}g$, then continue with shared suffixes.
States with the same right coset have continuation agreement at least .995,
whereas agreement across cosets is at most .012. The identity class recovers
the same three-element subgroup without specifying its members in the test.
Replacing the selected layer with norm-matched Gaussian states or
coordinate-shuffled states gives only .04--.065 agreement with the donor's
coset, near the $1/20$ chance rate; self-replacement changes no logits.
These controls distinguish transfer of state content from an arbitrary
perturbation. Their coverage is two models, separate from the six-model
subspace measurements above.

\FloatBarrier

\section{Measurements for the main figures}
\label{app:figure-methods}

\subsection{The \texorpdfstring{$A_4$}{A4} output example}
Fig.~\ref{fig:quotient-output}(a--c) uses the raw recipe~A checkpoint trained with seed
42 for 74,219 updates. Its 12-element vocabulary gives 3,165,696 parameters.
The canonical evaluation uses 8,192 i.i.d.\ sequences of length 128, sampling seed
248041, FP32 evaluation and batches of 256; panels (a) and (b) retain positions
1--100. Natural-log diagnostic values are divided by $\ln 2$ to report bits.
In positions 17--100, mean KL is $2.3212\times10^{-4}$ bits and
$\Delta\mathrm{LL}=-2.5281\times10^{-4}$ bits.

A replay with the same checkpoint and sampling settings also exports
softmax probabilities. All 688,128 observations in positions 17--100 enter the
matrix, grouped by true state; no selection on quotient correctness is made.
Each row is the mean distribution over candidate elements conditional on a true
element. Both axes use the same ordering of the 12 elements in three actual
cosets of four. Within-coset entries range from 0.249395 to 0.250588; outside-coset
entries range from $4.6831\times10^{-5}$ to $9.2188\times10^{-5}$.
The mean probability outside the true coset is 0.0526916\%.
These are measured probabilities, without thresholding or replacement by an
ideal block matrix. The KL and log-likelihood diagnostics are
computed per example after restriction and renormalization to the true fiber,
not from the averaged matrix. Panel (b) plots the KL diagnostic; the
log-likelihood gain is reported in the text.

Replay checks find at most $1.03\times10^{-7}$ nats discrepancy in per-position KL,
$8.24\times10^{-8}$ nats in the likelihood difference, and two of 8,192 argmax
decisions at any position. Panels (a) and (b) preserve the canonical evaluation;
the output matrix uses the exported replay. GPU forward passes need not be
bit-identical across evaluations.

Panel (c) uses an additional census of eight complete input orbits at the same
checkpoint. Label is the target-specific term in
Eq.~\ref{eq:fiber-gradient-decomposition}; Net is the corresponding projected
cross-entropy term after adding Model. Their norms after averaging are divided
by the largest Label or Model norm on the measured grid at positions $\geq3$.
The exclusion of positions 1--2 keeps the repeated-input artifact above from
setting the scale. The ratio $R$ is computed within each orbit as
$\|\sum_k\ell_k\|/\sum_k\|\ell_k\|$ and then summarized by its median; it is
not divided by the common norm reference. The gray band is the min--max range
of this ratio across the six matched $\mathrm{SL}(2,3)$ and $S_4$ models.
The full nine-model profiles use 32 orbits and are reported separately in
App.~\ref{app:gradient-details}.
For the eight-orbit $A_4$ measurement in panel (c), the raw Net norm peaks at
44.6659 at position 5 and is .7272 at position 40, or 1.63\% of that peak.

\paragraph{The early accuracy dip.}
\label{app:early-dip}
Recipe~A uses rotary embeddings in queries and keys, with no additive position
embedding or prepended start token. On a constant prefix $(g,\ldots,g)$, all
value vectors are identical at each layer: normalized attention averages the
same vector regardless of its weights, and the pointwise updates preserve this
equality. The model therefore gives the same output at every position of that
prefix. Correctly predicting $g$ at position 1 forces an error on $(g,g)$ at
position 2 whenever $g\ne e$, whose target is $g^2$.

Enumerating all ordered pairs for all three $A_4$ seeds and all three $C_8$ seeds
finds exactly these errors, with every other pair correct. Their counts in the
canonical evaluation are 639/8,192 for $A_4$ and 912/8,192 for $C_8$, reproducing
the plotted accuracies of 92.20\% and 88.87\%. For a quotient $G/N$, the repeated
pair is wrong only when $g\notin N$; the $A_4/V_4$ count is 475/8,192, giving
94.20\% quotient accuracy. These are short-prefix errors of the specified
architecture, not evidence of the later within-fiber uncertainty. The reported
plateau measurements use positions 17--100.

Table~\ref{tab:other-groups} extends the main example to the other twelve groups.
Its selected completed endpoints differ from the budget-specific cohorts in
Table~\ref{tab:law}, which retain incomplete quotient channels.
Neither table should be read as a common-budget success rate for the
selected endpoints.
\begin{table}[htbp]
\centering
\small
\setlength{\tabcolsep}{4pt}
\begin{tabular}{lrrrrrrr}
\toprule
Group & Seed & Updates & Full (\%) & Quotient (\%) & $100/f$ & KL & $\Delta$LL \\
\midrule
$D_4$ & 43 & 148,438 & 50.063 & 99.948 & 50.000 & 0.313 & -0.329 \\
$Q_8$ & 42 & 148,438 & 49.910 & 99.927 & 50.000 & 0.230 & -0.251 \\
$S_3$ & 42 & 74,219 & 33.353 & 99.906 & 33.333 & 0.346 & -0.337 \\
$D_6$ & 44 & 148,438 & 33.364 & 99.968 & 33.333 & 0.236 & -0.179 \\
$\mathrm{Dic}_3$ & 44 & 74,219 & 33.235 & 99.933 & 33.333 & 0.593 & -0.667 \\
$D_5$ & 42 & 74,219 & 20.078 & 99.993 & 20.000 & 0.244 & -0.212 \\
$C_7\rtimes C_3$ & 42 & 74,219 & 14.240 & 99.926 & 14.286 & 0.311 & -0.311 \\
$\mathrm{SL}(2,3)$ & 42 & 74,219 & 12.494 & 99.938 & 12.500 & 0.383 & -0.428 \\
$S_4$ & 42 & 74,219 & 8.342 & 99.986 & 8.333 & 0.234 & -0.193 \\
$S_5$ & 42 & 74,219 & 1.671 & 99.966 & 1.667 & 0.438 & -0.397 \\
$C_8$ & 42 & 74,219 & 99.256 & 99.256 & 100.000 & -- & -- \\
$A_5$ & 42 & 74,219 & 1.658 & 100.000 & 1.667 & 0.058 & -0.050 \\
\bottomrule
\end{tabular}
\caption{The other 12 groups, supplementing the $A_4$ main figure. All values use positions 17--100 and all examples of the same 8,192-sequence FP32 evaluation protocol. KL and $\Delta$LL are in $10^{-3}$ bits. Endpoint selection for nontrivial proper quotients requires quotient accuracy $\geq99.9\%$, preferring the smaller eligible completed budget and then the smaller eligible seed; full-state accuracy is not used for selection. $C_8$ ($Q=G$) and $A_5$ (trivial $Q$) are controls, fixed to seed 42 at 74,219 updates. $C_8$ singleton-fiber diagnostics are omitted. This table describes selected endpoints, not success rates over seeds.}
\label{tab:other-groups}
\end{table}

\subsection{Finite and infinite Heisenberg curves}
Fig.~\ref{fig:quotient-output}(d,e) uses the original final evaluations of
$H_3(\mathbb Z/3)$ with full-group inputs and
$G_3=H_3(\mathbb Z)/\langle z^3\rangle$ with the finite alphabet
$\{-1,0,1\}^2\times\mathbb Z/3$. The former predicts an element index;
the latter predicts three coordinates. Both have abelianization classes
of size three. Each curve averages training seeds 42--44; bands show their
observed minimum and maximum at every position, without smoothing.
Evaluations use 8,192 length-128 sequences at seed 248041 with bf16 autocast;
the figure retains positions 1--100. These are the same evaluations used
for the corresponding rows of Table~\ref{tab:law}, rather than the separate
FP32 replay used for the $A_4$ example. Full configurations and per-run
results are in App.~\ref{app:heisenberg-law}.

\subsection{Matched training, probes, and gradients}
\paragraph{Matched models and behavior.}
The three groups use the same full-alphabet uniform i.i.d.\ input distribution,
four-layer NeoX backbone (width 256, four attention heads), training length 100,
batch size 256, and AdamW recipe (initial learning rate $5\times10^{-5}$,
zero weight decay, no warmup, linear decay over 74,219 updates).
All diagnostics use raw final weights from training seeds 42, 43, and 44;
checkpoint hashes are shared across the probe and gradient measurements.
Vocabulary-dependent parameter counts differ: 3,165,696 for $A_4$ and
3,171,840 for the other groups. The source records retain all 76 original
training-log evaluations and their original precision.
Quotient accuracy maps the full-state argmax prediction to its quotient class.

\paragraph{Linear probes.}
We apply the same fixed linear-probe protocol to all nine models. From 2,048 fresh i.i.d.\ 
sequences (sampling seed 500000), we use hidden states at positions 17--100,
splitting entire sequences into 1,638 training and 410 test sequences. A linear
softmax classifier is trained for 300 full-batch Adam updates at learning rate
0.01, with fixed initialization seeds 164 for fiber rank and 165 for quotient
class. Fiber rank is the sorted index of the exact state within its true
quotient class. The numbers of fiber and quotient classes are $(4,3)$,
$(8,3)$, and $(12,2)$. A within-sequence shuffled-label control uses shuffle seed
9999 and probe seed 166. Layer-0 and shuffled-control results are retained
with the figure data. This experiment constrains the specified linear
decoding task; it does not exclude other encodings or nonlinear decoders.

\paragraph{Gradient measurements.}
The same nine frozen models enter the gradient analysis in
App.~\ref{app:gradient-details}. That appendix gives the decomposition,
sample construction, normalization, numerical checks, and interpretation
together with Fig.~\ref{fig:matched-gradients}.

The gradient and direction diagnostics are defined together in
App.~\ref{app:gradient-details}.

\section{Protocols and endpoint coverage for Figure~\ref{fig:long-training-stages}}
\label{app:fig4-protocol}

Fig.~\ref{fig:training-context} preserves the original broader comparison.
Fig.~\ref{fig:s4-training} retains the $S_4$ Transformer--LSTM comparison,
including the dense LSTM replay through 3,000 updates. Its Transformer data
also underlie panels (a,b) of the main Fig.~\ref{fig:long-training-stages};
the $S_5$ panels use the distinct protocols in App.~\ref{app:nonsolvable-census}.
Fixed-checkpoint gradients remain in Fig.~\ref{fig:matched-gradients}, with
definitions and protocol in App.~\ref{app:gradient-details}.

\begin{figure}[htbp]
\centering
\includegraphics[width=\linewidth]{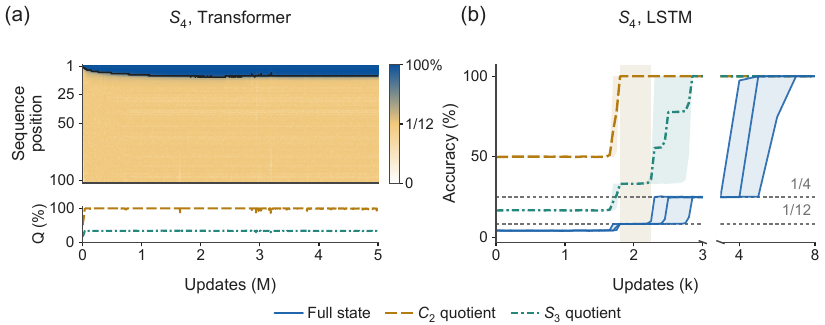}
\caption{\textbf{Quotient stages in the $S_4$ comparison.}
(a) The long Transformer run retained in the main figure.
(b) LSTM exact accuracy for individual seeds and quotient accuracy means
with seed ranges. The horizontal scale changes at 3k updates, where dense
replay gives way to the original logs. Beige shading marks the shared
$1/12$ stage; dashed references give $1/12$ and $1/4$.}
\label{fig:s4-training}
\end{figure}

\paragraph{The $S_4$ comparison.}
We use the complete five-million-update Transformer log for seed 43.
The four-layer, width-256 NeoX model uses four heads, full-group i.i.d.\
uniform inputs, training length 100, effective batch size 256, BF16 autocast,
and AdamW with constant learning rate $5\times10^{-5}$, no warmup or decay,
zero weight decay, and gradient clipping at 1. Evaluations use 1,024
length-128 sequences at update 1 and every 5,000 updates thereafter.
We retain all 1,001 evaluations without smoothing. The supplementary heatmap shows full-state
accuracy at positions 1--100. Its black curve gives the longest contiguous
prefix with exact accuracy at least .75, marking the last qualifying position.
The color scale is linear separately on $[0,1/12]$ and $[1/12,1]$, assigning
half the color range to each interval. The lower curves average $C_2$ and $S_3$ quotient accuracies over positions
17--100. Quotient correctness is measured by projecting the full-state argmax.
The exact frontier never exceeds
10 and ends at 9, so this window stays outside it. The final logged quotient and conditional
exact accuracies are 99.67\% and 8.28\%; the maximum $S_3$ accuracy across
the log is 34.05\%. Table~\ref{tab:long-finals} uses a separate terminal
evaluation with 8,192 sequences, so its last decimal places differ.

The other $S_4$ logs, seeds 44 and 45, end at 2,410,000 and 965,000 updates
after interrupted runs. Their last quotient accuracies are 99.97\% and
100.00\%, and conditional exact accuracies are 9.41\% and 8.81\%, respectively.
We do not extrapolate these shorter records or pool them into a five-million-update
seed range. The original $D_4$ trajectory remains in
Fig.~\ref{fig:training-context}a. The $S_4$ Transformer and LSTM panels illustrate
different learning outcomes; their schedules and numerical precision differ,
so the comparison does not isolate architecture as the cause.

\paragraph{Resolving the $S_4$ LSTM stages.}
For seeds 42--44, Fig.~\ref{fig:s4-training}b uses the existing dense
replays through update 3,000, with evaluations every 50 updates, followed by
the original logs at 1,000-update intervals through update 8,000. The replay
uses the same training script, initialization/data seeds, recipe, FP32 precision,
and 74,219-update learning-rate schedule; it stops at update 3,000. Both logs
evaluate 1,024 sequences on the same window. Across every shared evaluation
and all three channels, the largest absolute difference between replay and
original accuracies is .000884. The display separates the two records at
3,000 updates and changes the horizontal scale there; it does not omit a time
interval. Full-state curves retain individual seeds on both sides. Quotient
curves are seed means, and bands are seed ranges. We interpolate only between
recorded evaluations and apply no smoothing.

The $C_2$ stage requires quotient accuracy at least .99 and exact accuracy
within .05 of $1/12$, at an evaluation preceding the run's first exact accuracy
of .90. Without this ordering condition the criterion would also score a late
regression, in which a run that has already solved the task falls back toward
the class-size baseline while its quotient readout stays accurate. No run in
this replay window meets the criterion only after solving, so the count of 15
of 18 is the same with and without the condition.
The stage spans updates 1,700--2,800, 1,800--2,250, and
1,800--2,450 for seeds 42, 43, and 44, respectively, containing 23, 10,
and 14 dense evaluations. The beige interval is their intersection,
1,800--2,250, rather than their union. The original 1,000-update grid samples
each stage only once, at update 2,000.

The separate 2,048-sequence coset census used for the 15-of-18 count
(App.~\ref{app:blind-census}) gives slightly different boundary evaluations
from the 1,024-sequence accuracy replay above. In that census, the fifteen
stages differ widely in duration. On the 50-update grid they span 50 to 1,050 updates, with a median of 200, and
contain 2 to 22 evaluations, with a median of 5. Four rest on two
evaluations: $A_4$ at $5\times10^{-5}$ seed 44, $D_4$ at $10^{-3}$ seeds 43
and 44, and $S_4$ at $10^{-3}$ seed 43. The 1,000-update grid of the main
training logs places at most one evaluation inside any of the fifteen, and
none inside nine of them, so that grid cannot resolve a stage even where one
is present. We therefore report stage widths rather than requiring a minimum
duration. The denser display makes these
transient stages visible without averaging their different departure times
into a single full-state curve. These readouts do not establish within-class
softmax uniformity during the transient stages.

Panel letters below refer
to Fig.~\ref{fig:training-context}.

\begin{figure}[htbp]
\centering
\includegraphics[width=\linewidth]{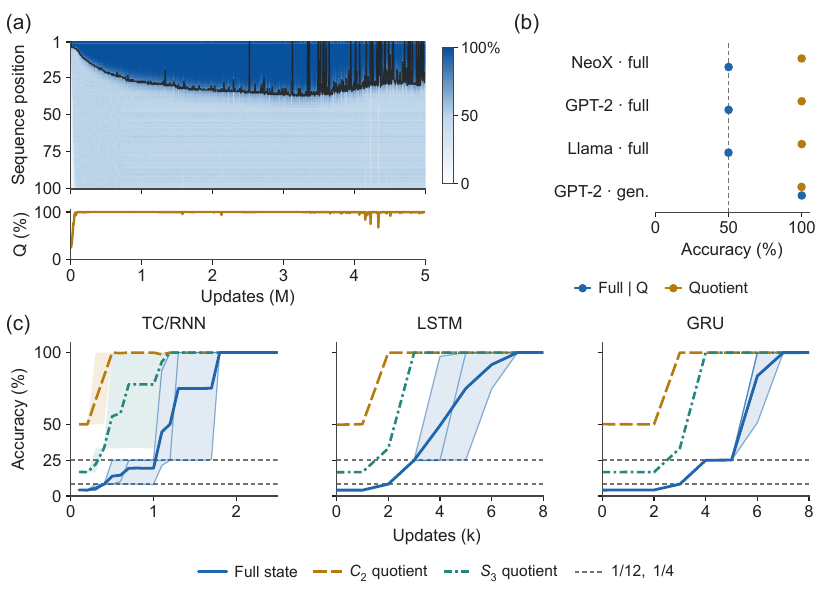}
\caption{\textbf{Additional training and architecture comparisons.}
(a) The original $D_4$ long-training display. (b) $D_4$ endpoints at positions
70--100: quotient accuracy and exact accuracy conditional on a correct
quotient, under different budgets and recipes. (c) Original $S_4$ trajectories
for a bilinear RNN (TC/RNN), LSTM, and GRU. Lines and bands show means and
ranges across three seeds; thin blue lines retain individual full-state runs.
TC/RNN selects the most probable class after summing its members' probabilities;
LSTM and GRU project the full-state argmax.}
\label{fig:training-context}
\end{figure}

Each panel reports a distinct experiment. In panel (a), the four-layer,
width-256 NeoX model uses full-alphabet i.i.d.\ 
$D_4$ inputs, batch size 256, training length 100, and a constant learning rate
$5\times10^{-5}$ without warmup or decay. Seed 45 is evaluated at step 1 and every
5,000 updates thereafter, with 1,024 evaluation sequences per log entry. The
heatmap shows positions 1--100. Its contiguous-prefix curve is a descriptive
readout, not a fitted growth law; a low-accuracy hole at an earlier position can
reduce that prefix even when later positions remain accurate. The lower trace
averages quotient accuracy over positions 17--100.

Panel (b) uses 8,192 sequences per terminal evaluation and fixes the window to
positions 70--100 in every row. The baseline recipe (A in
Table~\ref{tab:fig4-endpoints}) has four layers, width 256, batch size 256, AdamW
at $5\times10^{-5}$, and linear learning-rate decay without warmup; raw weights
are evaluated. The Liu replication recipe (B) uses a three-layer, width-512 GPT-2,
batch size 16, sinusoidal positional embeddings, AdamW initially at $10^{-4}$,
and a plateau scheduler; EMA weights are evaluated. The full-alphabet GPT-2 in
recipe A instead uses learned positional embeddings. See App.~\ref{app:config}
for the full model specifications. Thus the two GPT-2 rows are not an isolated
input-alphabet ablation.

For each run, the evaluator projects the full argmax onto the quotient, so exact
correctness implies quotient correctness. We sum the exact and quotient correct
counts over the window, take their ratio, and then average these per-run ratios
over seeds. The bars are observed seed ranges, not confidence intervals or
temporal variability. The main panel retains both existing NeoX seeds at the
148,438-update budget (43 and 44), and all three seeds (42--44) in each other
displayed condition, without an accuracy threshold. Conditions were selected
after inspecting results; the panel is not an exhaustive success-rate estimate.
The generator endpoints reach full-state accuracy. Terminal evaluations exist for all three generator seeds,
but complete local training logs are available only for seed 42; the endpoint
summary does not establish temporal stability for all three.

\begin{table}[!htbp]
\centering
\small
\begin{tabular}{@{}llrrrr@{}}
\toprule
Model & Input / recipe & Updates & Seeds & $Q$ (\%) & Full $\mid Q$ (\%) \\
\midrule
NeoX & full / A & 74,219 & 3 & 81.727 & 50.043 \\
NeoX & full / A & 148,438 & 2 & 99.932 & 49.987 \\
GPT-2 & full / A & 74,219 & 3 & 99.945 & 49.990 \\
Llama & full / A & 74,219 & 3 & 99.990 & 49.902 \\
NeoX & gen. / A & 74,219 & 3 & 72.575 & 50.015 \\
GPT-2 & full / B & 312,500 & 3 & 94.214 & 49.959 \\
GPT-2 & gen. / B & 312,500 & 3 & 99.857 & 99.999 \\
\bottomrule
\end{tabular}

\caption{All 20 available candidate endpoints considered for Fig.~\ref{fig:training-context}b,
grouped by condition and budget. Entries are seed means on positions 70--100.
Shorter budgets and partial quotient channels are retained. The table uses the
same evaluator and per-run ratio as the main panel; budgets are not pooled.}
\label{tab:fig4-endpoints}
\end{table}

Panel (c) uses the original $S_4$ logs for seeds 42--44 in each recurrent family.
The TC bilinear RNN uses fixed 64-dimensional input encodings, hidden dimension
128, batch size 256, Adam at $10^{-3}$, and 20,000 updates. LSTM (two layers,
width 444) and GRU (two layers, width 512) use the matched baseline recipe at
$5\times10^{-5}$, with learned embeddings and 74,219 updates, evaluated in FP32
every 1,000 updates. All plotted metrics average positions 17--100. The display
ends at 2,500 updates for TC and 8,000 for LSTM/GRU; full-state accuracy remains
at least $99\%$ at all subsequent logged evaluations in these nine runs.
TC quotient readouts sum candidate probabilities within each coset before taking
the argmax, whereas LSTM/GRU project the full argmax. Their quotient accuracies
are therefore different readouts, and the conditional ratio in panel (b) is not
applied to TC. The runs in Fig.~\ref{fig:training-context}c use only their original logs.

For $S_4$, the two displayed quotient images are
$S_4/[S_4,S_4]\cong C_2$ and $S_4/V_4\cong S_3$, with kernel sizes 12 and 4.
The reference floors are consequently $1/12$ and $1/4$. Acquisition times vary:
TC seed 44 has no resolved $C_2$-only stage on its recorded grid. A mean curve
can blur or interpolate the individual plateaus, which is why individual
full-state curves and seed ranges are retained. These accuracy trajectories
alone do not establish per-sample softmax uniformity at the transient stages.

\section{Additional curves under their original experimental setups}
\label{app:legacy-figures}

These supporting plots retain their original data and setup-specific captions.
They are separate from the matched final checkpoints in Figs.~\ref{fig:quotient-output}
and~\ref{fig:matched-gradients}. The run-level law plot retains the original census,
including incomplete quotient channels; it is distinct from the selected endpoints
in Table~\ref{tab:other-groups}.
Figure~\ref{fig:frontier} shows how the quotient and exact channels separate
over training and compares the original run-level plateaus with the class-size
prediction. Figure~\ref{fig:decay} connects an intermediate quotient to the
loss of within-class predictive gain beyond the frontier.
Figure~\ref{fig:levers} shows how depth changes tracking on generator inputs
and how accuracy deteriorates beyond the training horizon.

\begin{figure}[!htb]
\centering
\includegraphics[width=\linewidth,trim=0 10bp 0 0,clip]{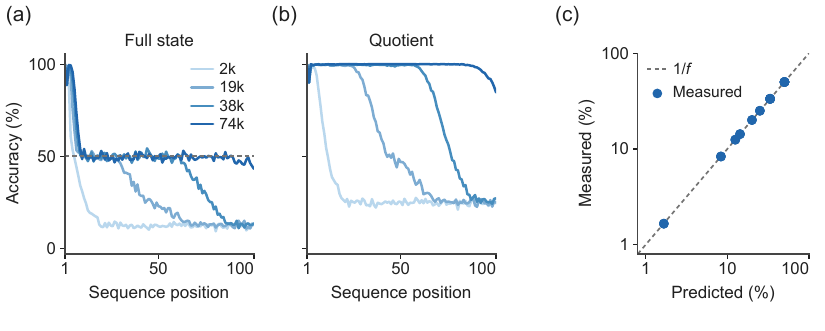}
\caption{(a, b) Accuracy against position for $D_4$ under recipe~A at four
checkpoints (2k, 19k, 38k and 74k steps), for the exact channel (a) and the
quotient channel (b). The dashed line in the exact panel is $1/|[G,G]|$. The
quotient channel becomes accurate throughout the horizon as training proceeds, while the exact
channel advances only a few positions. (c) The measured plateau against
$1/|[G,G]|$ for
every run in the original census with a recorded quotient readout; the updated
$S_3$ evaluations in Table~\ref{tab:law} are not included. Both axes in (c) are logarithmic
and report percentages;
the diagonal is the reference prediction, with no fitted parameters.}
\label{fig:frontier}
\end{figure}

\begin{figure}[!htb]
\centering
\includegraphics[width=\linewidth,trim=0 10bp 0 0,clip]{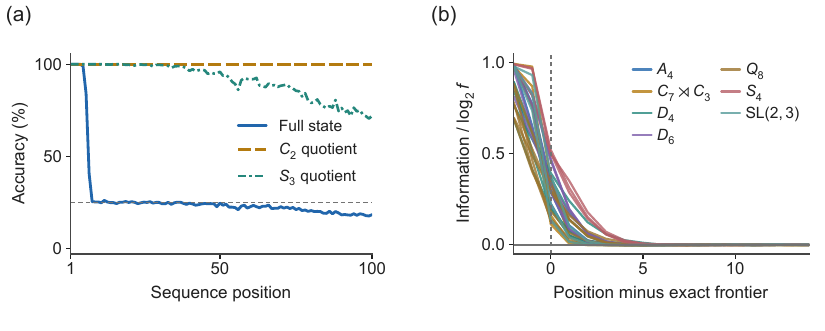}
\caption{(a) A run that stops at an intermediate quotient. On $S_4$ under narrowed
support, the $S_3$-image accuracy is near 100\% through position 20 and
falls to 72\% at position 100, while the exact frontier is 6. The residual inside the corresponding normal subgroup sits at
$1/|V_4|$. (b) Decrease in within-fiber information near the exact frontier. Each curve is
one of 28 checkpoints and plots $\Delta\mathrm{LL}/\log_2 f$,
with $\Delta\mathrm{LL}$ defined in App.~\ref{app:readouts}.
Curves are aligned so that zero denotes the first position after the
contiguous exact prefix.
Values near zero indicate no mean log-score gain over uniform prediction.}
\label{fig:decay}
\end{figure}

\begin{figure}[!htb]
\centering
\includegraphics[width=\linewidth]{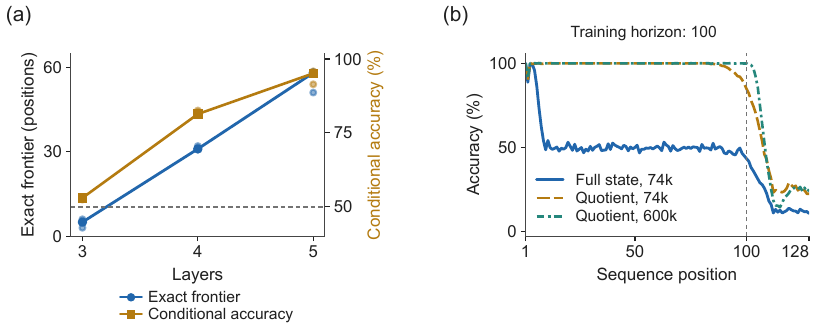}
\caption{(a) On the generator stream (recipe~B, NeoX backbone with rotary
embeddings) additional depth extends the exact frontier and improves conditional accuracy; the dashed
line is $1/|[G,G]|$ for $D_4$. (b) Past the training horizon of 100 positions both exact and quotient accuracy decline.}
\label{fig:levers}
\end{figure}

\section{Computing the order-blind ceiling}
\label{app:ceiling}

Proposition~\ref{prop:ceiling} bounds the exact accuracy of every order-blind model by
$C_t = \mathbb{E}_M[\max_g \Pr(q_t = g \mid M)]$. Here the numerical calculations
use finite groups and independent uniform full-group inputs. We compute or
estimate $C_t$ while enumerating each multiset's orderings exactly. Write
$D_M$ for the law of the product of a uniformly random ordering of the multiset $M$.
Removing the last factor gives the recursion
$D_M(g) = \sum_{x} \frac{m_x}{|M|}\, D_{M - x}(g\,x^{-1})$. We run this recursion over
all multisets of size $t$ in increasing $t$, weighting each $M$ by its multinomial
probability. For $S_3$, $D_4$ and $Q_8$ this is exact over multisets and over orderings through
$t=17$. For $D_5$ it is exact through $t=14$; for $D_6$, $\mathrm{Dic}_3$ and $A_4$
through $t=12$; for $C_7\!\rtimes\!C_3$, $\mathrm{SL}(2,3)$ and $S_4$ through $t=8$.
Beyond those lengths the number of multisets is too large to enumerate, and at $t=17$
we run the same exact recursion over the sub-multisets of each of 100 sampled multisets and report the mean
with its standard error. This mean is an unbiased estimate of $C_t$ (an estimator
that samples orderings would be biased upward by the maximum). The enumerator carries
a brute-force self-test against all $|G|^t$ sequences at small $t$.
Table~\ref{tab:ceiling} lists the excess $C_t - 1/|[G,G]|$ and the probability mass of
multisets whose orderings do not cover the whole coset. This mass is positive at
every finite $t$ for the non-abelian groups considered, as shown by the
all-identity multiset. The enumeration quantifies its finite-length contribution.

\begin{table}[ht]
\centering
\small
\begin{tabular}{lrrrrrrr}
\toprule
$G$ & $f$ & $t=3$ & $t=4$ & $t=8$ & $t=12$ & $t=17$ & non-sweep mass \\
\midrule
$S_3$ & 3 & 0.3009 & 0.2130 & 0.0516 & 0.0120 & 0.0019 & 8.4e-06 ($t=17$) \\
$D_4$ & 2 & 0.2188 & 0.1484 & 0.0323 & 0.0070 & 0.0010 & 2.3e-05 ($t=17$) \\
$Q_8$ & 2 & 0.2188 & 0.1484 & 0.0323 & 0.0070 & 0.0010 & 2.3e-05 ($t=17$) \\
$D_6$ & 3 & 0.3009 & 0.2130 & 0.0516 & 0.0120 & 0.0027$\pm$6.8e-04 & 3.9e-04 ($t=12$) \\
$\mathrm{Dic}_3$ & 3 & 0.3009 & 0.2130 & 0.0516 & 0.0120 & 0.0031$\pm$9.4e-04 & 3.9e-04 ($t=12$) \\
$A_4$ & 4 & 0.2500 & 0.1568 & 0.0200 & 0.0023 & 1.6e-04$\pm$5.2e-05 & 1.1e-05 ($t=12$) \\
$D_5$ & 5 & 0.3500 & 0.2480 & 0.0614 & 0.0142 & 0.0021$\pm$3.6e-04 & 6.7e-05 ($t=14$) \\
$C_7\!\rtimes\!C_3$ & 7 & 0.2902 & 0.1799 & 0.0244 & -- & 1.5e-04$\pm$1.8e-05 & 1.2e-03 ($t=8$) \\
$\mathrm{SL}(2,3)$ & 8 & 0.3507 & 0.2457 & 0.0550 & -- & 0.0020$\pm$3.5e-04 & 1.9e-02 ($t=8$) \\
$S_4$ & 12 & 0.3244 & 0.2066 & 0.0323 & -- & 7.1e-04$\pm$2.7e-04 & 1.7e-02 ($t=8$) \\
\bottomrule
\end{tabular}

\caption{Excess of the exact order-blind ceiling over $1/|[G,G]|$, by prefix length.
Entries at $t=17$ with a standard error are Monte Carlo over 100 multisets with exact
orderings; all other entries are exact over all multisets. The last column is the
probability mass of multisets whose orderings do not sweep their coset, at the largest
exactly enumerated length.}
\label{tab:ceiling}
\end{table}

\clearpage
\section{Recovering central coordinates}
\label{sec:recovery}

We now investigate how models recover state distinctions beyond the
abelianization. Prior work shows that models can reach exact state tracking when updates are drawn from a small generating set rather than uniformly from the full group \citep{liu2023shortcut,li2025state}. We therefore use generator-restricted inputs to study how models refine quotient-level solutions. The prediction task remains unchanged: at every position, the model predicts the running product. These input sets are different from the
full-group baseline in Sec.~\ref{sec:law}.

\paragraph{A finite-group example.}
\label{sec:crossing}
On $D_4$, let $r$ denote a quarter-turn and $s$ a reflection. With inputs
restricted to $\{r,s\}$, three parities determine the full state: the total
number of $r$ tokens, the total number of $s$ tokens, and the number of $r$
tokens at even input positions. For example, $rs$ and $sr$ contain the same
tokens but place $r$ at different position parities and produce different
states. We recover these counts with probes and test their contributions
through activation replacement (App.~\ref{app:support-details},
Eq.~\ref{eq:generator-three-parities}). We next study a task requiring
finer order information than these parity counts.

\subsection{Tracking beyond odd/even counts}
\label{sec:heisenberg-extension}

We use the integer Heisenberg group $H_3(\mathbb Z)$, whose multiplication is
\[
(x,y,z)(u,v,w)=(x+u,y+v,z+w+xv).
\]
Each input $g_i=(u_i,v_i,0)$ is one of $a=(1,0,0)$, $a^{-1}=(-1,0,0)$,
$b=(0,1,0)$, or $b^{-1}=(0,-1,0)$. Starting from $(0,0,0)$,
write the running product as $(X_t,Y_t,Z_t)$. The first two coordinates are
the net counts of $a$ and $b$ and give the abelianization, while $Z_t$ retains input
order. For example, $a^4b^4$ and $b^4a^4$ have identical token counts at both odd
and even positions, but give $(4,4,16)$ and $(4,4,0)$ respectively.

To test whether Transformers distinguish such products, we train three four-layer models at learning rate $10^{-3}$.
We predict $x$, $y$, and $z$ with separate classification heads and minimize 
the sum of their cross-entropies at every position. We find models reach
98.07--98.17\% joint state accuracy (all three coordinates correct) over positions 17--100, and
92.81--93.74\% at position 100. For separate, coordinate-controlled pairs
that preserve odd/even token counts while changing $z$, both predictions
are correct in 96.9\% of pairs, averaged over seeds. We thus conclude that
these models distinguish products that knowing odd/even counts alone cannot resolve
(App.~\ref{app:heisenberg}).

\paragraph{Ordered prefix counts.}
We use the group update to identify information needed for $z$ prediction.
Multiplying by $(u_i,v_i,0)$ adds $u_i$ to $X$, $v_i$ to $Y$,
and $X_{i-1}v_i$ to $Z$. Summing these updates gives
\begin{equation}
X_t=\sum_{i\le t}u_i,\qquad Y_t=\sum_{i\le t}v_i,\qquad
Z_t=\sum_{i\le t}X_{i-1}v_i.
\label{eq:heisenberg-prefix}
\end{equation}
For example, $aa$ gives $(2,0,0)$, multiplying by $b$ gives $aab=(2,1,2)$, whereas
appending $b^{-1}$ gives $(2,-1,-2)$. Thus $a$ changes the count that a later
$b$ or $b^{-1}$ uses to update $Z_t$. 

We first test whether the running count $X_t$ is represented in the model.
A single linear readout per model recovers $X_t$ across sampled positions
16--100 from normalized second-block outputs, with $R^2=.994$--$.999$.
Subtracting the known current increment $u_t$ gives the preceding count
$X_{t-1}$ used in the update of $Z_t$
(App.~\ref{app:heisenberg-shared-counts}).

Next, we ask whether count-associated activations influence predictions of 
the central coordinate. At position \(j\), we regress the concatenated outputs 
of the first-layer attention heads on \(X_{j-1}\), \(Y_j\), \(Z_j\), 
the total count of $a$ and $a^{-1}$, and the current token. The fitted coefficient 
\(\beta_x\) is the activation-space direction associated with a one-unit 
increase in the preceding net \(a\)-count after controlling for the other 
predictors. On held-out inputs, we replace \(h_j\) by \(h_j+\delta\beta_x\), 
leave the tokens and model weights unchanged, and run the rest of the network. 
We then compare the predicted mean \(z\) with and without the perturbation
(App.~\ref{app:heisenberg-encoding-fit}).

Figure~\ref{fig:heisenberg-count-interventions}a shows the change in predicted
mean $z$, divided by $\delta$, for each current token. We find positive effects
at $b$ and negative effects at $b^{-1}$, consistent with $v_i$ in
Eq.~\ref{eq:heisenberg-prefix}. Mean responses at $a^{\pm1}$ remain closer
to zero (App.~\ref{app:heisenberg-a-response}).

Since we have shown that perturbing count-associated activations at one 
position changes the predicted \(z\), we next ask whether the same perturbation 
at two positions with $b^{\pm1}$ has additive effects, corresponding to the 
last identity in Eq.~\ref{eq:heisenberg-prefix}. For example, if the two 
individual interventions change the prediction by \(+2.0\) and \(+1.8\), 
additivity will imply a joint prediction change of \(+3.8\). 
Figure~\ref{fig:heisenberg-count-interventions}b compares this predicted 
sum with the measured joint change. The 18 paired effects broadly follow
the additivity line, with a maximum deviation of .30 units of $z$.
The individual responses vary across positions and models, so mixed
$b,b^{-1}$ pairs need not cancel. This comparison tests whether measured
effects add; the fitted perturbations need not have unit gain.
\begin{figure}[!htbp]
\centering
\includegraphics[width=\linewidth]{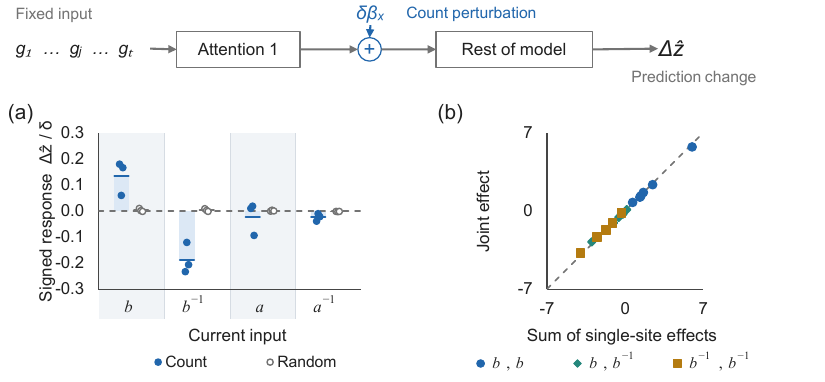}
\input{figures/fig4_count_interventions_caption}
\label{fig:heisenberg-count-interventions}
\end{figure}

\FloatBarrier

\section{Configurations}
\label{app:config}

\paragraph{Recipe A.} GPT-NeoX causal transformer \citep{black2022neox},
4 layers, width 256, 4 heads,
rotary position embedding \citep{su2024rope} on a quarter of each head,
3{,}163{,}648 parameters; token
alphabet is the whole group; training length 100; AdamW
\citep{loshchilov2019adamw}, learning rate
$5\times10^{-5}$, weight decay 0, linear decay to zero with no warmup, batch 256,
74{,}219 steps, bf16, no dropout, no weight averaging. Evaluation every 1{,}000 steps
on 1{,}024 sequences of length 128; final evaluation on 8{,}192 sequences at seed
248041 (16{,}384 for the $S_5$ anchor). Seeds 42, 43, 44 unless stated. The
constant-learning-rate arms of \S\ref{sec:levers} keep the recipe and hold the rate at
$5\times10^{-5}$ for 600k steps, with one historical $D_4$ trajectory
and four additional runs continued to 5M updates (App.~\ref{app:long-finals}). The doubled-budget cells of
Table~\ref{tab:law} run the same recipe for 148{,}438 steps. Settlement window:
positions 17--100.

\paragraph{Recipe B.} Replication of \citet{liu2023shortcut}, App.~B: GPT-2
block with fixed sinusoidal position embeddings and a separate linear state head, 3
layers (4 in one arm), width 512, 8 heads, MLP width 512, dropout .1; training length
100, evaluation length 128; AdamW at $10^{-4}$ (their grid is
$\{3\times10^{-5}, 10^{-4}, 3\times10^{-4}\}$), weight decay $10^{-4}$, batch 16,
5M sequences (312{,}500 steps), fp32, gradient clip 1, weights evaluated through an
exponential moving average with decay .9; we implement the incompletely specified patience scheduler with
ReduceLROnPlateau (factor .5, patience 6 evaluations of 5{,}000 steps) on
the EMA all-position accuracy. Token alphabet is the generator set $\{r,s\}$ in the
generator-input runs, the full group in the full-alphabet comparison, and a union of complete cosets in
the support-axis and ladder arms ($[G,G]$-cosets throughout, except the $C_3$-cosets
of $C_9$). Settlement window: positions
9--100, with frontier thresholds as in recipe~A.

\paragraph{Comparison C.} The architecture and training settings of recipe~B on the full alphabet of
$D_4$, in four cells: GPT-2 or GPT-NeoX block, with fixed sinusoidal position
embeddings or full-head rotary embeddings and no additive position embedding;
312{,}500 steps, 8{,}192-sequence final evaluation, three seeds per cell.

\paragraph{Recurrent control.} A weight-shared bilinear recurrence
$z_t = u_t / (1 + \mathrm{RMS}(u_t))$ with $u_t = (z_{t-1} W_z) \odot (\varphi(x_t) W_x)$,
state dimension 128, $\varphi$ a frozen random orthonormal encoding of dimension 64
(32 for $D_9$ and $S_3\times C_3$, whose runs predate the wider groups), linear
readout to $G$ and to each intermediate quotient (coset probabilities summed before
the argmax); Adam at $10^{-3}$, gradient clip 1, batch 256, training length 100,
full-alphabet i.i.d.\ inputs, 20k steps, evaluation every 100 steps to step 2{,}000
and every 500 thereafter on 1{,}024 sequences, final evaluation on 8{,}192 sequences;
groups $D_8$, $Q_{16}$, $S_4$, $D_{15}$, $D_{27}$, $C_9$, $D_9$, $S_3\times C_3$,
$D_{21}$, $D_{25}$; seeds 42, 43, 44 in the confirmatory round. The frequency battery
adds ten seeds per group at a 5k-step budget (a truncation whose class agrees with the
full history in 30 of 30 earlier runs) and a $2\times2$ crossing of initialization
seed against data seed on $D_9$ and $D_{27}$ at 20k steps. A stage is recorded when
some quotient's accuracy over positions 17--100 reaches .95 while exact accuracy is at
most .5. Its accuracy is checked against $1/|N|$ for the kernel $N$ of the quotients
reached, within $\pm .04$.

\paragraph{160M depth ladder.} GPT-NeoX/Pythia-160M architecture (hidden 768, 12
heads, MLP 3072, rotary on a quarter of each head, parallel residual, untied
embeddings, vocabulary 50283 after the released resize) at 3, 4, 7 and 12 blocks
(98.5M, 105.6M, 126.9M and 162.3M parameters); for each seed one 12-block master is
initialized from scratch and every shallower model is its exact block-prefix, with
shared components checked by hash across depths; $S_3$, i.i.d.\ uniform over all six
elements, length 100, cross-entropy at every prefix; fp32, AdamW at $5\times10^{-5}$,
weight decay 0, no warmup, linear-decay horizon 148{,}440, gradient clip 1,
micro-batch 32 with accumulation 4, 10{,}000 steps; evaluation on 1{,}024 fixed
sequences (seed 248041) every 500 steps with exact and parity frontiers defined as the
first position below .98; formal seeds 2901--2908, bridge seeds 42--44 reported
separately.

\paragraph{$S_5$ depth census.} The recipe~A architecture at 4, 8 and 12 layers (width
256, 4 heads), $S_5$ full alphabet, training length 100 and evaluation length 200,
fp32 master weights with bf16 autocast, effective batch 256, 74{,}219 steps with
linear decay; 36 paired (initialization, data-order) seed pairs reused across depths;
evaluation on 4{,}096 fixed sequences every 1{,}000 steps. A run is parity-first when
its parity frontier (first position below .98) is at least 20 and at least 20
positions ahead of the exact frontier. A run is associative-like when the exact
frontier is at least 20 and the two frontiers are within 5 of each other.

\paragraph{Pythia-160M arms.} We train from scratch under the recipe of
\citet{li2025state} for $S_3$ on the Pythia-160M architecture, using AdamW at
$5\times10^{-5}$, their schedule and batch size, and fp32. We run three seeds
from scratch and one from the pretrained checkpoint, and classify the outcomes
at 10k steps using their cutoff rule. Sequences are
sampled online, i.i.d.\ uniform, which matches the distribution of their released
data (1M sequences, 900k for training) but not its repetition over 20 epochs. The
within-fiber entropy, coset mass and orbit statistics of \S\ref{sec:fiber} and
App.~\ref{app:gradient-details} are read on the one from-scratch seed that entered the
parity-first regime.

\paragraph{Looped transformer (one-seed pilot).} Two shared
pre-norm GELU transformer blocks applied for $L$ loops, with no position encoding of
any kind, width 128, 4 heads, MLP width 512, no dropout, 398{,}848 parameters; $D_4$ full alphabet, curriculum over lengths 2,
4, 8 and 16 with loop count equal to length and stage caps of 500, 1{,}000, 1{,}500 and
3{,}000 steps; AdamW at $3\times10^{-4}$, weight decay .01, gradient clip 1, batch
128, bf16; evaluation at lengths 16 and 32 for up to 40 loops on 2{,}048 sequences;
seed 42. The run was classified as \textsc{other} because no stage reached its .98
promotion threshold; it is excluded from the confirmatory evidence. The number of accurate positions increased by about one per loop, with position 2 a persistent
error. No region was both correct on the quotient and flat inside the fiber.

\paragraph{Expert-split MLP (one-seed pilots).} A
function-preserving split of each recipe~A MLP into two routed halves. Continued from
the 148{,}438-step $D_4$ seed-43 weights for 20k steps against a dense continuation,
both arms have frontier 6 and conditional accuracy .498. Trained from scratch at horizon
40 for 40k steps against a dense twin, the respective frontiers are 11 and 6, while
conditional accuracy over positions 17--40 is .502--.505. These pilots are excluded from the confirmatory comparisons. A learned
mixture-of-experts readout and a split
readout, both trained on the frozen weights, achieve .500--.506 at positions 40, 70 and
100.

\paragraph{Pythia-70M prelude (exploratory).} Two arms were run at 10k steps on $S_3$
with the recipe above. From scratch, exact and parity accuracy first fall below .98
at positions 9 and 91 (a quotient-first run). From the pretrained checkpoint, both
fall at position 8, with parity accuracy .529 over positions 8--100. This run
retains only a short accurate prefix.

\section{Coverage of architectures and training supports}
\label{app:coverage}

The tables below catalog the experimental packages underlying this draft.
``Completed'' denotes a package evaluated under its fixed criteria, not
asymptotic convergence. Pilot and exploratory results document coverage and
are excluded from the confirmatory comparisons. The matched two-layer LSTM
and GRU studies include 18 runs each on $D_4$, $A_4$, and $S_4$.
The Mamba-2 study adds 18 two-layer runs on these groups and three
four-layer $D_4$ runs (App.~\ref{app:ssm}). No Markov-input or power-law-input
results are included here.
The repeated-epoch run with the released $S_5$ pipeline remained at chance
under bf16 and is excluded from the mechanism comparison; the 160M
experiments reported here sample inputs online.

The additional five-million-update runs are listed in
App.~\ref{app:long-finals}; the three integer Heisenberg models and their
probe/intervention coverage are specified in App.~\ref{app:heisenberg}.

\begingroup\footnotesize
\setlength{\tabcolsep}{5.5pt}
\setlength{\LTleft}{\fill}
\setlength{\LTright}{\fill}
\setlength{\LTcapwidth}{\linewidth}
\begin{longtable}{>{\raggedright\arraybackslash}p{0.28\textwidth}>{\raggedright\arraybackslash}p{0.16\textwidth}>{\raggedright\arraybackslash}p{0.10\textwidth}>{\raggedright\arraybackslash}p{0.19\textwidth}>{\raggedright\arraybackslash}p{0.13\textwidth}}
\caption{Architecture coverage, part 1 of 3. Status refers to the evaluation package; completed packages need not reach full tracking. All listed outcomes and seed counts are retained.}
\label{tab:architecture-coverage-1}\\[4pt]
\toprule
Architecture & Recipe / measure & Groups; seeds & Outcome & Status; location \\
\midrule
\endfirsthead
\multicolumn{5}{l}{\small\tablename\ \thetable\ (continued)}\\[4pt]
\toprule
Architecture & Recipe / measure & Groups; seeds & Outcome & Status; location \\
\midrule
\endhead
\midrule
\multicolumn{5}{r}{\footnotesize Continued on next page}\\
\endfoot
\bottomrule
\endlastfoot
GPT-NeoX 4L, $d$=256, 4 heads, quarter rotary (recipe~A) & A; full alphabet & 13 groups; 3 & plateau law & completed; \S\ref{sec:law} \\[2pt]
Recipe~A at 2 and 6 layers & A; full alphabet & $D_4$; 3+3 & ratio .4986--.5034; frontier 2--3 at two layers, 5--12 at six & completed; \S\ref{sec:levers} \\[2pt]
Recipe~A at 2/3/5/6 layers, widths 364/296/228/208 (parameter-matched) and 256 (shared grid) & A; full alphabet & $D_4$; 3 each & ratio .4991--.5017, frontier 3--8, no escape & completed; \S\ref{sec:levers} \\[2pt]
Recipe~A at 2 and 6 layers, and at the doubled budget & A; full alphabet & $A_4$, $\mathrm{SL}(2,3)$; 3 each & ratio $1/f\pm.0011$; frontier $+3$/$+5$ on $A_4$, $+2$/$+1$ on $\mathrm{SL}(2,3)$ & completed; \S\ref{sec:levers} \\[2pt]
Recipe~A at 4, 8, 12 layers & A; full alphabet & $S_5$; 36 pairs & parity-first or no progress; 0 of 36 associative & completed; \S\ref{sec:levers} \\[2pt]
Recipe~A at $d$=512, 8 heads & A; full alphabet & $D_4$; 3 & plateau, ratio .4998--.5013 & completed; \S\ref{sec:levers} \\[2pt]
Recipe~A + sinusoidal absolute PE & A; full alphabet & $D_4$; 3 & slower quotient, ratio .4995--.5011 & completed; \S\ref{sec:levers} \\[2pt]
Recipe~A backbone under recipe~B optimizer ($10^{-4}$; $3\times10^{-4}$) & full alphabet & $D_4$; 3+3 & plateau (.499--.500); collapse to uniform & completed; App.~\ref{app:support-details} \\[2pt]
GPT-2 block, learned absolute PE, 4L/256 & A; full alphabet & $D_4$; 3 & frontier 15--24, ratio .4993--.5008 & completed; \S\ref{sec:levers} \\[2pt]
Llama block (full rotary, SwiGLU, RMSNorm), 4L/256 & A; full alphabet & $D_4$; 3 & frontier 18--20, ratio .4977--.4997 & completed; \S\ref{sec:levers} \\
\end{longtable}
\endgroup

\begingroup\footnotesize
\setlength{\tabcolsep}{5.5pt}
\setlength{\LTleft}{\fill}
\setlength{\LTright}{\fill}
\setlength{\LTcapwidth}{\linewidth}
\begin{longtable}{>{\raggedright\arraybackslash}p{0.28\textwidth}>{\raggedright\arraybackslash}p{0.16\textwidth}>{\raggedright\arraybackslash}p{0.10\textwidth}>{\raggedright\arraybackslash}p{0.19\textwidth}>{\raggedright\arraybackslash}p{0.13\textwidth}}
\caption{Architecture coverage, part 2 of 3. Status refers to the evaluation package; completed packages need not reach full tracking. All listed outcomes and seed counts are retained.}
\label{tab:architecture-coverage-2}\\[4pt]
\toprule
Architecture & Recipe / measure & Groups; seeds & Outcome & Status; location \\
\midrule
\endfirsthead
\multicolumn{5}{l}{\small\tablename\ \thetable\ (continued)}\\[4pt]
\toprule
Architecture & Recipe / measure & Groups; seeds & Outcome & Status; location \\
\midrule
\endhead
\midrule
\multicolumn{5}{r}{\footnotesize Continued on next page}\\
\endfoot
\bottomrule
\endlastfoot
NeoX 4L/256, no position embedding & A; full alphabet & $D_4$; 1 & frontier 5, quotient unfinished, ratio .5002 & completed; \S\ref{sec:levers} \\[2pt]
Recipe~A on the generator stream & A; $\{r,s\}$ & $D_4$; 3 & plateau, ratio .499--.516 & completed; App.~\ref{app:support-details} \\[2pt]
GPT-2 sinusoidal 3L, $d$=512, 8 heads (recipe~B) & Liu; $\{r,s\}$ & $D_4$; 3 (+1 at depth 4, +1 at $3\times10^{-4}$) & escape 3/3; depth 4 escapes; $3\times10^{-4}$ partial & completed; App.~\ref{app:support-details} \\[2pt]
Recipe~B minus PE; on full alphabet; with recipe~A optimizer; MLP 2048 & Liu / A block & $D_4$; 3 each & chance; plateau 2/3; frozen frontier; escape & completed; App.~\ref{app:support-details} \\[2pt]
GPT-2 rotary; NeoX sinusoidal; NeoX rotary, all 3L/512 & Liu; $\{r,s\}$ & $D_4$; 3 each & escape 1 / partial 2 (ratio .82--.998); escape 1 / partial 2 (.86--.99); plateau (.524--.532) & completed; NeoX-rotary cell in \S\ref{sec:levers}, the other two only here \\[2pt]
NeoX rotary 3L/512 at 4 and 5 layers & Liu; $\{r,s\}$ & $D_4$; 3 each & ratio .814, .951; frontier 31--32, 51--58 & completed; \S\ref{sec:levers} \\[2pt]
Comparison~C: GPT-2 / NeoX $\times$ sinusoidal / rotary, 3L/512 & Liu; full alphabet & $D_4$; 3 each & plateau 10/12, partial 2/12, escape 0 & completed; \S\ref{sec:law} \\[2pt]
Pythia-160M, 12 blocks & Li; online & $S_3$; 3 scratch + 1 pretrained & 1 parity-first, 3 associative & completed; \S\ref{sec:fiber}, App.~\ref{app:gradient-details} \\[2pt]
Pythia-160M architecture, 3/4/7/12 blocks, paired & Li; online & $S_3$; 8 (+3 bridge) & frontier median 6.5/9/15/39; quotient 7/5/7/7 of 8 & completed; \S\ref{sec:levers} \\[2pt]
Pythia-70M & Li; online & $S_3$; 1+1 & quotient-first; stall & exploratory; not used \\
\end{longtable}
\endgroup

\begingroup\footnotesize
\setlength{\tabcolsep}{5.5pt}
\setlength{\LTleft}{\fill}
\setlength{\LTright}{\fill}
\setlength{\LTcapwidth}{\linewidth}
\begin{longtable}{>{\raggedright\arraybackslash}p{0.28\textwidth}>{\raggedright\arraybackslash}p{0.16\textwidth}>{\raggedright\arraybackslash}p{0.10\textwidth}>{\raggedright\arraybackslash}p{0.19\textwidth}>{\raggedright\arraybackslash}p{0.13\textwidth}}
\caption{Architecture coverage, part 3 of 3. Status refers to the evaluation package; completed packages need not reach full tracking. All listed outcomes and seed counts are retained.}
\label{tab:architecture-coverage-3}\\[4pt]
\toprule
Architecture & Recipe / measure & Groups; seeds & Outcome & Status; location \\
\midrule
\endfirsthead
\multicolumn{5}{l}{\small\tablename\ \thetable\ (continued)}\\[4pt]
\toprule
Architecture & Recipe / measure & Groups; seeds & Outcome & Status; location \\
\midrule
\endhead
\midrule
\multicolumn{5}{r}{\footnotesize Continued on next page}\\
\endfoot
\bottomrule
\endlastfoot
Looped transformer, 2 shared blocks, no position encoding, $d$=128 & curriculum to 16 & $D_4$; 1 & one position per loop; no quotient shelf & pilot; not used \\[2pt]
Expert-split MLP; MoE readout & A & $D_4$; 1 pair each & frontier 11 vs 6; fiber .498--.506 & exploratory; not used \\[2pt]
Bilinear recurrence, state 128 & Adam $10^{-3}$; full alphabet & 10 groups; 3 (+10) & $1/|N|$ floors 13/13; acquired across positions in 16/16 & completed; \S\ref{sec:levers} \\[2pt]
Two-layer polynomial net on group-Fourier input & 6 configs, 30k & $D_4$; 1 & amplitude sets learning order & exploratory; not used \\[2pt]
Pooled (order-blind by construction) MLP; pooled + recurrent sum & Adam $10^{-3}$, 20k & $S_3\times C_3$; 1 & quotient only .40 at 20k; registered ceiling test missed (gap .255) & exploratory; not used \\[2pt]
Two-layer LSTM, width 444 (parameter-matched to recipe~A) & A; full alphabet & $D_4$, $A_4$, $S_4$; 3 per arm, 2 arms & escape 18/18; transient stages at $1/|N|$ levels & completed; \S\ref{sec:levers} \\[2pt]
Two-layer GRU, width 512 (parameter-matched to recipe~A) & A; full alphabet & $D_4$, $A_4$, $S_4$; 3 per arm, 2 arms & escape 18/18; 1,000-step evaluation grid, no dense replay & completed; \S\ref{sec:levers} \\[2pt]
Mamba-2, 2L/496, state 80 (parameter-matched) & Full alphabet; two learning rates; FP32 & $D_4$, $A_4$, $S_4$; 3 per arm & at $10^{-3}$, 6/9 resolve the quotient; conditional exact near $1/|N|$ & completed; App.~\ref{app:ssm} \\[2pt]
Mamba-2, 4L/352, state 48 (parameter-matched) & Full alphabet; $5\times10^{-5}$; FP32 & $D_4$; 3 & quotient 98.82--99.97\%; conditional exact near $1/2$ & completed; App.~\ref{app:ssm} \\[2pt]
Not run or incomplete & & & recipe~B quarter-rotary arms; parameter-matched recipe~B depth arms (registered, running); $Q_8$ support attribution & \\
\end{longtable}
\endgroup

\begingroup\footnotesize
\setlength{\tabcolsep}{5.5pt}
\setlength{\LTleft}{\fill}
\setlength{\LTright}{\fill}
\setlength{\LTcapwidth}{\linewidth}
\begin{longtable}{>{\raggedright\arraybackslash}p{0.28\textwidth}>{\raggedright\arraybackslash}p{0.15\textwidth}>{\raggedright\arraybackslash}p{0.13\textwidth}>{\raggedright\arraybackslash}p{0.19\textwidth}>{\raggedright\arraybackslash}p{0.11\textwidth}}
\caption{Training-support and encoding coverage, part 1 of 2. Outcomes refer to the listed recipes and budgets. Pilot and inconclusive results are retained.}
\label{tab:support-coverage-1}\\[4pt]
\toprule
Training support / encoding & Recipe & Groups; seeds & Outcome & Location \\
\midrule
\endfirsthead
\multicolumn{5}{l}{\small\tablename\ \thetable\ (continued)}\\[4pt]
\toprule
Training support / encoding & Recipe & Groups; seeds & Outcome & Location \\
\midrule
\endhead
\midrule
\multicolumn{5}{r}{\footnotesize Continued on next page}\\
\endfoot
\bottomrule
\endlastfoot
Full alphabet, i.i.d.\ uniform & A, C, LSTM, bilinear, 160M & all & quotient stages and $1/|N|$ levels & \S\ref{sec:law} \\[2pt]
Generator stream $\{r,s\}$, i.i.d. & B; A & $D_4$; 3 per arm & escape (B); plateau (A) & App.~\ref{app:support-details} \\[2pt]
Two $[G,G]$-cosets of $D_4$, $\{rK,sK\}$, lift bias $\varepsilon\in\{0,.5,.9,.95,.99\}$ & B & $D_4$; 3 at $\varepsilon$=0 and .99, 1 otherwise & crawl or escape at $\varepsilon$=0; the $\varepsilon$ axis is retired as an explanatory variable & App.~\ref{app:support-details} \\[2pt]
Three $[G,G]$-cosets of $D_4$ & B & $D_4$; 3 & flat, frontier 7--11 & App.~\ref{app:support-details} \\[2pt]
Two cosets of $Q_8$, $\{iK,jK\}$; matched full-support baseline & B & $Q_8$; 3+1 & all crawl; baseline flat & App.~\ref{app:support-details} \\[2pt]
Single reflection coset of $D_3$, $D_5$, $D_7$ & B & 3, 1, 3 & escape at 3; crawl at 5, 7 & App.~\ref{app:support-details} \\[2pt]
Ladder supports: reflections of $D_9$, $D_{15}$, $D_{27}$; odd permutations of $S_4$; two cosets of $D_8$, $Q_{16}$; non-multiples of 3 in $C_9$ & B & 7 groups; 3 each & stages at intermediate quotients & \S\ref{sec:ladder} \\[2pt]
Horizons 24, 48, 100; curriculum 24$\to$48$\to$100; reverse continuation to 24 & A & $D_4$; 3, 3, 3; 2 & frontier 10--13 at 24; no curriculum gain & \S\ref{sec:levers} \\[2pt]
Frozen Fourier encoding, 2-dim irrep amplitude $\times\{1/4,1,2,4,8\}$, fp32 control & A & $D_4$; 17 runs & frontier 4.3$\to$9.5, characters starved & \S\ref{sec:levers} \\
\end{longtable}
\endgroup

\begingroup\footnotesize
\setlength{\tabcolsep}{5.5pt}
\setlength{\LTleft}{\fill}
\setlength{\LTright}{\fill}
\setlength{\LTcapwidth}{\linewidth}
\begin{longtable}{>{\raggedright\arraybackslash}p{0.28\textwidth}>{\raggedright\arraybackslash}p{0.15\textwidth}>{\raggedright\arraybackslash}p{0.13\textwidth}>{\raggedright\arraybackslash}p{0.19\textwidth}>{\raggedright\arraybackslash}p{0.11\textwidth}}
\caption{Training-support and encoding coverage, part 2 of 2. Outcomes refer to the listed recipes and budgets. Pilot and inconclusive results are retained.}
\label{tab:support-coverage-2}\\[4pt]
\toprule
Training support / encoding & Recipe & Groups; seeds & Outcome & Location \\
\midrule
\endfirsthead
\multicolumn{5}{l}{\small\tablename\ \thetable\ (continued)}\\[4pt]
\toprule
Training support / encoding & Recipe & Groups; seeds & Outcome & Location \\
\midrule
\endhead
\midrule
\multicolumn{5}{r}{\footnotesize Continued on next page}\\
\endfoot
\bottomrule
\endlastfoot
Frozen Peter--Weyl encoding, tied scores & A & $S_3\times C_3$; 3 & characters learned, 2-dim irrep not & App.~\ref{app:training-details} \\[2pt]
$S_3\times C_3$ encoding: frozen / trainable $\times$ amplitude 1, 2 & A & 1 each, 10k steps & no rescue of the 2-dim irrep & pilot; incomplete; not used \\[2pt]
Equal-dimension embedding rows, four arms & 12L/256, Li schedule & $S_3$, $S_5$; 1k steps & frontier 2 in every arm & incon\-clusive; not used \\[2pt]
Coset-union flows of $S_3\times C_3$, $D_5$, $D_7$, $Q_8$, $S_4$, $S_5$ & 3L/128 pilot, 6k steps, length 64 & 6 groups; 1 & quotient learned; exact$\mid$quotient stays at $1/|[G,G]|$; gain vanishes on the full stream & exploratory; not used \\[2pt]
Exact Fourier degree of the fiber and quotient sectors under a measure family from the full alphabet to the generator stream & zero training & $D_4$, $S_3$, $D_6$, $Q_8$ & pure degree $t$ under the full alphabet; reduced degree only on unpaired-lift supports & completed; App.~\ref{app:support-details} \\[2pt]
Released pipeline of \citet{li2025state} on $S_5$: fixed 1M-sequence dataset repeated for 20 epochs, bf16 & their code and architecture & $S_5$; 1 & remained at chance under bf16; excluded from mechanism analysis & not used \\[2pt]
Not run & & & Markov or correlated tokens; power-law frequencies; tilted full-support measures & \\
\end{longtable}
\endgroup

\section{Runs behind Table~\ref{tab:law}}
\label{app:rows}

The first table preserves the original finite-group suite:
three seeds per group at 74k updates, plus two $D_4$, three $Q_8$, and three
$D_6$ runs at 148k. Its manifest is \nolinkurl{paper/figures/law_rows_current.json}.

The original manifest, \texttt{law\_rows.json}, is retained for
Fig.~\ref{fig:frontier}. \emph{quotient} is the quotient-channel accuracy;
dashes denote the abelian and simple-group controls. The $S_3$ rows use
the updated evaluation described in App.~\ref{app:output-details}.

\begingroup
\small
\begin{longtable}{llrrrr}
\caption{Original finite-group runs retained for reproducibility.}
\label{tab:original-law-runs}\\
\toprule
$G$ & seed & budget & exact$\mid$quotient & quotient & exact \\
\midrule
\endfirsthead
\multicolumn{6}{l}{\tablename\ \thetable\ (continued)}\\
\toprule
$G$ & seed & budget & exact$\mid$quotient & quotient & exact \\
\midrule
\endhead
\midrule
\multicolumn{6}{r}{Continued on next page}\\
\endfoot
\bottomrule
\endlastfoot
$C_8$ & 42 & 74k & 0.9918 & -- & 0.9918 \\
$C_8$ & 43 & 74k & 0.9945 & -- & 0.9945 \\
$C_8$ & 44 & 74k & 0.9849 & -- & 0.9849 \\
$D_4$ & 42 & 74k & 0.5002 & 0.990 & 0.4950 \\
$D_4$ & 43 & 74k & 0.4992 & 0.951 & 0.4748 \\
$D_4$ & 44 & 74k & 0.4999 & 0.849 & 0.4245 \\
$D_4$ & 43 & 148k & 0.5005 & 0.999 & 0.5002 \\
$D_4$ & 44 & 148k & 0.4995 & 1.000 & 0.4994 \\
$Q_8$ & 42 & 74k & 0.4989 & 0.826 & 0.4122 \\
$Q_8$ & 43 & 74k & 0.4992 & 0.761 & 0.3799 \\
$Q_8$ & 44 & 74k & 0.4997 & 0.630 & 0.3146 \\
$Q_8$ & 42 & 148k & 0.4996 & 0.999 & 0.4993 \\
$Q_8$ & 43 & 148k & 0.4991 & 0.999 & 0.4987 \\
$Q_8$ & 44 & 148k & 0.4996 & 0.999 & 0.4992 \\
$S_3$ & 42 & 74k & 0.3338 & 0.999 & 0.3335 \\
$S_3$ & 43 & 74k & 0.3339 & 0.999 & 0.3336 \\
$S_3$ & 44 & 74k & 0.3332 & 0.999 & 0.3329 \\
$D_6$ & 42 & 74k & 0.3322 & 0.944 & 0.3137 \\
$D_6$ & 43 & 74k & 0.3327 & 0.819 & 0.2726 \\
$D_6$ & 44 & 74k & 0.3341 & 0.912 & 0.3047 \\
$D_6$ & 42 & 148k & 0.3345 & 0.999 & 0.3341 \\
$D_6$ & 43 & 148k & 0.3331 & 0.999 & 0.3326 \\
$D_6$ & 44 & 148k & 0.3347 & 1.000 & 0.3345 \\
$\mathrm{Dic}_3$ & 42 & 74k & 0.3336 & 0.997 & 0.3326 \\
$\mathrm{Dic}_3$ & 43 & 74k & 0.3327 & 0.995 & 0.3311 \\
$\mathrm{Dic}_3$ & 44 & 74k & 0.3325 & 0.999 & 0.3323 \\
$A_4$ & 42 & 74k & 0.2498 & 1.000 & 0.2497 \\
$A_4$ & 43 & 74k & 0.2503 & 1.000 & 0.2502 \\
$A_4$ & 44 & 74k & 0.2509 & 1.000 & 0.2509 \\
$D_5$ & 42 & 74k & 0.2007 & 1.000 & 0.2006 \\
$D_5$ & 43 & 74k & 0.1999 & 0.893 & 0.1786 \\
$D_5$ & 44 & 74k & 0.2002 & 0.997 & 0.1995 \\
$C_7\!\rtimes\!C_3$ & 42 & 74k & 0.1430 & 0.999 & 0.1429 \\
$C_7\!\rtimes\!C_3$ & 43 & 74k & 0.1424 & 0.999 & 0.1423 \\
$C_7\!\rtimes\!C_3$ & 44 & 74k & 0.1427 & 1.000 & 0.1426 \\
$\mathrm{SL}(2,3)$ & 42 & 74k & 0.1251 & 0.999 & 0.1250 \\
$\mathrm{SL}(2,3)$ & 43 & 74k & 0.1246 & 1.000 & 0.1245 \\
$\mathrm{SL}(2,3)$ & 44 & 74k & 0.1248 & 0.999 & 0.1247 \\
$S_4$ & 42 & 74k & 0.0830 & 1.000 & 0.0829 \\
$S_4$ & 43 & 74k & 0.0836 & 0.998 & 0.0835 \\
$S_4$ & 44 & 74k & 0.0833 & 1.000 & 0.0833 \\
$S_5$ & 42 & 74k & 0.0166 & 1.000 & 0.0166 \\
$S_5$ & 43 & 74k & 0.0166 & 1.000 & 0.0166 \\
$S_5$ & 44 & 74k & 0.0165 & 0.794 & 0.0131 \\
$A_5$ & 42 & 74k & -- & -- & 0.0165 \\
$A_5$ & 43 & 74k & -- & -- & 0.0168 \\
$A_5$ & 44 & 74k & -- & -- & 0.0168 \\
\end{longtable}

\endgroup

Table~\ref{tab:law} uses all available matching 148k runs for $D_4$, $Q_8$,
$D_6$, $A_4$, $\mathrm{SL}(2,3)$, and $H_3(\mathbb Z/5)$, and all original
74k runs for the remaining groups. Its manifest is
\nolinkurl{paper/tables/tab1_law_manifest.json}.
This is a post-hoc choice of reporting budget, without a seed-level accuracy
filter. Some longer-budget cohorts were commissioned after incomplete quotient
learning at the initial budget; $D_4$ has only two doubled-budget seeds.
The table therefore summarizes observed resolutions at stated budgets,
rather than estimating their frequency under a common training budget.
The Heisenberg runs and additional doubled-budget evaluations appear below.

\subsection{Finite and infinite Heisenberg groups in Table~\ref{tab:law}}
\label{app:heisenberg-law}

Write Heisenberg elements as $(x,y,z)$, with product
$(x,y,z)(u,v,w)=(x+u,y+v,z+w+xv)$.
For $H_3(\mathbb Z/p)$ all coordinates are reduced modulo $p$;
the group has $p^3$ elements, abelianization $(\mathbb Z/p)^2$, and
commutator subgroup of order $p$. We use $p=3,5$, uniform full-group
inputs, and an element-index output head.
For $G_p=H_3(\mathbb Z)/\langle z^p\rangle$, only the central coordinate
is reduced modulo $p$. This group is infinite, with abelianization
$\mathbb Z^2$ and $p$ elements in each abelianization class. For $p=3,4,5$,
inputs are uniform on $\{-1,0,1\}^2\times\mathbb Z/p$: a finite alphabet
of $9p$ elements containing nine complete commutator cosets.
Separate output heads classify $x,y\in[-128,128]$ and $z\in\mathbb Z/p$.
The quotient readout requires both $x$ and $y$ to be correct.

All five conditions use four GPT-NeoX blocks, width 256, four attention
heads, and seeds 42--44. Training uses AdamW at $5\times10^{-5}$, zero
weight decay, linear decay without warmup, batch size 256, length 100,
74,219 updates, and bf16 autocast. A further three full-alphabet
$H_3(\mathbb Z/5)$ runs use the same recipe with 148,438 updates.
Final evaluations use 8,192 iid
length-128 words. Table~\ref{tab:heisenberg-law-runs} uses the same positions
17--100 for every run and reports the ratio of pooled exact to pooled
quotient accuracy. Exact correctness implies quotient correctness in both
readouts, so this ratio is conditional exact accuracy.

The initial $H_3(\mathbb Z/5)$ runs have quotient accuracies .9458--.9966;
all remain below, including the two below the registered .99 criterion.
Table~\ref{tab:law} uses the full doubled-budget cohort, with quotient
accuracy .9987--.9997 and conditional exact accuracy .1995--.1996.
For $G_5$, positions 17--100 overlap the exact-accuracy transition and give
conditional accuracy .2115--.2181. At positions 70--100, exact accuracy is
.1995--.1997. We retain the common window in Table~\ref{tab:law} rather than
substituting this later window. The agreement with $1/p$ is an empirical
result under these configurations, not a consequence of finite commutator
size alone. Finite-group element classification and infinite-group
coordinate prediction also differ in their support and output heads.

\begin{table}[ht]
\centering\small
\begin{tabular}{lrrrr}
\toprule
$G$ & Seed & Exact$\mid$quotient & Quotient & Exact \\
\midrule
$H_3(\mathbb Z/3)$ & 42 & 0.3343 & 0.9991 & 0.3339 \\
$H_3(\mathbb Z/3)$ & 43 & 0.3332 & 0.9982 & 0.3326 \\
$H_3(\mathbb Z/3)$ & 44 & 0.3336 & 0.9996 & 0.3334 \\
\midrule
$H_3(\mathbb Z/5)$ & 42 & 0.1999 & 0.9966 & 0.1993 \\
$H_3(\mathbb Z/5)$ & 43 & 0.1997 & 0.9786 & 0.1954 \\
$H_3(\mathbb Z/5)$ & 44 & 0.1986 & 0.9458 & 0.1879 \\
\midrule
$G_3$ & 42 & 0.3329 & 0.9996 & 0.3328 \\
$G_3$ & 43 & 0.3340 & 0.9996 & 0.3339 \\
$G_3$ & 44 & 0.3334 & 0.9996 & 0.3332 \\
\midrule
$G_4$ & 42 & 0.2505 & 0.9996 & 0.2504 \\
$G_4$ & 43 & 0.2511 & 0.9995 & 0.2510 \\
$G_4$ & 44 & 0.2499 & 0.9997 & 0.2498 \\
\midrule
$G_5$ & 42 & 0.2181 & 0.9992 & 0.2179 \\
$G_5$ & 43 & 0.2159 & 0.9993 & 0.2158 \\
$G_5$ & 44 & 0.2115 & 0.9993 & 0.2113 \\
\bottomrule
\end{tabular}

\caption{Original Heisenberg runs. All values use positions
17--100 and the original final evaluations; each condition has 74,219 updates.
Seed identifiers are retained here for reproducibility.}
\label{tab:heisenberg-law-runs}
\end{table}

\begin{table}[ht]
\centering\small
\begin{tabular}{lrrrrr}
\toprule
$G$ & Seed & Updates & Exact$\mid$quotient & Quotient & Exact \\
\midrule
$A_4$ & 42 & 148k & 0.2502 & 0.9997 & 0.2501 \\
$A_4$ & 43 & 148k & 0.2502 & 0.9997 & 0.2501 \\
$A_4$ & 44 & 148k & 0.2500 & 0.9997 & 0.2500 \\
$\mathrm{SL}(2,3)$ & 42 & 148k & 0.1249 & 0.9998 & 0.1249 \\
$\mathrm{SL}(2,3)$ & 43 & 148k & 0.1245 & 0.9998 & 0.1245 \\
$\mathrm{SL}(2,3)$ & 44 & 148k & 0.1246 & 0.9995 & 0.1245 \\
$H_3(\mathbb Z/5)$ & 42 & 148k & 0.1996 & 0.9997 & 0.1995 \\
$H_3(\mathbb Z/5)$ & 43 & 148k & 0.1995 & 0.9988 & 0.1992 \\
$H_3(\mathbb Z/5)$ & 44 & 148k & 0.1995 & 0.9987 & 0.1992 \\
\bottomrule
\end{tabular}

\caption{Additional doubled-budget runs used in Table~\ref{tab:law}.
All values use positions 17--100 of the independent final evaluation.
The doubled-budget $D_4$, $Q_8$, and $D_6$ runs are listed above with the
original finite-group suite.}
\label{tab:law-extended-runs}
\end{table}

\FloatBarrier

\section{Readout definitions and window sensitivity}
\label{app:readouts}

Final evaluations use 8,192--16,384 held-out sequences. The default windows
are positions 17--100 for recipe A and 9--100 for recipe B; we separately
check whether they overlap the exact frontier or its transition region.

\paragraph{Within-class output diagnostics.}
We ask whether a prediction distinguishes the true state from the other members
of its quotient class. Let $s=x_{1:t}$ denote an input prefix, $q_t$ its true
product, and $C_s=q_tN$ its finite true class. For the model output
$p_\theta(g\mid s)$, define the probability mass on this class and the
renormalized distribution within it by
\[
m_s=\sum_{g\in C_s}p_\theta(g\mid s),\qquad
\tilde p_s(g)=\frac{p_\theta(g\mid s)}{m_s}\quad(g\in C_s).
\]
This separates two questions. The mass $m_s$ measures how much probability
reaches the correct class; $\tilde p_s$ measures how it is divided among class
members. The identity
$\log_2 p_\theta(q_t\mid s)=\log_2 m_s+\log_2\tilde p_s(q_t)$
makes that separation explicit. We use the true class for this diagnostic even
when the model's most probable state lies outside it.

\paragraph{Deriving the log-likelihood gain.}
A model that knows only the class and predicts uniformly assigns $1/|N|$ to
the true state. The improvement in log score for one prefix is therefore
\[
d_s=\log_2\tilde p_s(q_t)-\log_2(1/|N|)
   =\log_2\bigl(|N|\tilde p_s(q_t)\bigr).
\]
Averaging yields
\begin{equation}
\Delta\mathrm{LL}=\mathbb E_s[d_s]
=\log_2|N|-\mathbb E_s[-\log_2\tilde p_s(q_t)].
\label{eq:fiber-readouts}
\end{equation}
Thus $\Delta\mathrm{LL}$ is the reduction in within-class cross-entropy relative
to uniform prediction. Positive values mean better average log scores, zero
means equal scores, and negative values mean worse scores. For a four-state
class, assigning the true state probabilities $1/4$, $1/2$, and $1/8$ gives
per-prefix gains of $0$, $1$, and $-1$ bits, respectively. A negative gain is
not negative mutual information; this statistic is a prediction score.

\paragraph{Why also measure KL divergence.}
A near-zero mean gain does not establish uniform outputs. Nonuniform
predictions can favor wrong states, and positive and negative gains can offset
across examples. We therefore also compute
\[
D=\mathbb E_s\!\left[\sum_{g\in C_s}\tilde p_s(g)
       \log_2\bigl(|N|\tilde p_s(g)\bigr)\right]
 =\log_2|N|-\mathbb E_s[H(\tilde p_s)]\ge0.
\]
Here $H$ is entropy in bits. Unlike $\Delta\mathrm{LL}$, this quantity detects
any concentration within a class, whether or not it favors the true state.
It is zero exactly when the within-class prediction is uniform almost surely.
Small $D$ supports near-uniform outputs on average; neither diagnostic rules
out information in hidden representations that the output head does not use.
For abelianization classes, $N=[G,G]$ and $|N|=f$.

\paragraph{Averaging and the plotted example.}
We compute each diagnostic per example and position, then average over the
specified evaluation window. We include all examples regardless of quotient
prediction. These expectations are within one model, not averages across seeds
or diagnostics of an averaged confusion matrix. For $A_4$ in
Sec.~\ref{sec:fiber}, 8,192 sequences and positions 17--100 give 688,128
observations. Figure~\ref{fig:quotient-output}b instead shows per-position means
over the same 8,192 sequences at positions 1--100. The separate output-matrix
replay averages probabilities first, so its block structure alone cannot
establish that each prediction is nearly uniform. The per-example diagnostics
test that possibility.

\paragraph{Frontiers and threshold selection.} The exact frontier $F_e$ is
the longest prefix whose exact accuracy is at least .75 at every position.
The quotient frontier $F_q$ uses .90. The .75 threshold was selected on
2026-08-22 after an exploratory inspection of fourteen runs and then fixed;
it was not selected before those runs. The .90 threshold and package-specific
criteria were fixed before the runs to which they apply. The intermediate-
quotient analysis also uses a .70 class-accuracy threshold selected
at analysis time. We distinguish completion of a registered evaluation package
from a persistent stage in a training trajectory; neither establishes
asymptotic convergence.

\paragraph{Window checks.} Baseline windows cover positions 17--100; recipe B
and its full-group variants initially used 9--100. A plateau window should
exclude the exact frontier and its transition region. The standard check starts
at least six positions beyond $F_e$, supplemented by position-wise readouts.
In a standard-budget battery, 28 of 28 runs have
$|\Delta\mathrm{LL}|<.0145$ bits beyond $F_e+6$. Longer training can broaden
the transition: a 600k-update $D_4$ run with frontier 39 has conditional
accuracies .75, .64, .54, and .50 at offsets 1, 11, 21, and 31, respectively.
Thus six positions is not a universal bound on transition width.

\paragraph{Full-group cross-recipe correction.} Three GPT-2/sinusoidal runs
under the full-group variant of B initially gave conditional accuracies
.5489, .5321, and .5252. Their exact frontiers, 17, 13, and 12, overlap the
original 9--100 window. Restricting the window to 17--100 gives argmax ratios
.5071, .5015, and .5002, with mean true-element masses .5033, .5008, and .5002;
residual information is at most .00621 bits. Starting only one position beyond
the frontier leaves .0017--.0120 bits and two ambiguous cells under the original
criterion. These are window corrections, not independent replications.

The other backbone/position cells include an uncorrected .5230 ratio whose
frontier overlaps the window; no corrected reading is available. Two
NeoX/sinusoidal seeds have incomplete quotient accuracy (.795 and .717).
Elsewhere, fixed 17--100 windows give ratios .6475, .5302, and .5370 when they
include positions already tracked exactly. Such values are retained in the
source records and do not constitute measurements of a uniform plateau.

\section{Additional evidence for quotient predictions}
\label{app:output-details}

\paragraph{Coverage of the baseline table.} The three standard-budget $D_4$
ratios are .50018, .49922, and .49995. One run passed its registered criterion,
one had insufficient statistical power, and one had quotient accuracy .849.
Two doubled-budget runs give .50047 and .49952. Table~\ref{tab:law} uses
these two; App.~\ref{app:rows} retains all five evaluations.
The updated $S_3$ cell uses seeds 42--44, each trained for 74,219 updates
under the same recipe. FP32 evaluation on 8,192 sequences at seed 248041
over positions 17--100 gives quotient accuracy 99.89--99.92\% and
conditional exact accuracy 33.32--33.39\%. Quotient predictions map the
full-state argmax to its coset. Seeds 43--44 were added to the original
seed 42; all three use this evaluation protocol. The $C_8$ and $A_5$
controls use exact accuracy. For comparison with the original $C_8$ evaluation
in Table~\ref{tab:law}, FP32 replay on 8,192 sequences over positions 17--100
gives 99.2555\%, 99.5356\%, and 98.5979\% for seeds 42, 43, and 44.
These are the main-text control values; the table retains the original
bf16 evaluation over positions 16--100. The reported cohorts and budgets are listed below.

\begin{table}[ht]
\centering\footnotesize
\setlength{\tabcolsep}{2.5pt}
\begin{tabular}{@{}lrrrrrrrrr@{}}
\toprule
$G$ & $|G|$ & $f$ & $1/f$ & Cond. (\%) & Quot. (\%) & Exact (\%) & Chance (\%) & $\times$chance & Updates \\
\midrule
$C_8$ & 8 & 1 & 1 & -- & -- & 98.49--99.45 & 12.50 & 7.88--7.96 & 74k \\
$D_4^{*}$ & 8 & 2 & $1/2$ & 49.95--50.05 & 99.94--99.99 & 49.94--50.02 & 12.50 & 4.00 & 148k \\
$Q_8^{*}$ & 8 & 2 & $1/2$ & 49.91--49.96 & 99.92--99.93 & 49.87--49.93 & 12.50 & 3.99 & 148k \\
$S_3$ & 6 & 3 & $1/3$ & 33.32--33.39 & 99.89--99.92 & 33.29--33.36 & 16.67 & 2.00 & 74k \\
$D_6^{*}$ & 12 & 3 & $1/3$ & 33.31--33.47 & 99.86--99.97 & 33.26--33.45 & 8.33 & 3.99--4.01 & 148k \\
$\mathrm{Dic}_3$ & 12 & 3 & $1/3$ & 33.25--33.36 & 99.51--99.93 & 33.11--33.26 & 8.33 & 3.97--3.99 & 74k \\
$A_4^{*}$ & 12 & 4 & $1/4$ & 25.00--25.02 & 99.97 & 25.00--25.01 & 8.33 & 3.00 & 148k \\
$D_5$ & 10 & 5 & $1/5$ & 19.99--20.07 & 89.34--99.99 & 17.86--20.06 & 10.00 & 1.79--2.01 & 74k \\
$C_7\!\rtimes\!C_3$ & 21 & 7 & $1/7$ & 14.24--14.30 & 99.93--99.95 & 14.23--14.29 & 4.76 & 2.99--3.00 & 74k \\
$\mathrm{SL}(2,3)^{*}$ & 24 & 8 & $1/8$ & 12.45--12.49 & 99.95--99.98 & 12.45--12.49 & 4.17 & 2.99--3.00 & 148k \\
$S_4$ & 24 & 12 & $1/12$ & 8.30--8.36 & 99.77--99.98 & 8.29--8.35 & 4.17 & 1.99--2.00 & 74k \\
$S_5$ & 120 & 60 & $1/60$ & 1.65--1.66 & 79.36--99.99 & 1.31--1.66 & 0.83 & 1.57--1.99 & 74k \\
$A_5$ & 60 & 60 & $1/60$ & -- & -- & 1.65--1.68 & 1.67 & 0.99--1.01 & 74k \\
\midrule
$H_3(\mathbb Z/3)$ & 27 & 3 & $1/3$ & 33.32--33.43 & 99.82--99.96 & 33.26--33.39 & 3.70 & 8.98--9.02 & 74k \\
$H_3(\mathbb Z/5)^{*}$ & 125 & 5 & $1/5$ & 19.95--19.96 & 99.87--99.97 & 19.92--19.95 & 0.80 & 24.90--24.94 & 148k \\
\midrule
$G_3$ & $\infty$ & 3 & $1/3$ & 33.29--33.40 & 99.96 & 33.28--33.39 & -- & -- & 74k \\
$G_4$ & $\infty$ & 4 & $1/4$ & 24.99--25.11 & 99.95--99.97 & 24.98--25.10 & -- & -- & 74k \\
$G_5$ & $\infty$ & 5 & $1/5$ & 21.15--21.81$^{\dagger}$ & 99.92--99.93 & 21.13--21.79$^{\dagger}$ & -- & -- & 74k \\
\bottomrule
\end{tabular}

\caption{Budget-specific cohorts underlying Table~\ref{tab:law}.
Measured accuracies are in percent; 74k and 148k denote 74,219 and
148,438 updates. $^{\dagger}$The common $G_5$ window includes an early transition
(App.~\ref{app:heisenberg-law}).}
\label{tab:law-expanded}
\end{table}

\paragraph{Time-resolved readouts.} In 35 runs carrying both accuracy
readouts, 3,218 logged evaluations have a largest standardized deviation from
$1/f$ of 3.6 standard errors on the original evaluated windows. Evaluations
are 1,000 updates apart, so this does not exclude intermediate excursions.
The original standard errors treat token observations as Bernoulli samples;
positions within a sequence and successive checkpoints are correlated. This
standardized summary is therefore not a test with independent observations.
A separate distribution study covers 37 checkpoints in eight groups and
two recipes. Its original evaluators include selection on correct quotient
predictions before averaging within-fiber diagnostics. The historical bounds $|\Delta\mathrm{LL}|\le.00722$ and $D\le.0130$ bits
retain those protocols and should not be read as one unconditional test.
Figure~\ref{fig:quotient-output} and Table~\ref{tab:other-groups} instead use
all examples. Within-fiber renormalization and selection on quotient
correctness are distinct operations. Near-uniform residuals under the original
selection rules also occur when overall class accuracy is low,
including $Q_8$ at .59 and .83 and $D_6$ at .82. These cases motivate reporting
class accuracy separately from conditional exact accuracy.

\paragraph{Linear decoding within and between classes.}
We separately train linear probes to predict the quotient class and the fixed
index of a state within its class (its \emph{within-class rank}). These are
readouts of hidden representations, distinct from the model's output accuracy.
In matched $A_4$, $\mathrm{SL}(2,3)$, and $S_4$ models, mean layer-4 rank
accuracies are 24.85\%, 12.60\%, and 8.24\%, close to chance levels of
25\%, 12.5\%, and 8.33\%; mean quotient-class decoding exceeds 99.9\%.
Figure~\ref{fig:linear-probes} retains all four blocks. The protocol and
shuffled-label controls are specified in App.~\ref{app:figure-methods}.
These null results concern this linear readout, not the absence of all
within-class information.

\begin{figure}[ht]
\centering
\includegraphics[width=\linewidth,trim=0 10bp 0 3bp,clip]{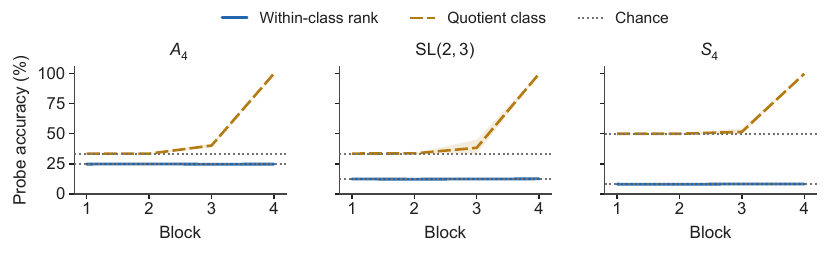}
\caption{\textbf{Linear decoding of quotient classes and within-class ranks.}
Curves and bands are means and ranges across three training seeds per group.
Dotted lines give each task's chance accuracy. Probes use hidden states at
positions 17--100, with training and test sequences disjoint.}
\label{fig:linear-probes}
\end{figure}

\paragraph{Additional probes.} A $D_4$ binary within-fiber probe gives
.4980--.5024 across layers in checkpoints where the quotient class is decodable
at layer 4; shuffled-label controls pass. Nevertheless, a learned character
can be linearly unreadable (.511) when represented as a product of two
linearly readable characters. This is a concrete reason to restrict conclusions
from null probes. In $\mathrm{SL}(2,3)$, three quotient-class means are
114--124 degrees apart at layer 4, and within-class spread is .0024 of
between-class distance. These geometric measurements support class separation
without proving absence of finer nonlinear information.

\paragraph{Finite resolution.} The flatness bounds are empirical tolerances.
For $\mathrm{Dic}_3$, a within-fiber improvement of approximately .0052 bits
has a confidence interval excluding zero. Such probabilistic information is
not excluded by the theorem about exact recovery on every input. The
accuracy correction in Proposition~\ref{prop:ceiling} is a different quantity
from this log-likelihood improvement and cannot be compared directly in bits.
Transition positions and confidently
incorrect outputs are not classified as uniform quotient states.

\section{Intermediate quotient stages and order sensitivity}

\subsection{Reordering under recipe A}
\label{app:recipe-a-reordering}
We test whether the full-group models' predictions use input order under recipe A.
For each of 2,048 sequences per run, we randomly reorder the prefix at each
position and compare the original and reordered output distributions at that
position using Jensen--Shannon divergence in bits. The elements and their
counts are preserved, and reordering preserves the input distribution under
independent uniform full-group sampling. Across 28 runs on $D_4$, $Q_8$,
$A_4$, $D_6$, $S_4$, $\mathrm{SL}(2,3)$, and $C_7\rtimes C_3$, mean divergence
beyond each run's exact frontier ranges from .0012 to .0163 bits. All runs
pass the identity-permutation and within-frontier sensitivity controls.
Position-wise information and reordering readouts reach their low-information
regions within two positions of each other in 23 of 26 runs.

\begin{figure}[htbp]
\centering
\includegraphics[width=\linewidth]{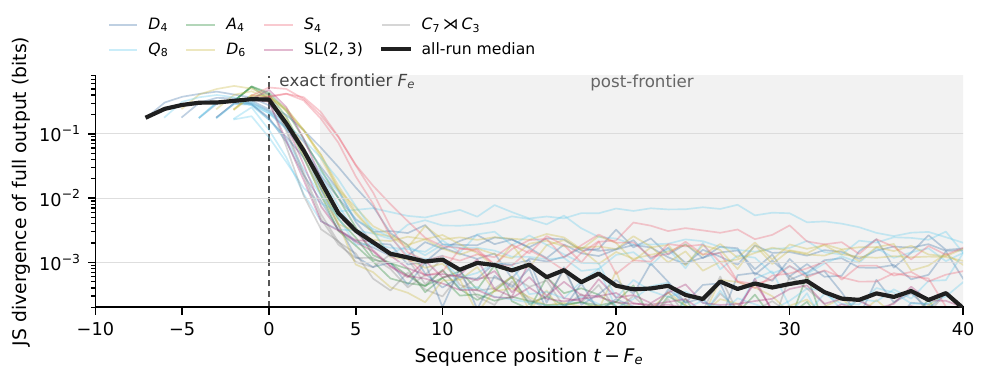}
\caption{\textbf{Full-output sensitivity to prefix reordering under recipe A.}
Each thin curve is one of 28 Transformer checkpoints; colors identify the
seven underlying groups and the black curve is their pointwise median. We align
each curve at its exact frontier $F_e$, the last position in the contiguous
prefix with exact accuracy at least .75. At each position, we randomly reorder
each input prefix without changing its multiset, then compute the Jensen--Shannon
divergence in bits between the original and reordered \emph{full} output
distributions. Sensitivity drops sharply after $F_e$. This behavioral test does
not establish strict order-blindness.}
\label{fig:recipe-a-reordering}
\end{figure}

\subsection{Intermediate quotient stages under recipe B}
\label{app:quotient-stages}

The 21-run study uses recipe B with three seeds on each of seven groups.
Supports are reflection cosets for $D_9$, $D_{15}$, and $D_{27}$; odd
permutations for $S_4$; two commutator cosets for $D_8$ and $Q_{16}$; and
nonmultiples of three for $C_9$. Both the support and recipe differ from
the baseline table. Intermediate behavior occurs in one $S_4$, two $D_{15}$,
and two $Q_{16}$ seeds; no such stage is observed in $D_{27}$. The $D_9$ gaps
close rapidly, while two $D_8$ seeds remain at the abelianization. Of all
21 runs, the frozen endpoint criteria classify five as plateaus, fifteen as
partial, and one as escape. These endpoint labels differ from the five runs
with intermediate quotient behavior; they do not imply that every partial
run is still improving. In this study, escape means exact accuracy at least
.90 over positions 9--100. A plateau requires exact accuracy at most .60,
joint quotient accuracy at least .80, and conditional exact accuracy within
.05 of the baseline $1/f$. The remaining runs in this study are labeled
partial. These are endpoint categories, not guarantees of subsequent training
behavior.

The $S_4$ example is exact through position 6 and has conditional accuracy
approximately .25 inside its $S_3$ classes from positions 7--100. Its $S_3$
class accuracy declines from 1.00 through position 20 to .72 at position 100,
while parity remains accurate throughout. A $D_{15}$ stage gives .201
against the $1/5$ baseline. Runs still progressing can give .8623 on this
same quotient or .6174 on a $Q_{16}$ quotient with baseline $1/2$.

For the main-text $S_4$ example, a fixed-window FP32 replay of the same saved
checkpoint uses 8,192 sequences and positions 17--100. Quotient accuracy is
89.3312\%, exact accuracy is 22.3406\%, and their ratio is 25.0087\%.
Restricting every prediction to its true $S_3$ class before averaging gives
KL .0027219 bits and $\Delta\mathrm{LL}=-.0026327$ bits. This diagnostic
includes all examples, including incorrect quotient predictions. The original
9--100 evaluation gives conditional exact accuracy .2502; the .2500 residual
readout below instead excludes the exact frontier and six further positions.

The residual analysis examines 13 runs and 31 candidate quotient levels.
Three learned intermediate stages give conditional accuracies .2500, .2010,
and .4995, with $|\Delta\mathrm{LL}|\le .00289$ bits. Their coarser classes
retain .43--1.36 bits. Only sufficiently learned class channels are included;
conditioning on correct guesses from an unlearned channel can spuriously
produce information estimates as high as 1.37 bits. The .70 class threshold
was chosen during analysis, and transition-region filtering changes the
number of apparently informative levels. The study therefore does not
establish an exhaustive or preregistered classification of learned partitions.

\paragraph{Reordering under recipe B.}
On the restricted input alphabets used above, the three runs with acquired
non-abelian quotients give .46, .54, and .24 bits on those channels, versus
at most .02 bits on the abelianization. One still-changing run remains
order-sensitive even when its projected channels appear insensitive,
illustrating why both diagnostics and the stage qualification are necessary.
The two recipes differ in architecture, input alphabet, and training budget;
their reordering results do not constitute a controlled comparison.

\section{Additional training and architecture comparisons}
\label{app:training-details}

\paragraph{Training duration.} In the 600k-update $D_4$ study, exact frontiers
end at 20, 22, and 39. The fastest run undergoes three temporary collapses at
325k, 365k, and 415k updates. Fits to its 120 evaluation points do not separate
linear growth from a logarithmic trajectory (RMSE 6.302 versus 6.306).
Extrapolated completion budgets of $10^{6.2}$ and $10^{8.1}$ updates depend on
these indistinguishable fits and are not predictions of asymptotic behavior.
The separate five-million-update run reaches frontier 37 at 3.005M updates
and ends at 21. Mean exact accuracy over positions 17--100 peaks at 62.51\%
and ends at 51.91\%. Quotient accuracy falls to 67.62\% before recovering to
98.71\%; long training can lose either type of information.

\paragraph{Additional five-million-update endpoints.}
\label{app:long-finals}
Four further full-group runs complete five million updates with the four-layer,
width-256 NeoX backbone, batch size 256, and constant learning rate
$5\times10^{-5}$: $D_4$ seed 43, $A_4$ seeds 43 and 44, and $S_4$ seed 43.
This is 1.28 billion training sequences per run, distinct from recipe B's
five-million-sequence budget. These runs are separate from the historical
$D_4$ trajectory in Fig.~\ref{fig:training-context}a. Table~\ref{tab:long-finals}
uses final position-wise accuracy arrays and common fixed windows.
They do not include a per-example output-distribution audit. On $S_4$,
the finer $S_3$ image accuracy is only 33.14\% over positions 17--100,
so the endpoint does not establish that this non-abelian quotient was learned.

\begin{table}[ht]
\centering\small
\begin{tabular}{lrrrr}
\toprule
Group & Exact 17--100 & Quotient 17--100 & Exact 70--100 & Exact 101--128\\
\midrule
$D_4$ & 53.33 & 99.55 & 49.65 & 15.04 \\
$A_4$ (43) & 35.47 & 99.96 & 25.13 & 19.87 \\
$A_4$ (44) & 33.13 & 99.98 & 25.01 & 18.09 \\
$S_4$ & 8.30 & 99.69 & 8.31 & 4.43 \\
\bottomrule
\end{tabular}
\caption{Additional five-million-update final accuracies (\%). Quotient
readouts use the abelianization. Parentheses distinguish the two $A_4$ seeds.}
\label{tab:long-finals}
\end{table}

\paragraph{Horizon and spectrum.} At horizon 24, exact frontiers cover
42--54\% of the horizon; horizons 48 and 100 reach similar absolute positions.
A two-seed continuation at horizon 24 adds 3 and 2 positions, versus a loss of
one position in each horizon-100 continuation, but does not meet the registered
confirmation criterion.
At horizon 24 on $\mathrm{SL}(2,3)$, the preregistered windows contain only
10--11 positions and yield conditional accuracies .1442/.1697/.1585
(seeds 42/43/44); the latter two retain their \texttt{PARTIAL} labels,
although exact accuracy at position 24 is .126/.124/.127 and subsequent
distribution diagnostics identify transition shoulders extending 9--11
positions beyond the exact frontier.\footnote{The horizon interventions use
bf16 on an RTX 5060, while their horizon-100 references use bf16 on RTX 4090s.} Curriculum training provides no clear benefit.
Boosting a two-dimensional Fourier component increases the mean frontier from
4.3 to 9.0--9.5 while reducing quotient accuracy to .37--.41. Comparisons use
the frozen-encoding control, whose quotient accuracy is .843, rather than the
learned-embedding baseline. Direction and dose dependence are supported;
the registered effect-size criterion is ambiguous. The lower-amplitude arm
learns characters approximately 10k updates earlier. The input-encoding intervention remains confined to $D_4$; the horizon
intervention also covers $A_4$ and $\mathrm{SL}(2,3)$.

\paragraph{Backbone, width, and depth.} GPT-2 and Llama reach $D_4$ frontiers
of 15--24 and 18--20 positions, respectively, while preserving post-frontier
conditional accuracy near $1/2$. Removing positions
from the NeoX baseline gives frontier 5 in one underpowered run; adding
sinusoidal positions slows quotient learning. Increasing width to 512 does
not remove the plateau. Under full-group inputs, the 2--6-layer $D_4$ grid comprises 18 runs:
12 parameter-matched runs with counts within 1.1\% of 3,163,648, and six
shared-width runs at width 256; their stored conditional accuracies span
.4991--.5017 across the full grid. Additional depth mainly
improves class accuracy and extends frontiers to 4--8 positions. Budget and depth interventions on $A_4$ and $\mathrm{SL}(2,3)$ comprise
18 runs; doubling the budget and adding two layers give median exact-frontier
gains of 3 and 5 positions on $A_4$, and 2 and 1 on $\mathrm{SL}(2,3)$,
respectively. All 18 stored conditional ratios are within .0012 of the
corresponding class-size baseline on the original evaluation windows:
positions 5--100, 17--100, and 65--100 for two, four, and six layers,
respectively. On a common 17--100 window, the largest excesses are .0275
for an $A_4$ six-layer run and .0144 for an $\mathrm{SL}(2,3)$ six-layer
run, consistent with wider positional transition regions. The six two-layer
runs have incomplete quotient learning and retain their
\texttt{UNDERPOWERED} labels. The original generator-input depth comparison is not parameter-matched; its
largest models remain partially accurate (.69--.79), despite losing the
original plateau. The subsequent parameter-matched generator-input comparison
gives conditional accuracies .5092, .6304, and .6470 for three-layer models
and .9137, .9746, and .9404 for five-layer models (seeds 42/43/44), with the
stored ratios evaluated over positions 9--100 and 33--100, respectively,
rather than a common window.

\paragraph{Depth and input distribution.} A full-group $D_4$ sweep over
2--6 layers contains 18 runs, including 12 parameter-matched runs and six
shared-width runs; their stored conditional accuracies span .4991--.5017. Additional depth improves quotient learning and extends the exact
prefix modestly. Under a generator-input recipe, a separate 3--5-layer sweep
extends the frontier from 3--6 to 51--58 positions, with median conditional
accuracy rising from .530 to .951. The original sweep also increases parameter count, so it does not isolate
depth. In a subsequent $D_4$ generator-input comparison at approximately 6.32 million
parameters, five-layer models reach token accuracies .776--.884 over positions
1--100, compared with .353--.396 for three-layer models across three seeds
each, although all six runs remain classified as \texttt{PARTIAL}. The support and optimization dependence is examined
in \S\ref{sec:crossing}. Reducing the training horizon from 100 to 24 on $A_4$ moves the median exact
frontier from 4 to 8 positions (conditional accuracy .2481--.2500), and
reducing it to 48 on $\mathrm{SL}(2,3)$ moves the median from 3 to 6
(.1284--.1319), while the 24-position $\mathrm{SL}(2,3)$ condition retains
two \texttt{PARTIAL} runs with conditional accuracies .1697 and .1585 in
the preregistered $[F_e(.75)+6,H]$ windows, where $H$ is the training horizon.

\paragraph{Larger models.} A paired Pythia-160M-architecture study uses eight
new seeds with shared block prefixes at 3, 4, 7, and 12 layers. At the 10k-update
classification point, median exact cutoffs are 6.5, 9, 15, and 39, while
quotient-first outcomes occur in 7, 5, 7, and 7 of eight seeds. Parameter count
varies from 98.5M to 162.3M. Three previously observed seeds are separate
replication checks. A width-256 $S_5$ study with 36 paired seeds at 4, 8, and
12 layers finds no associative-like outcome; parity-first counts are 33, 28,
and 31. These are budget-specific observations, not claims about eventual
convergence. Definitions and complete model specifications are in
App.~\ref{app:config}.

\paragraph{Matched recurrent models.} Two-layer LSTM and GRU models have
embedding/hidden widths 444 and 512 and match baseline parameter counts within
.2\% and .15\%. Each family has eighteen runs: three groups, two learning
rates, and three seeds. Fig.~\ref{fig:s4-training}b shows the
LSTM's dense 50-update replay through update 3,000 and the original logs
thereafter (App.~\ref{app:fig4-protocol}).
Fifteen runs meet the abelianization-stage criterion on 2--23 dense evaluations;
four of these runs have only two qualifying evaluations, 50 updates apart.
These are transient stages, not evidence of persistent plateaus. The three $A_4$ runs at $10^{-3}$
cross quotient and exact thresholds in the same interval. Stage accuracies
are .499--.509 on $D_4$, .253--.261 on $A_4$, and .083--.085 followed by
.247--.253 on $S_4$. The GRU curves in Fig.~\ref{fig:training-context} use the
original 1,000-update grid; the coset census below uses denser evaluations.
Recurrent controls use fp32 and shared GPUs, whereas the baseline uses
bf16 autocast. Three groups and one size do not support a general convergence
claim for gated recurrence.

\paragraph{Blind partition census.}
\label{app:blind-census}
For each stored checkpoint we evaluate a fixed set of 2,048 sequences
(seed 248,041) at positions 17--100 and form $P[q,p]$, the mean probability
placed on $p$ when the running product is $q$. We cluster the rows of $P$ by
connecting every pair at total variation below .15 and taking connected
components (single-linkage clustering). We take $H$ to be the block containing the
identity, and test whether $H$ is a subgroup, whether the blocks are its right
cosets $Hq$ or its left cosets $qH$, and whether $H$ is normal. A second diagnostic repeats these tests on the sets
$S(q)=\{p:P[q,p]\ge\frac12\max_{p'}P[q,p']\}$. No step refers to $[G,G]$ or to
any pre-chosen subgroup. We accept the support sets as a coset partition
only when they equal $Hq$ (or $qH$) for every $q$, with $H=S(e)$ a subgroup.
Thus the procedure can return a partition the quotient
readouts of \S\ref{sec:law} cannot express.

We run it on 41 Transformer checkpoints across nine groups and on recurrent
checkpoints from the matched arms. Row clustering returns 38 non-trivial
identity blocks, each equal to $[G,G]$, and three singleton partitions.
The singletons are two long-trained $D_4$ checkpoints whose exact frontier
has entered the evaluation window (exact accuracies .6022 and .6387), and
a nearly solved twelve-layer $S_5$ checkpoint (.9048).
The support procedure also returns 38 non-trivial commutator blocks and
three singleton partitions, but not at the same checkpoints. The two $D_4$
checkpoints retain commutator cosets under this procedure; two twelve-layer
$S_5$ checkpoints with residual errors instead split into singletons.
Neither procedure recovers a non-normal coset partition. These counts concern
this checkpoint census, not every Transformer resolution; the finer normal
quotients in \S\ref{sec:ladder} use a separate recipe and cohort.

The recurrent side returns one. Replaying the $S_4$ GRU runs with a census
every 100 updates ($5\times10^{-5}$: every 250), three runs hold a stage at
exact accuracy .48--.55 with the $S_3$ readout at .99 or above. Writing each
off-diagonal prediction as $p=hq$, the mass concentrates on the single element
$h=(0\,3)(1\,2)$ at .945--1.000; writing it as $p=qh$ instead spreads the same
mass over the three double transpositions at .321--.337 each. That spread is
the prediction rather than a null result: on a right-coset stage the left
reading returns $q^{-1}hq$, which sweeps the conjugacy class of $h$. The blind
census recovers $H=\{e,(0\,3)(1\,2)\}$ with right cosets and no normality on
the same checkpoints. One stage earlier, at exact accuracy .249--.251, the same
tally splits evenly over all three double transpositions at .331--.335, leaving
the whole $V_4$ fiber unresolved; that stage is the ordinary quotient
$S_4/V_4$, and one statistic therefore separates a normal stage from the coset
stage above it.

Three matched LSTM runs on $S_5$ at learning rate $10^{-3}$, replayed with
a census every 50 updates through update 8,000, take different refinement
paths. Two first show the quotient by $A_5$; the third has no observed
intermediate nontrivial quotient stage. We score
all 156 subgroups of $S_5$ at every evaluation rather than reporting a single
candidate: for each $H$ we take the total variation between the observed tally
and the uniform distribution on $H\setminus\{e\}$ that a right-coset stage on
$H$ would produce. A stage is an interval on which the best $H$ has total
variation below .15, exact accuracy within 15\% of $1/|H|$, and a duration of at
least 300 updates.

Seeds 42 and 44 first hold $A_5$, the only non-trivial proper normal subgroup of $S_5$
and the one the sign readout names, for 1,000 and 1,150 updates. Seed 42 then
holds an order-ten subgroup at .0986--.1090 for 2,300 updates and
$H=\{e,(0\,4)(1\,3)\}$ at .4904--.5010 for the remaining 3,250; seed 44 holds
the Klein group $\{e,(0\,1)(2\,4),(0\,2)(1\,4),(0\,4)(1\,2)\}$ at .2276--.2759
for 1,900. Seed 43 has no observed intermediate nontrivial quotient stage. $A_5$ is the only
non-trivial proper normal subgroup of $S_5$, and across all 161 evaluations of
that run the total variation to $A_5$ never falls below .500, while its sign
readout stays between .496 and .839 until the run solves. It instead walks the
point-stabiliser chain: $S_4$ at .0410--.0421 for 400 updates, the copy of
$S_3$ permuting $\{1,3,4\}$ at .1638--.1704 for 900, and $H=\{e,(3\,4)\}$ at
.4805--.5129 for 1,200. Of these eight stages the two $A_5$ stages are normal
and the other six are not, and seed 43 reaches exact tracking by update 3,850 without an observed
intermediate nontrivial quotient stage. At update 8,000, seeds 42 and 44
remain at exact accuracies .4997 and .4868, respectively. The latter is
already above its identified $1/4$ stage and is not classified as a new stage.

Within each of the eight identified stage intervals, exactly one of the 156
subgroups meets the fit threshold at every evaluation. The largest accepted
total variation is .112, while every alternative is at least .503 away.
This separation establishes uniqueness within the accepted intervals, not
throughout training. Of 483 scheduled evaluations, 386 have one fitting
subgroup and 97 have none; 22 of the latter precede the run's first exact
accuracy of .90. Evaluations with no errors do not define a normalized
off-diagonal distribution and are not treated as partial stages. Orders
3, 5, 8, 12, and 20 occur among $S_5$ subgroups but not among the eight stages.

The left reading again lands where a right-coset stage puts it, on the
conjugacy class of $h$: .0666 against $1/15$ for the double transposition and
.0977 against $1/10$ for the transposition, two different predictions on one
group. A blind partition of the stored seed-42 checkpoint at update 8,000 agrees,
returning 60 blocks of two with nearest-neighbour distances .0002--.0008
against the .15 threshold. This test clusters mean softmax rows. In contrast,
the dense replay criterion uses pooled full-state argmax errors and accuracy;
it does not establish per-example softmax uniformity or independently recover
a full row partition at every stage. The non-normal interpretation is supported
by the error structure and subgroup checks, not by reciprocal accuracy alone.

\paragraph{Display of the $S_5$ stages.}
Figure~\ref{fig:long-training-stages}c retains all 161 scheduled evaluations
of the selected LSTM seed, each using 2,048 sequences on positions 17--100.
The log accuracy axis separates the reciprocal levels. Colored segments
mark accepted intervals; intervening segments remain gray. The selection
rule, GRU and Transformer trajectories, and pooled error-statistic panel (d)
are specified in App.~\ref{app:nonsolvable-census}.

These censuses are descriptive and were not preregistered. The stages they
describe are transient, and the recurrent runs carry the fp32 and shared-GPU
disclosures above.

\paragraph{Causal audit of a non-normal coset stage.}
\label{app:causal-coset}
The censuses above read outputs. To ask whether the recurrent state itself
carries the coset, we preregistered a five-seed replication on $S_5$ with the
matched two-layer LSTM, 8,000 updates, a checkpoint every 50 updates, and four
nested claims: the full state is a causal $G$-set state; a class-mean
between-coset subspace is sufficient and necessary; canonical $H\backslash G$
coordinates are causally used; and the raw coordinates satisfy a literal affine
representation law. A selector reading only the output census admitted a run
when one non-normal order-two subgroup fit at seven consecutive evaluations
(300 updates) with total variation below .15, exact accuracy within .075 of
$1/2$, and a runner-up gap of at least .10. It took the first such window and
the seventh checkpoint in it. Confirmation required three of five runs.

One run qualified. Seed 23,204 held $H=\{e,(0\,4)(1\,2)\}$ from update 2,100,
and the audit ran on its update-2,400 checkpoint with fresh length-40 prefixes
and disjoint fit and test splits of 2,880 sequences each, with intervals
clustered by prefix history. Enumerating all 120 continuations of every test
prefix (345,600 resumed forward passes), right-coset top-1 accuracy has a
clustered 95\% lower bound of .998, and the functional stabilizer of the
identity state separates $H$ without being given it: the largest distance
inside $H$ is .008 and the smallest outside is .988. Replacing the full
$(h,c)$ state moves the prediction to the donor coset in every cross-coset
pair, retains the recipient coset in none, and reproduces the clean donor
output. A rank-59 between-coset subspace, fitted from fit-split class means
alone, gives target-coset accuracy .999 when kept and .040 when removed, while
a matched random subspace gives .073 and .999. A decoder of the 60 canonical
coordinates reads .979 held out and .981 on the true next state, and a
minimum-norm edit setting the recipient's coset scores to the donor's moves the
prediction to the donor coset in .992 of pairs, against .000 for the matched
random control.

The literal representation law fails at the same checkpoint. Input-conditioned
affine operators predict held-out one-step coordinates with $R^2$ .958, and
states reconstructed from them keep the target coset in .998 of cases, but the
operator stabilizer at the identity coset does not separate $H$ (largest
distance inside 26.0 against smallest outside 23.0) and $T_xT_y$ does not match
$T_{xy}$ ($R^2=-0.52$). We therefore report a causal coset-valued state with a
strong one-step affine approximation, not a recovered representation. An
earlier audit of a different $S_5$ LSTM checkpoint, not preregistered, gives
the same pattern including this failure.

Incidence, not the interventions, is what the package leaves open. Scored post
hoc by the descriptive criterion used above, the five runs hold thirteen
stages: four are $A_5$ and nine are not normal, at orders 2, 3, 6, and 12. Only
seed 23,204 held an order-two stage, which is what the selector was
preregistered to admit, so the package returned one qualifying run of five
(Wilson 95\% interval .036--.624) and no across-seed replication.

We then preregistered the same audit for two order-three stages in the same
cohort, conditional on output selection rather than on incidence: seeds 23,203
and 23,205 at updates 4,050 and 6,750, with $|H|=3$ and 40 right cosets.
Replacing the full state again moves the prediction to the donor coset in every
cross-coset pair, retains the recipient coset in none, and reproduces the clean
donor output, in both cells, and the between-coset subspace is again sufficient
and necessary in both. The registered thresholds are met in one of the two.
Seed 23,205 falls below the .98 bound on continuation coset accuracy at .976
and below the .95 bound on the held-out coset decoder at .901, so that cell is
scored a miss on two of the four claims and the conditional package returns no
replication either. The affine representation law fails at all three audited
checkpoints.

\paragraph{Bilinear recurrence.} A separate state-128 model covers ten groups
with three seeds each, plus ten-seed follow-ups. All thirteen tested stages
match their $1/|N|$ accuracy levels within the registered band. Among thirty
initial cells, ten pass through the abelianization, eleven through finer
quotients, and nine show no detected stage. The abelian $C_9$ also exhibits a
$C_3$ stage in two seeds. Unlike the transformers and matched LSTMs, this model
acquires a stage across positions within three evaluations in all sixteen
measured cases. It differs in encoding, optimizer, and readout: coset
probabilities are summed before argmax. Its stage accuracies therefore do not
establish the per-example output flatness measured for transformers.

\paragraph{Equal-score encoding control.}
\label{app:score-control}
On $S_3\times C_3$, we match the factors entering the analytic representation
score of \citet{marchetti2026sgc}. The complex one-dimensional $C_3$ character
(with its conjugate) and the real two-dimensional standard representation of
$S_3$ have equal normalized Fourier amplitudes and equal $C_\rho n_\rho=2$:
$2\times1$ and $1\times2$, respectively. Here $n_\rho$ is the representation
dimension and $C_\rho$ is the real/complex type factor. These matches tie the
candidate score at every prefix length. The full Peter--Weyl input encoding
is frozen, with all block amplitudes equal and a random orthogonal mixing
chosen by the run seed; each token embedding has norm .32.

\begin{table}[!htbp]
\centering
\caption{Matched-score encoding control on $S_3\times C_3$ at 74,219 updates.
Accuracies (\%) are averaged over positions 17--100. Each frontier is the
largest $F\leq100$ for which every position $1,\ldots,F$ reaches its stated
accuracy threshold.}
\label{tab:score-control}
\small
\begin{tabular}{@{}lrrrrr@{}}
\toprule
Seed & $C_3$ & Sign & $S_3\mid\mathrm{sign}$ & $F_{C_3}(.90)$ & $F_{S_3}(.75)$ \\
\midrule
42 & 99.88 & 99.89 & 33.37 & 100 & 3 \\
43 & 99.58 & 99.66 & 33.35 & 100 & 2 \\
44 & 98.60 & 99.80 & 33.36 & 100 & 2 \\
\bottomrule
\end{tabular}
\end{table}

Three seeds (42--44) use a four-layer, width-256 GPT-NeoX model with four heads,
rotary positions, and no dropout. Inputs are sampled independently and
uniformly from all 18 group elements, with cross-entropy supervision at each
prefix of length-100 sequences. AdamW uses learning rate $5\times10^{-5}$,
linear decay to zero over 74,219 updates, batch size 256, zero weight decay,
gradient clipping at 1, and BF16 autocast. Final evaluation uses 8,192 fresh
length-128 sequences with a fixed evaluation seed. Table~\ref{tab:score-control}
reports positions 17--100. Readouts sum full-state probabilities over the
other factors before taking argmax; $S_3\mid\mathrm{sign}$ is $S_3$ accuracy
restricted to examples with a correct sign readout. Since
$[S_3\times C_3,S_3\times C_3]=A_3\times\{0\}$, sign information alone leaves
three possible $S_3$ states and gives the $1/3$ reference level.

All three seeds retain the separation despite the matched score factors.
This limits a score-only account of output resolution in this Transformer
setup. It does not test the original two-layer theorem: the architecture,
loss, prefix supervision, and training rule differ, and the deliberate tie
violates its distinct-score assumption. The control does not uniquely identify
order dependence as the cause, nor does an output readout establish the absence
of internal matrix features.

\subsection{Matched state-space models}
\label{app:ssm}

We test whether a selective state-space model learns the same output
resolutions as our Transformer and recurrent controls. We train Mamba-2
\citep{dao2024mamba2} on $D_4$, $A_4$, and $S_4$, retaining all 18 two-layer
runs and all three four-layer $D_4$ runs in Table~\ref{tab:ssm-results}.
No run reaches .90 full-state accuracy on positions 17--100 at the final
evaluation. Several instead predict the abelianization accurately while
their conditional exact accuracy remains near the class-size baseline.

\paragraph{Models and parameter matching.}
The two-layer model has hidden width 496, expansion factor 2, head dimension
32, 31 heads, state size 80, one state group, and a convolution kernel of size 4.
It has 3,164,778, 3,168,746, and 3,180,650 trainable parameters on $D_4$,
$A_4$, and $S_4$, respectively. These are within 0.28\% of the corresponding
four-layer, width-256 NeoX controls. The four-layer $D_4$ model uses width
352, 22 heads, and state size 48, with 3,166,312 parameters. All other
configuration fields are shared. Thus the depth comparison holds total
parameters approximately fixed by changing width and state size as well.
We use \texttt{Mamba2ForCausalLM} from Transformers 5.13.1, with untied
input embeddings and output weights, no cache, and the pure-PyTorch chunked
state-space duality implementation (chunk size 50), without fused kernels.

\paragraph{Training and arm selection.}
We sample length-100 sequences online, independently and uniformly from the
full group, and supervise the product at every position with cross-entropy.
AdamW uses zero weight decay, gradient clipping at 1, no warmup, and linear
learning-rate decay to zero over 74,219 updates. The effective batch size is
256, accumulated over four batches of 64. Training and evaluation use FP32
without autocast, as in the LSTM and GRU controls; recipe-A Transformers use
BF16 autocast. The comparison therefore does not isolate architecture from
numerical precision.

For each group we run seeds 42--44 at initial learning rates
$5\times10^{-5}$ and $10^{-3}$. The preregistered rule assigns the higher-rate
arm to the group-level comparison when at least two of three lower-rate runs
have quotient accuracy below .99. All three groups meet this condition.
The registered four-layer $D_4$ branch is run because the two-layer study
does not reach exact tracking. At $5\times10^{-5}$, only one of its three
seeds falls below .99 quotient accuracy, so no higher-rate four-layer arm
is triggered. The table retains both two-layer arms and every seed,
including runs with incomplete quotient learning.

\paragraph{Readouts and results.}
During training we evaluate 1,024 length-128 sequences every 1,000 updates
using sampling seed $s+148999$ for training seed $s$. Final evaluation uses
8,192 fresh sequences with sampling seed 248041. All values below average
positions 17--100. We project each full-state argmax onto the abelianization
to measure quotient correctness. Since an exact prediction is necessarily
quotient-correct, conditional exact accuracy is the ratio of the two correct
counts on this common window. The finer $S_3$ readout for $S_4$ uses the
same projection rule.

At learning rate $10^{-3}$, two of three seeds per group exceed .99 quotient
accuracy, with conditional exact accuracy within .001 of $1/2$, $1/4$, and
$1/12$ on $D_4$, $A_4$, and $S_4$, respectively. The remaining seed in each
group does not meet the quotient threshold. The two quotient-resolved $S_4$
runs reach only 33.33\% and 33.26\% on the finer $S_3$ quotient.
At $5\times10^{-5}$ on $D_4$, the four-layer configuration improves quotient
accuracy from 41.70--62.54\% to 98.82--99.97\%, while conditional exact
accuracy stays near $1/2$. These results distinguish acquiring a quotient
from resolving states within its classes. They establish finite-budget
outcomes under the reported recipes; accuracy readouts alone do not establish
uniform per-example output probabilities or a complete sequence of transient
quotient stages in Mamba-2.

\begin{table}[!htbp]
\centering
\small
\setlength{\tabcolsep}{5pt}
\begin{tabular}{@{}llrrrrr@{}}
\toprule
Group & Layers & Initial LR & Seed & Exact & Quotient & Exact $\mid$ quotient\\
\midrule
$D_4$ & 2 & $5\times10^{-5}$ & 42 & 31.17 & 62.54 & 49.84\\
$D_4$ & 2 & $5\times10^{-5}$ & 43 & 30.29 & 60.69 & 49.90\\
$D_4$ & 2 & $5\times10^{-5}$ & 44 & 20.83 & 41.70 & 49.97\\
\addlinespace[3pt]
$D_4$ & 2 & $10^{-3}$ & 42 & 49.06 & 98.11 & 50.01\\
$D_4$ & 2 & $10^{-3}$ & 43 & 49.84 & 99.85 & 49.91\\
$D_4$ & 2 & $10^{-3}$ & 44 & 49.86 & 99.91 & 49.91\\
\addlinespace[3pt]
$A_4$ & 2 & $5\times10^{-5}$ & 42 & 10.00 & 40.10 & 24.94\\
$A_4$ & 2 & $5\times10^{-5}$ & 43 & 8.46 & 33.49 & 25.27\\
$A_4$ & 2 & $5\times10^{-5}$ & 44 & 24.77 & 99.38 & 24.92\\
\addlinespace[3pt]
$A_4$ & 2 & $10^{-3}$ & 42 & 15.93 & 63.92 & 24.92\\
$A_4$ & 2 & $10^{-3}$ & 43 & 25.03 & 99.89 & 25.05\\
$A_4$ & 2 & $10^{-3}$ & 44 & 25.02 & 99.86 & 25.05\\
\addlinespace[3pt]
$S_4$ & 2 & $5\times10^{-5}$ & 42 & 5.70 & 68.04 & 8.37\\
$S_4$ & 2 & $5\times10^{-5}$ & 43 & 7.63 & 91.55 & 8.33\\
$S_4$ & 2 & $5\times10^{-5}$ & 44 & 5.34 & 64.92 & 8.23\\
\addlinespace[3pt]
$S_4$ & 2 & $10^{-3}$ & 42 & 5.50 & 65.71 & 8.37\\
$S_4$ & 2 & $10^{-3}$ & 43 & 8.35 & 99.97 & 8.35\\
$S_4$ & 2 & $10^{-3}$ & 44 & 8.35 & 99.71 & 8.37\\
\addlinespace[3pt]
$D_4$ & 4 & $5\times10^{-5}$ & 42 & 49.40 & 98.82 & 50.00\\
$D_4$ & 4 & $5\times10^{-5}$ & 43 & 49.91 & 99.97 & 49.93\\
$D_4$ & 4 & $5\times10^{-5}$ & 44 & 49.79 & 99.57 & 50.00\\
\bottomrule
\end{tabular}

\caption{All matched Mamba-2 endpoints after 74,219 updates. Accuracies are
percentages on positions 17--100. ``Quotient'' is the abelianization readout;
``Exact $\mid$ quotient'' conditions on its correctness. Reference conditional
accuracies are 50\%, 25\%, and $100/12\%$ on $D_4$, $A_4$, and $S_4$.
Each row is one completed run, without selection by accuracy.}
\label{tab:ssm-results}
\end{table}

\section{Support interventions and failed predictions}
\label{app:support-details}

\paragraph{Task and counting solution.} We train models to predict each
running product, using inputs drawn independently and uniformly from a
specified set of group elements. We call this set the input alphabet.
In $D_4$, let $r$ be a quarter-turn and $s$ a reflection, with
$r^4=s^2=e$ and $srs=r^{-1}$. On the input set $\{r,s\}$, let $R_t$ and $S_t$
count rotations and reflections through position $t$, and let $E_t$ count
rotations at even positions. Then
\begin{equation}
q_t=r^{(R_t\bmod2)+2(E_t\bmod2)}s^{S_t\bmod2}.
\label{eq:generator-three-parities}
\end{equation}
Thus three parities recover the full state. We prove the identity below and
test both whether models use these counts and how changing the allowed inputs
affects learning. All support comparisons in this appendix specify their
training recipe; they are separate from the full-group baseline.

The recipe-B generator replication reaches .9992--.9999 at depth 3 and 1.0000
in a depth-4 arm. A single higher-learning-rate run is partial rather than
fully accurate. Cross-support transfer is poor in both directions, so success
on generators does not imply competence on arbitrary full-group inputs.

Removing positions gives chance performance in three seeds. Switching to full
group inputs produces two plateaus and one partial solution; replacing the
optimization block gives frontiers 12, 21, and 22 that do not improve with
3.8 times the budget. Changing MLP width preserves successful tracking but
changes its speed. Reverse comparisons using the baseline model on generators,
or under the alternative optimizer at $10^{-4}$, remain near $1/2$.
At $3\times10^{-4}$, three reverse-comparison seeds improve and then collapse
to uniform predictions. These are conditional ablations of complete recipes,
not universal necessity or sufficiency results.

For two $D_4$ commutator cosets, frontiers are 44, 81, and 100; three cosets
give frontiers 7--11. A lift-bias sweep is retained in the coverage table but
not interpreted as a single causal axis because it also changes exposure and
path statistics. Reflection-coset inputs permit full tracking in two of three
$D_3$ seeds, with the third reaching frontier 85; $D_5$ and $D_7$ have frontiers
30 and 17--42. The unequal seed counts are listed in App.~\ref{app:coverage}.

The matched two-coset $Q_8$ and $D_4$ comparison controls group order, fiber and
quotient sizes, support size, and budget. Both input supports admit the position-weighted counting form characterized by Lemma \ref{lem:support-counting}. All three $Q_8$ runs improve,
with frontiers 17, 21, and 48 and ratios .5483, .5733, and .7026. This rejects
the registered prediction that the decomposition is required for improvement.
A possible speed difference is based on three paired runs with overlapping
cross-seed ranges.

Exact Fourier-degree calculations provide a separate description of the input
supports. Under full-group inputs, both quotient and residual sectors have
pure degree $t$ in the checked groups. On $D_4$ generators, residual degrees
are $\lfloor t/2\rfloor$ and $\lceil t/2\rceil$; on $S_3$ and $D_6$ they span
multiple degrees centered near $t/2$. Balanced $Q_8$ inputs $\{\pm i,\pm j\}$
and two-coset $D_4$ inputs retain degree $t$. These identities distinguish
the supports but do not supply a learning-rate law.

\paragraph{A counting criterion for restricted supports.}
Let $G$ be nilpotent of class at most two, let $K=[G,G]\subseteq Z(G)$, and
let $\pi:G\to Q=G/K$ be the quotient map. Fix a section $s:Q\to G$,
so every element has a unique
representation $s(q)z$ with $z\in K$. We write $K$ additively.
For a nonempty input alphabet $\Sigma\subseteq G$, let $A=\pi(\Sigma)\subseteq Q$ be its
quotient image and define the commutator pairing
\[
\omega(q,q')=[s(q),s(q')]\in K,\qquad [u,v]=uvu^{-1}v^{-1}.
\]
Centrality of $K$ makes this pairing independent of the section.
For a word $w=x_1\cdots x_T$, let $z(w)$ denote the central coordinate of its
product and let $M(w)$ denote its input multiset.

\begin{lemma}[Position-weighted counting on a support]
\label{lem:support-counting}
The following conditions are equivalent:
\begin{enumerate}[leftmargin=*,nosep]
\item There are functions $F:\{\text{multisets over }\Sigma\}\to K$ and $c:A\to K$
such that, for every word over $\Sigma$,
\[
z(w)=F(M(w))+\sum_{i=1}^T i\,c(\pi(x_i)).
\]
\item The commutator pairing has a potential on $A$:
$\omega(q,q')=c(q')-c(q)$ for some $c:A\to K$.
\item For every $q,q',q''\in A$,
$\omega(q,q')+\omega(q',q'')+\omega(q'',q)=0$.
\end{enumerate}
\end{lemma}
\begin{proof}
For two letters $u,v$, centrality gives
$uv=[u,v]vu$ and hence $z(uv)-z(vu)=\omega(\pi(u),\pi(v))$.
The two words have the same multiset. Under condition 1, subtracting their
weighted terms gives
\[
\big[c(\pi(u))+2c(\pi(v))\big]-\big[c(\pi(v))+2c(\pi(u))\big]
=c(\pi(v))-c(\pi(u)),
\]
which proves $1\Rightarrow2$.

Assume condition 2. Swapping adjacent letters $u,v$ at positions $i,i+1$
changes the product coordinate by $\omega(\pi(u),\pi(v))$ when the difference
is taken as the original word minus the swapped word. The surrounding
prefix and suffix do not change this difference because the commutator is
central. The weighted sum changes by exactly
$c(\pi(v))-c(\pi(u))$. Thus
$z(w)-\sum_i i\,c(\pi(x_i))$ is invariant under adjacent swaps.
Any two words with the same multiset are connected by such swaps, so this
difference defines $F(M(w))$, proving $2\Rightarrow1$.

Condition 2 implies condition 3 by telescoping. Conversely, fix $q_0\in A$
and set $c(q)=\omega(q_0,q)$. Condition 3 applied to
$(q_0,q,q')$ yields $\omega(q,q')=c(q')-c(q)$.
\end{proof}

The proof does not require a special choice of section. Any additional terms
that depend only on the input multiset are absorbed into $F$, so we do not
need an explicit pairwise formula for the central coordinate. The lemma
characterizes the displayed position-weighted counting form, not every
computation that might use counts.

For $D_4$, $K=\langle r^2\rangle\cong C_2$ and $Q\cong C_2\times C_2$.
Writing quotient elements as $(a,b)$, the commutator pairing is
$\omega((a,b),(a',b'))=ab'+ba'$ in $C_2$. Any two quotient classes satisfy
the cycle condition. For three distinct classes $q,q',q''$, their cycle
sum is $\omega(q+q'',q'+q'')=1$, because the two arguments are distinct
nonzero vectors in $C_2\times C_2$. The counting form therefore exists exactly
when the allowed inputs belong to at most two commutator cosets.
Concretely, $r$ and $r^3$ belong to one coset, and $s$ and $r^2s$ to another.
Adding $r^3$ to $\{r,s\}$ therefore preserves the criterion, whereas adding
$e$, which belongs to a third coset, violates it. In the latter case, no
sum of token contributions weighted only by absolute position can recover
the central coordinate, even with an arbitrary multiset-dependent term.
This does not exclude richer computations using input order.

On $\{r,s\}$, let the $k$th $r$ occur at position $p_k$.
There are $p_k-k$ preceding reflections, so the rotation exponent is
\[
a_t=\sum_{k=1}^{R_t}(-1)^{p_k-k}
\equiv R_t-2\sum_{k=1}^{R_t}(p_k-k)
\equiv R_t^2-2E_t
\equiv (R_t\bmod2)+2(E_t\bmod2)\pmod4.
\]
The reflection exponent is $S_t\bmod2$, proving
Eq.~\ref{eq:generator-three-parities} for all lengths.
For $H_3(\mathbb Z/p)$ with odd $p$, the quotient images of
$\{a,a^{-1},b,b^{-1}\}$ violate the cycle condition:
the triple $(a,a^{-1},b)$ has cycle sum $-2\omega(a,b)\ne0$.
This excludes the displayed counting form on that support; the associated
training outcomes are empirical results rather than consequences about
learnability.

\begin{table}[htbp]
\centering
\small
\begin{tabular}{@{}p{.30\linewidth}cccp{.36\linewidth}@{}}
\toprule
Alphabet & Cosets & Form & Runs & Outcome (recipe B) \\
\midrule
$\{r,s\}$ (preregistered) & 2 & yes & 3 & conditional .9997--1.000 at 25--45k updates \\
$\{r,r^3,s,r^2s\}$ (two cosets) & 2 & yes & 3 & frontiers 44, 81, 100 \\
$\{r,r^3,s\}$ & 2 & yes & 2 & frontier 38 rising; 24 stalled (300k) \\
$\{r,s,e\}$ (preregistered) & 3 & no & 3 & frontiers 25, 26, 23; conditional .50 \\
$\{r,s,rs\}$ & 3 & no & 1 & frontier 20 \\
$\{r,s,r^2\}$ & 3 & no & 1 & frontier 14 \\
three nontrivial cosets & 3 & no & 3 & frontiers 7--11 \\
full alphabet & 4 & no & 3 & frontiers 12--22 \\
\bottomrule
\end{tabular}
\caption{Support criterion against $D_4$ outcomes under recipe B. Frontiers are
the number of leading positions with exact accuracy at least .75, read at 150k
updates unless noted; earlier rows come from the support interventions above.
The two preregistered rows use seeds 47--49 under a frozen protocol; an earlier
exploratory set (seeds 42--44) gave the same classification for both
alphabets, and the remaining new rows are exploratory. The criterion characterizes availability of the stated counting form;
it does not guarantee that training learns it.}
\label{tab:support-criterion}
\end{table}

\paragraph{Transfer in both directions.} Twelve generator-trained endpoints
evaluated on the full alphabet give exact accuracy .124--.125 and quotient
accuracy .249--.250, the chance levels, with peaked rather than uniform
outputs; these inputs contain tokens never seen in training. In the reverse
direction all tokens are trained: six recipe-B full-alphabet models evaluated
on $\{r,s\}$ sequences keep conditional accuracy .50--.55 with exact frontiers
7--11, and five of six lose the quotient itself (joint accuracy .29--.35
against a positional chance level of .50).

\paragraph{What the $\{r,s\}$ models compute.} At the input of the last block
of the successful $\{r,s\}$ models, ridge regressions read the running counts of $r$
at odd positions, of $r$ at even positions, of $r$, and of $s$ with $R^2$
.97--.99, with different directions at odd and at even positions; the parities
of these counts are at chance until the final layer. We then replace the
residual stream entering the last block. With a donor that flips one odd
and one even token, replacing all odd-position prefix states moves the output
toward the odd-flip label only (never toward the even-flip label), replacing a
growing fraction of those states moves it monotonically, and replacing states
before the flipped position does nothing. Replacing half the prefix leaves the
takeover incomplete (.21--.69) because the query position's own residual also
carries the base counts; replacing it as well gives takeover .76--1.00 at
positions 32 and 64 in eight models, with cross-class rates at most .03. A
prefix-product representation would instead make the outcome depend on which
flip comes later and would produce the double-flip label, neither of which
occurs. A preregistered held-out test on two new seeds meets all 18 frozen
criteria in one seed and 14 in the other (misses: .760 against .80 at position
64, and three margins of .004--.06 at position 100).
The results above distinguish the exploratory intervention study from its
held-out confirmation; their criteria and outcomes are reported separately.

\section{Heisenberg tracking, prefix counts, and interventions}
\label{app:heisenberg}
The integer Heisenberg task lets us ask how a model recovers state
information that depends on input order. Here we first derive the
computation required for the central coordinate. We then test whether
trained models distinguish words that have the same counts but different
products, whether the required prefix count is readable from their
activations, and how editing count-associated activations changes their
predictions. These tests address behavior, representation, and the effect
of an intervention. However, they do not identify a complete circuit.

\subsection{Task and training}
\paragraph{The counting identity.}
An element of $H_3(\mathbb Z)$ has coordinates $(x,y,z)$, with multiplication
\[
(x,y,z)(u,v,w)=(x+u,\;y+v,\;z+w+xv).
\]
We use the four tokens $a=(1,0,0)$, $a^{-1}=(-1,0,0)$,
$b=(0,1,0)$, and $b^{-1}=(0,-1,0)$. Write token $i$ as
$g_i=(u_i,v_i,0)$ and its running product as
$g_1\cdots g_i=(X_i,Y_i,Z_i)$, starting from $(0,0,0)$.
Thus $u_i$ is the signed increment to the $a$ count and $v_i$
is the signed increment to the $b$ count. Multiplication gives
\[
X_i=X_{i-1}+u_i,\qquad
Y_i=Y_{i-1}+v_i,\qquad
Z_i=Z_{i-1}+X_{i-1}v_i.
\]
The first two coordinates are therefore net counts. Summing the last
update gives
\[
Z_t=\sum_{i\le t}X_{i-1}v_i
   =\sum_{k<i\le t}u_kv_i.
\]
Each $b$ or $b^{-1}$ contributes according to the signed number of
$a$ tokens that preceded it. This derivation follows directly from group multiplication. To give an example how $Z$ differs while $(X,Y)$ is the same, $aab$, $aba$, and $baa$
all have $(X,Y)=(2,1)$, but their products have $Z=2,1,0$,
respectively: the $b$ encounters two, one, or no preceding $a$ tokens.
Final net counts alone do not determine $Z$.
Therefore, a model that predicts \(Z\) accurately for arbitrary words must take their order into account.

\paragraph{Architecture and data.}

We test whether models' predictions distinguish order-sensitive products. We use $H_3(\mathbb Z)$ with the multiplication in
Sec.~\ref{sec:heisenberg-extension}. Each word contains independent uniform
samples from $\{a,a^{-1},b,b^{-1}\}$. The backbone has four GPT-NeoX blocks,
width 256, four attention heads, and 3,160,576 parameters. Separate linear
heads classify $x,y\in[-128,128]$ and $z\in[-512,512]$, with an additional
out-of-window class for $z$, giving 3,554,816 parameters in total. Loss is the
sum of three cross-entropies at every position. A prediction is jointly
correct only when all three coordinates are correct; the abelianization
readout requires both $x$ and $y$ to be correct.

Three models use learning rate $10^{-3}$ and seeds 42--44, with AdamW,
zero weight decay, linear
learning-rate decay without warmup, batch size 256, length 100, 74,219 updates,
and bf16 autocast. The final evaluation uses 8,192 iid length-128 words.
Table~\ref{tab:heisenberg-final} reports the original final evaluations;
subsequent frozen-checkpoint replays can differ slightly with numerical
precision and batch shape. Paired tests below use their own controls and
must not replace these unconditional iid accuracies.

\begin{table}[ht]
\centering\small
\begin{tabular}{lrrrr}
\toprule
Seed & Joint 17--100 & Quotient 17--100 & Joint at 100 & Joint 101--128\\
\midrule
42 & 98.10 & 100.00 & 92.90 & 36.82 \\
43 & 98.17 & 100.00 & 93.74 & 37.32 \\
44 & 98.07 & 100.00 & 92.81 & 27.32 \\
\bottomrule
\end{tabular}

\caption{Integer Heisenberg final accuracy (\%). All models use learning rate $10^{-3}$. Joint accuracy requires all coordinates; quotient
accuracy requires $(x,y)$.}
\label{tab:heisenberg-final}
\end{table}

The underlying group is infinite, but the tested words and output ranges are
bounded. A sufficiently large finite quotient could agree on this window.
These experiments establish order-sensitive tracking at tested lengths, not
unbounded integer computation or a causal benefit of an infinite state space.
Uniform-within-class baselines from the finite-group experiments do not extend
to a uniform distribution on the infinite center. Comparisons with finite
Heisenberg tasks also change support and output parameterization.

\subsection{Reordering tests}

We test whether the models respect several consequences of the Heisenberg
group law. Each condition constructs a pair of input sequences through a
controlled edit. Several tests edit a block within a common prefix $p$ and
suffix $s$; the odd/even-count test instead reorders the whole sequence.
Some edits preserve the final product,
whereas others preserve $(X,Y)$ but change $Z$ by a known amount. We use
$A=a^{-1}$, $B=b^{-1}$, and
\[
c=[a,b]=abAB=(0,0,1).
\]
\paragraph{Adjacent inverse-pair reversal.}
We sample the same random prefix $p$ and suffix $s$ and compare
$pgg^{-1}s$ with $pg^{-1}gs$, where $g$ is one of the four input
tokens. Both inserted pairs multiply to the identity. The two input
sequences therefore have the same length, counts of each token, and final
product, although their intermediate states differ. This local test asks
whether reversing an immediate cancellation changes the model's prediction.

\paragraph{Adjacent noncommuting swap.}
We compare $pghs$ with $phgs$, where $g$ and $h$ act along different
coordinate axes. For example, $ab=(1,1,1)$ whereas $ba=(1,1,0)$.
Exchanging the two tokens preserves their counts and the final $(X,Y)$
coordinates, but changes $Z$ by $1$ or $-1$, depending on their signs.
This test asks whether the model responds to the simplest order-dependent
change in the product.

\paragraph{Odd/even-count-preserving reordering.}
We independently permute the tokens occupying the odd and even positions
of a randomly sampled input sequence, retaining pairs for which $Z$
changes. The original and reordered sequences contain the same number of
each token at odd positions and at even positions, and consequently have
the same final $(X,Y)$. Their different values of $Z$ show that neither
total token counts nor separate odd- and even-position counts determine
the product. This test distinguishes Heisenberg tracking from the
position-parity statistic sufficient for the $D_4$ task.

\paragraph{Long-range cancellation.}
Let $w=g_1\cdots g_8$. Within the same random prefix and suffix, we compare
the locally cancelling block
\[
g_1g_1^{-1}\cdots g_8g_8^{-1}
\]
with the nested block
\[
ww^{-1}=g_1\cdots g_8g_8^{-1}\cdots g_1^{-1}.
\]
Both blocks contain the same tokens and multiply to the identity. In the
first block, every generator is cancelled immediately; in the second, the
matching inverse can occur many positions later. This test asks whether
the longer-range arrangement makes an otherwise identical cancellation
harder to track.

\paragraph{Central-commutator relocation.}
Let $c=[a,b]=abAB=(0,0,1)$. We sample a prefix $p$, a middle segment $m$,
and a suffix $s$, and compare
\[
p\,c^4\,m\,s
\qquad\text{with}\qquad
p\,m\,c^4\,s.
\]
Because $c$ is central in $H_3(\mathbb Z)$, it commutes with the product
of $m$, and the two sequences have the same final product. They introduce
the same central contribution at different positions, however, and hence
follow different intermediate states. This test asks whether relocating
that contribution affects the prediction.

\paragraph{Commutator inversion.}
Within the same random prefix and suffix, we replace $c^4$ by $c^{-4}$.
The two blocks contain equal numbers of $a$, $A$, $b$, and $B$, and both
leave $X$ and $Y$ unchanged. Their contributions to $Z$ are $+4$ and
$-4$, so the replacement changes the final central coordinate by $8$.
This test asks whether the model predicts the direction and magnitude of
that change.

\paragraph{Rectangular commutator rewriting.}
We compare the block
\[
c^4=(abAB)^4
\]
with
\[
a^2b^2A^2B^2(aA)^2(bB)^2.
\]
The rectangular loop $a^2b^2A^2B^2$ and $c^4$ both have product
$(0,0,4)$. The appended inverse pairs multiply to the identity and make
the two blocks equal in length and in their counts of each token. Thus
the complete input sequences have the same final product but express the
central contribution through different local arrangements. This test asks
whether the model recognizes the same group element beyond the repeated
literal pattern $abAB$.

\paragraph{Evaluation protocol.}
We evaluate the three models at lengths 32, 64, and 100, using all seven
transformations above. Matched random controls have the same length, token counts, and final state as each test
word. All prefixes satisfy $|x|,|y|\le32$ and $|z|\le256$; final coordinates
satisfy $|x|,|y|\le16$ and $|z|\le64$. This conditional sample is easier than
the unconditional iid evaluation.

Each condition and length contains 1,024 shared pairs, giving 21,504 distinct
pairs across the battery, reused across models. Before model evaluation,
36 attempted pairs were excluded after failing to obtain a matched control
within 4,096 attempts. Integer $3\times3$ matrix multiplication independently
checks every prefix. We evaluate predictions at the final position of each word.

\paragraph{Results.}
On the odd/even-count-preserving test at position 100,
the models correctly predict both members in 96.875\%
of pairs on average. Their predicted displacement equals the true displacement
in 96.940\%; this latter criterion can hold even if both endpoints are wrong.
At length 100, long-range cancellation, central-commutator relocation, and
rectangular rewriting do not meet the prespecified criterion for degradation
relative to matched controls. That criterion requires the edited-minus-control
accuracy difference to have a 95\% confidence interval entirely below
$-10$ percentage points in at least two of three models.

\paragraph{Sensitivity to central displacement.}
We separately test whether predicted differences in $Z$ track changes of
different magnitudes and remain detectable after a common suffix. This test uses 24,576 fresh pairs, with generator
block exchanges inducing signed differences of magnitude 1, 4, or 16, plus
zero-displacement controls. Frozen coordinate readouts are evaluated after
0, 8, and 32 common suffix tokens. A threshold is set from the 95th percentile
of absolute readout differences on an independent zero-displacement calibration
set. Detection requires exceeding this threshold with the correct sign.
After 32 suffix tokens, response slopes across signed displacements are
.905--.952 across the three models. Detection of magnitude 16 reaches
100\%, with false-positive rates 4.44--4.83\%. Smaller displacements
are not certified in every model. These readouts measure central information,
not causal necessity or a unique update circuit.

Together, the reordering results show that the models distinguish products
that total token counts, including separate odd- and even-position counts,
cannot distinguish. We next ask whether their activations contain the prefix
counts appearing in the update for $Z$.

\subsection{Single-position raw-count probes}
\label{app:heisenberg-raw-counts}
We first ask whether an activation at a single position contains counts
accumulated over the preceding inputs. We fit separate linear readouts for
the numbers of $a$ and $a^{-1}$, and for their difference $X_{t-1}$.
This exploratory follow-up uses one frozen model (learning rate $10^{-3}$,
seed 42), after earlier net-count results were available.

We generate two disjoint sets of length-100 words. We use 2,048 words to fit
the readout and 2,048 new words to test it. For each word, we record the
model's 256-dimensional activation vector at one chosen layer and position.
A \emph{linear probe} is a fitted weighted sum of these 256 activation values,
plus a constant, that estimates a known count. The count is calculated from
the input word, not from the model's prediction. We fit a separate probe for
each quantity, layer, and position, using positions 32, 64, and 100.

Before fitting, we subtract each activation coordinate's fit-set mean and
divide by its fit-set standard deviation, floored at $10^{-8}$. We choose
the readout weights to minimize mean squared prediction error plus .001 times
the sum of squared weights. This penalty is called \emph{ridge regularization};
the additive constant is not penalized. We apply the same transformation and
fitted weights to the test words. Only the probes are fitted; the task model
remains unchanged. Probe fitting uses 64-bit arithmetic, while model
evaluations retain bf16 precision.

Besides the two individual counts and their difference $X_{t-1}$, we test
readouts of the current contribution $X_{t-1}v_t$ and accumulated coordinate
$Z_t$. Probe fitting includes all four possible current tokens.
Table~\ref{tab:heisenberg-counts} reports count predictions from the first
Transformer block's output, evaluated on test words whose current token is
$b$. ``Single-position'' means that the probe sees only one activation vector;
the count it predicts covers the entire preceding prefix. We verify the
target values independently by multiplying the inputs as integer
$3\times3$ Heisenberg matrices.

We measure the mean absolute error (MAE) in count units and the fraction of
predictions that equal the true count after rounding to the nearest integer.
We also report $R^2=1-\sum_r(y_r-\hat y_r)^2/\sum_r(y_r-\bar y)^2$, where
$y_r$ and $\hat y_r$ are the true and predicted values for test word $r$, and
$\bar y$ is the mean true value in that test set. Thus $R^2=1$ is perfect
prediction; $R^2=0$ matches the squared error of always predicting $\bar y$.
Negative values are worse than that constant prediction.

\begin{table}[ht]
\centering\small
\begin{tabular}{crrrrrr}
\toprule
 & \multicolumn{2}{c}{$a$ count} & \multicolumn{2}{c}{$a^{-1}$ count} & \multicolumn{2}{c}{Net $a$ count}\\
Position & MAE & Exact (\%) & MAE & Exact (\%) & MAE & Exact (\%)\\
\midrule
32 & 0.0381 & 100.00 & 0.0366 & 100.00 & 0.0040 & 100.00 \\
64 & 0.0887 & 100.00 & 0.0847 & 100.00 & 0.0069 & 100.00 \\
100 & 0.1513 & 99.22 & 0.1465 & 99.42 & 0.0088 & 100.00 \\
\bottomrule
\end{tabular}

\caption{First-block readout of preceding counts at a $b$ token. MAE is in
count units; exact denotes accuracy after rounding to the nearest integer.
Test subsets contain 532, 499, and 514 examples at positions 32, 64, and 100.}
\label{tab:heisenberg-counts}
\end{table}

At the tested positions, first-block activations support nearly exact
readout of both individual counts and their difference
(Table~\ref{tab:heisenberg-counts}). After rounding, the individual-count
readouts are 99.22--100\% accurate and the net-count readout is 100\% accurate
on the reported current-$b$ subsets. To check that these readouts depend on
learned activations, we fit two control readouts. The first uses only the
current token's embedding, before any Transformer block processes it. The
second randomly permutes the association between fitting activations and
their target counts, while leaving the test set correctly paired.
At position 64, the $a$-count embedding control gives $R^2=-.00046$ and
9.62\% rounded accuracy; the readout fitted with shuffled counts gives $R^2=-.01370$ and
10.22\%. A further linear readout uses only the true prefix coordinates $(X,Y)$
and four indicators identifying the current token, rather than activations.
It gives $R^2=.66939$ and 19.64\%. Thus the hidden state supports a more
accurate linear raw-count readout than this coordinate baseline. The comparison
does not exclude nonlinear coordinate encodings or prove that the output uses
the recovered counts.

Because these probes are fitted separately at each position, they do not
establish that the same linear readout works across positions.

\subsection{A count readout shared across positions}
\label{app:heisenberg-shared-counts}
We next ask whether the net count can be read out using the same linear map
at different positions. For each model and chosen point in the network, such as the output
of its second block, we fit one probe using examples from nine positions.
Each example consists of one word's activation at one position and the true
count at that position. All nine positions share the same readout weights
and additive constant.
This exploratory follow-up uses the three frozen models and independent sets of
4,096 fit and 1,024 test words of length 100, with i.i.d.\ uniform generator
inputs. The sets are disjoint from each other and from the earlier paired data
checked by the analysis script. We collect activations at positions
$8,12,16,24,32,48,64,80,100$. Positions 8, 12, and 16 are diagnostic locations
outside the standard 17--100 evaluation window.

The 4,096 fitting words therefore supply $4096\times9$ activation--count
pairs. We use the standardization and ridge fit described above, with penalty
.001. We test the resulting probe at each position on the 1,024 unseen words,
without changing its coefficients or additive constant. The target is
$X_t=\sum_{i\le t}u_i$, including the current input. The preceding count
required by the central-coordinate update is $X_{t-1}=X_t-u_t$; the known
current token supplies $u_t$. The interventions below target this preceding
count.

Table~\ref{tab:heisenberg-shared-counts} reports second-block outputs after each
activation vector $h$ is divided by its Euclidean length
$\|h\|_2=(\sum_k h_k^2)^{1/2}$, floored at $10^{-8}$. This removes differences
in the overall activation magnitude before fitting the readout. Across models and sampled positions
16--100, $R^2$ ranges from .99351 to .99880. As an additional check, we allow a separate rescaling and additive offset
of this probe's prediction at each position. The fitted multiplicative
factors are .973--1.016 over positions 16--100. The table uses the original
shared probe, without these adjustments.

\begin{table}[ht]
\centering\small
\setlength{\tabcolsep}{3pt}
\begin{tabular}{lrrrrrrrrr}
\toprule
Seed & 8 & 12 & 16 & 24 & 32 & 48 & 64 & 80 & 100 \\
\midrule
42 & 0.9786 & 0.9919 & 0.9942 & 0.9962 & 0.9977 & 0.9983 & 0.9984 & 0.9988 & 0.9986 \\
43 & 0.9767 & 0.9926 & 0.9948 & 0.9977 & 0.9979 & 0.9982 & 0.9984 & 0.9988 & 0.9985 \\
44 & 0.9775 & 0.9911 & 0.9935 & 0.9967 & 0.9973 & 0.9974 & 0.9969 & 0.9980 & 0.9972 \\
\bottomrule
\end{tabular}

\caption{Shared count readout: test $R^2$ at each position for a single probe
per model on normalized second-block outputs. All models use learning rate
$10^{-3}$. No position-specific refitting is used.}
\label{tab:heisenberg-shared-counts}
\end{table}

We also test activations without length normalization, token embeddings,
the separate attention and feedforward outputs within Transformer blocks,
and probes fitted separately at each position. A second shared probe targets
$X_t/t$, the average signed $a$ increment per input token. Across all tested
models, network locations, positions, and normalization choices, the largest
$R^2$ from a probe fitted to randomly reassigned target values is .12844. Both $X_t$ and $X_t/t$ have strong shared readouts at
the second block. Count readability therefore does not identify a unique code
or show that one layer converts an average count into a total count. This test concerns
readability on the sampled words.

For the separate position-specific probes, we fit linear probes, fitted with the same ridge penalty, for
$X_t$ and $Y_t$ from unnormalized second-block outputs at positions
$32,41,61,64,81,100$. Each model uses 4,096 fit words and 1,024 held-out words.
Across the eighteen model--position combinations, $X_t$ has
$R^2=.999777$--$.999940$ and rounded accuracy 99.51--100\%; $Y_t$ has
$R^2=.999704$--$.999962$ and rounded accuracy 99.61--100\%.

These results establish linear readability. We next test whether perturbing
count-associated activations changes the model's prediction of $Z$.

\subsection{Count-direction interventions}
\label{app:heisenberg-count-interventions}
We estimate how first-layer attention activations vary with the preceding net
count, while accounting for other measured input properties. We then add a
multiple of this fitted count direction to the activation and measure the
change in the predicted central coordinate, keeping the input word fixed.
Fig.~\ref{fig:heisenberg-count-interventions} tests the signs of these changes
at one position and their approximate additivity across two positions.
We use the same three models trained at $10^{-3}$, seeds 42--44, as the
behavioral results and shared decoding probes. No weights are retrained.

\paragraph{Fitting the encoding direction.}
\label{app:heisenberg-encoding-fit}
We fit an \emph{encoding} map from known input properties to activations,
then use its count coefficient to perturb the model. This reverses the
direction of the decoding probes above, which map activations to counts;
their weights are not used for these interventions. At position $j$, each
of the first layer's four attention heads produces 64 numbers. We join these
four lists into one vector $h_j\in\mathbb R^{256}$, before the layer's output
projection mixes them. For each input word, these 256 recorded values are
what the encoding regression is fitted to predict. We record them during
ordinary model evaluation with bf16 precision and fit the regression using
64-bit arithmetic.

For panel (a), we use $n=4096$ fitting words of length 100 and fix $j=64$.
Each token is sampled independently and uniformly from the four generators.
For each word, we calculate eight numbers to use as the regression's inputs,
and collect them in a vector
\[
d=(X_{j-1},Y_j,Z_j,S_j,\mathrm{onehot}(g_j)),\qquad
S_j=\sum_{i\le j}|u_i|.
\]
The first four entries are the preceding net count $X_{j-1}$, the current
coordinates $Y_j$ and $Z_j$, and $S_j$, the total number of $a$ or $a^{-1}$
tokens through position $j$, without subtracting inverse tokens. The last
four entries identify the current token. We use
$\mathrm{onehot}(a)=(1,0,0,0)$, $\mathrm{onehot}(a^{-1})=(0,1,0,0)$,
$\mathrm{onehot}(b)=(0,0,1,0)$, and
$\mathrm{onehot}(b^{-1})=(0,0,0,1)$.
All eight numbers come from the input word, not from model predictions.
For example, the prefix $aab$ at $j=3$ gives $d=(2,1,2,2,0,0,1,0)$.
The fitted map predicts the activation vector as
\[
\widehat h_j(d)=\alpha+B^\top d
=\alpha+\beta_x X_{j-1}+\beta_y Y_j+
\beta_z Z_j+\beta_s S_j+B_{\rm tok}\,\mathrm{onehot}(g_j).
\]
Here $\alpha$ is a 256-dimensional constant vector, each $\beta$ is a
256-dimensional coefficient vector, and the four columns of
$B_{\rm tok}$ give token-specific offsets. The matrix $B\in\mathbb R^{8\times256}$
collects these eight coefficient vectors as rows. The extra inputs let the
fit account for changes associated with $Y_j$, $Z_j$, $S_j$, and the current
token instead of assigning all such changes to $X_{j-1}$. This is a linear
adjustment, not a guarantee that the count is causally isolated.

We use one $\beta_x$ for all four current tokens; only the token offsets
change. We fit with the ridge penalty $10^{-3}$ and express the coefficients
in the original count units, as detailed below. Thus $\beta_x$ is the fitted
change in all 256 activation values when the preceding net count increases
by one and the other regression inputs are held fixed. We do not rescale
$\beta_x$ to have unit length.

\paragraph{From the fit to an intervention.}
Let $e_1=(1,0,\ldots,0)$ have the same length as $d$. Adding $\delta e_1$
to $d$ increases only its first entry, the preceding net count, by $\delta$.
The fitted map then changes by
\[
\widehat h_j(d+\delta e_1)-\widehat h_j(d)=\delta\beta_x.
\]
We add this difference to the model's observed activation, rather than
replacing that activation by the regression prediction. Thus $\delta=2$
adds twice the fitted per-count vector. The part of the actual activation
not explained by the regression, $h_j-\widehat h_j(d)$, is left unchanged.
Holding the other regression inputs fixed describes a change within the
fitted linear map; it does not edit the input word or guarantee a pure or exact
change of the model's internally represented count.

\paragraph{Intervention and output measurement.}
We replace only $h_j$ by $h_j+\delta\beta_x$, then apply the original output
projection and all remaining network operations. For predicted probabilities
$p_t(z)$ over the integer classes $\mathcal Z=\{-512,\ldots,512\}$, define
\[
\hat z_t=\frac{\sum_{z\in\mathcal Z}zp_t(z)}
                   {\sum_{z\in\mathcal Z}p_t(z)},\qquad
\Delta\hat z_t=\hat z_t^{\rm pert}-\hat z_t^{\rm clean}.
\]
The labels ``pert'' and ``clean'' denote evaluations with and without the
activation edit. The extra output class for values outside $[-512,512]$ has
no single numerical coordinate, so we omit it and divide by the probability
remaining on the integer classes. Thus $\hat z_t$ is the mean predicted
integer coordinate. It differs from the most probable integer (the argmax),
and $\Delta\hat z_t$ measures a change in this mean rather than accuracy.

\paragraph{One position and the signed response.}
The update $Z_j=Z_{j-1}+X_{j-1}v_j$ motivates comparing the four current
tokens. An ideal increase in the preceding count increases the current $Z$
at $b$, decreases it at $b^{-1}$, and leaves it unchanged at $a^{\pm1}$.
We test how the network responds to the fitted activation perturbation.
For panel (a), we perturb $j=64$ by $\delta=\pm2$ on 1,024 held-out words.
We group test words by their token at position 64. Within each group, we
average $\Delta\hat z_t/\delta$ over test words, perturbation signs, and
output positions $t=64,72,96,100$. Dividing by $\delta$ expresses the output
change per unit of the nominal count change. A filled point is this mean
for one model; bars and strokes average the three models and are not error
bars. Open points use a random activation direction. That original control
has length matched to a count coefficient fitted separately on words with
current token $a$, rather than to the coefficient fitted on all four tokens
and used for the filled points. The two direction lengths therefore need not
match. We plot the measured responses without rescaling either result.

\paragraph{Interpreting the token-dependent response.}
\label{app:heisenberg-a-response}
The count-direction perturbations produce opposite mean responses at $b$ and
$b^{-1}$ (Fig.~\ref{fig:heisenberg-count-interventions}a), consistent with the
signs in the group update. The $a^{\pm1}$ responses are smaller but need not
vanish. To see the distinction, consider two ideal group states just before
position $j$ that differ only in $X_{j-1}$ by $\delta$. They have the same
$Y_{j-1}$ and $Z_{j-1}$ and receive identical remaining inputs. Subtracting
their update equations gives
\[
\Delta X_t=\delta,\qquad \Delta Y_t=0,\qquad
\Delta Z_t=\delta\sum_{i=j}^{t}v_i\quad(t\ge j).
\]
Indeed, $\Delta X_i$ remains $\delta$, so
$\Delta Z_i=\Delta Z_{i-1}+\delta v_i$ with $\Delta Z_{j-1}=0$.
Summing yields the expression above. The immediate change
$\Delta Z_j=\delta v_j$ is zero at $a^{\pm1}$, but later outputs can change
through subsequent $b^{\pm1}$ tokens.

Our intervention edits an activation rather than a group state, and the
plotted response averages outputs at $64,72,96,100$. The ideal immediate
zero response at $a^{\pm1}$ is therefore not a constraint on the measured
response. For seeds 42--44, the signed mean responses $\Delta\hat z/\delta$
at $a$ are $(.0097,.0181,-.0943)$, compared with
$(-.0010,.0023,-.00001)$ for the random direction. At $a^{-1}$ they are
$(-.0398,-.0108,-.0221)$, compared with $(-.0022,-.0034,-.0017)$.
We retain these residuals and controls in the figure under the scaling
specified above. The results support a token-dependent signed response,
but do not establish a pure count edit or identify the source of the residuals.

\paragraph{Two positions and the additivity test.}
Panel (b) asks whether the joint effect of two perturbations is predicted by
their separate effects. For each fixed test word, we perform four forward
passes: clean, perturbing only $j$, perturbing only $k$, and perturbing both.
Write their final predicted means as $\hat z^0$, $\hat z^j$, $\hat z^k$,
and $\hat z^{jk}$. We form
\[
E_j=\hat z^j-\hat z^0,\quad E_k=\hat z^k-\hat z^0,\quad
E_{jk}=\hat z^{jk}-\hat z^0.
\]
The plotted point is $(\mathbb E[E_j]+\mathbb E[E_k],\mathbb E[E_{jk}])$,
with expectations over the same test words. The line $y=x$ is therefore the
prediction of additive \emph{mean} effects; it does not require each individual
word to satisfy additivity. Both axes measure changes in predicted $z$, not accuracies or alternative
settings of the perturbation size.

Length-100 words have fixed tokens at positions 16, 32, 48, 64, and 80,
using either pattern P, $(b,b^{-1},b,b^{-1},b)$, or pattern Q,
$(b^{-1},b,b^{-1},b,b^{-1})$. Pattern Q inverts each specified token without
reversing their order.
Other tokens remain i.i.d.\ uniform. Each pattern has 4,096 fit and 1,024
independent test words, checked against prior data for overlap. We fit
$\beta_x$ separately at each position and pattern using
$(X_{j-1},Y_j,Z_j,S_j)$; the current token is fixed at each fitted site.
We add $+4\beta_x$ at both selected positions and read $\hat z_{100}$.
The same positive count increment is used even at a $b^{-1}$ site, so
opposite response signs are not imposed by choosing opposite injection signs.

The $(b,b)$, $(b^{-1},b^{-1})$, and mixed site pairs are respectively
$(16,48)$, $(32,64)$, and $(48,32)$ for P; they are $(32,64)$, $(16,48)$,
and $(64,16)$ for Q. The implemented mixed pairs differ from the protocol's
example $(48,64)$; the figure reports the executed sites. All 18 points are
retained, one for each combination of three seeds, two patterns, and three
pair types. A mixed pair places one perturbation at $b$ and the other at
$b^{-1}$.
The $(b,b)$ pairs lie above and right of the origin, and the
$(b^{-1},b^{-1})$ pairs below and left. Mixed-pair effects range from
$-2.770$ to $+.127$; they need not vanish because single-site responses
differ across positions, models, and patterns even with the same $\delta=4$.

\paragraph{Additivity results.}
For each of the eighteen plotted conditions, we subtract the sum of the
two mean single-position effects from the mean joint effect. This difference,
$r=\mathbb E[E_{jk}]-\mathbb E[E_j]-\mathbb E[E_k]$, is zero for exactly
additive mean effects. Across the eighteen conditions, the square root of
the mean of $r^2$ is .125 units of $z$, and the largest $|r|$ is .299. The points broadly follow the additivity line, but this
does not imply equal sensitivities at different positions. For example,
seed 42's mixed pair in pattern Q has a joint effect of $-2.770$ against
a single-site sum of $-2.988$. Its substantial nonzero response is distinct
from its .219 additivity residual.
The experiment supports approximate additivity of mean responses, without
establishing exact sample-wise cancellation.

\paragraph{How we fit the encoding map.}
For each fitting word, we have a list of known input quantities $d$ and a
recorded activation vector $h_j$. We arrange these observations as two tables,
with one row per word. The table $D\in\mathbb R^{n\times p}$ contains the
$p$ regression inputs; the table $H\in\mathbb R^{n\times256}$ contains the
256 activation values we want to predict. Thus the same row in $D$ and $H$
always refers to the same word. Panel (a) has $n=4096$ rows and $p=8$ input
columns. Column 1 of $D$, for example, contains $X_{j-1}$ for all fitting words.

We first put the input columns on comparable scales. Let $\bar d_r$ be the
mean of input column $r$ across the fitting words, and let $s_r$ be its
standard deviation, floored at $10^{-8}$. The standard deviation uses divisor
$n$. We subtract the mean and divide by $s_r$ in each column. For each
activation column, we subtract its mean but do not divide by its standard
deviation. In matrix notation these operations are
\[
F_{ir}=\frac{D_{ir}-\bar d_r}{s_r},\qquad
H_c=H-\mathbf1\bar h^\top.
\]
Here $i$ indexes fitting words, $r$ indexes regression inputs, $\bar h$ is
the vector of 256 mean activation values, and $\mathbf1$ is a column of $n$
ones. Thus $F$ contains the rescaled regression inputs and $H_c$ contains
the activations after subtracting their means. This procedure also applies
to the four token-indicator columns.

We fit a coefficient matrix $W\in\mathbb R^{p\times256}$ so that $FW$
predicts $H_c$. We minimize the squared prediction errors across all words
and activation coordinates, with a penalty on large coefficients,
\[
\widehat W=\arg\min_W\left\{
\frac1n\|H_c-FW\|_F^2+\lambda\|W\|_F^2\right\},
\qquad \lambda=10^{-3}.
\]
The notation $\|M\|_F^2$ means the sum of the squared entries of a matrix
$M$. The first term measures fitting error; the second is the ridge penalty.
Setting the derivative with respect to $W$ to zero gives
\[
(F^\top F+n\lambda I_p)\widehat W=F^\top H_c,
\]
where $I_p$ is the $p\times p$ identity matrix. We solve this linear system
for $\widehat W$.

The intervention is specified in actual count units, so we undo the column
rescaling before using the fitted coefficients. Dividing row $r$ of
$\widehat W$ by $s_r$ gives row $r$ of $B$. We restore the mean activation
through the constant $\alpha$ and take the first coefficient vector as
the count direction,
\[
B_{r:}=\widehat W_{r:}/s_r,\qquad
\alpha=\bar h-B^\top\bar d,\qquad
\beta_x=B_{1:}^{\top}.
\]
The notation $B_{r:}$ denotes every entry in row $r$. These definitions give
$\widehat h_j(d)=\alpha+B^\top d$ in the original units. The constant
$\alpha$ is not penalized. All means, scales, and fitted coefficients are
computed from fitting words and kept fixed for the independent test words.

For panel (b), we repeat the same procedure separately for each model,
position, and token pattern, with 4,096 fitting words per pattern. There
are only four input columns, $d=(X_{j-1},Y_j,Z_j,S_j)$, because the current
token is fixed at each chosen position. Consequently these fits do not use
token-indicator columns and do not share a count direction across positions.

\subsection{Activation patching between matched inputs}
The preceding interventions add a fitted direction. Here we instead transfer
activations between two inputs that differ by a controlled change in the
preceding count, and test which replacements transmit the corresponding
change in $Z$. We call the unmodified run the recipient and the run on the
edited input the donor. We use the same three frozen models. Readouts fit
4,096 new words and are checked on 1,024 independent words. Three intervention
rounds each use 2,048 fresh pairs: 6,144 distinct pairs in total, reused across
models. At position $j=32$ or 64, changing $a$ to $a^{-1}$ or conversely at
$j-1$ changes $X_{j-1}$ by $\delta=\pm2$, while preserving $Y_{j-1}$,
$Z_{j-1}$, the current token, and the suffix. The candidate local contribution
therefore changes by $\delta v_j$. We replace selected recipient activations by the corresponding donor
activations, while leaving the recipient's other computation paths intact.

We first test the input to the second layer's feedforward network (MLP).
The quantity of interest is $C_j=X_{j-1}v_j$, the increment added to $Z$
at the current token. We compare replacing the full MLP input with transferring
only a fitted linear component associated with this increment. Subsequent exploratory rounds test attention
values and then queries, keys, and their combinations, using fresh pairs and
protocol amendments written before each round. We check that inserting an
activation's original value leaves the prediction unchanged, restoring the
original MLP input removes the corresponding input-path edit, and replacements
leave the coordinates outside the selected intervention unchanged. Model
weights remain fixed.

To compare the response across the two count changes and current tokens, we
divide the change in predicted mean $z$ by $\delta v_j$, the change in the
task's current central-coordinate increment. We average this ratio over words
with current token $b$ or $b^{-1}$ and output positions $j,j+8,j+32,100$.
A ratio of one means that the mean output response matches that current
increment. It does not mean that individual predictions are correct or that
this increment equals the full change in the true state at later positions.

Table~\ref{tab:heisenberg-causal} separates an internal readout from effects
on the output. The MLP $C$ column reports the change in the probe's estimate of $C_j$
after replacing an MLP input, divided by $\delta v_j$. It measures an internal
readout, not the model's output coordinate.

The isolated-$C$ column measures the downstream
response when transferring only a fitted component associated with $C_j$.
To construct it, we fit an encoding regression for the second-layer MLP
output using $X_{j-1},Y_{j-1},Z_{j-1}$, indicators for three of the four
current tokens, and $C_j$. The omitted token is represented by all-zero
indicators and the fitted constant. We take the coefficient of $C_j$ and
rescale it so that moving one unit along it increases the linear probe's
estimate of $C_j$ by one. We then measure the donor-minus-recipient change
in that estimate and multiply this difference by the rescaled direction.
Adding the resulting vector to the recipient's MLP output transfers the
decoded change without replacing the full output. This construction does
not guarantee that other information in the activation remains unchanged.

The layer-3 KV column instead transfers the donor's full key ($K$) and
value ($V$) vectors at the selected position in the third attention layer.
Keys determine how other positions attend to this position; values supply
the information they receive. Both this column and the isolated-$C$ column
report the normalized change in the model's predicted mean $z$. Its accompanying exact-shift accuracy
measures how often the argmax prediction changes by the expected $\pm2$.

\begin{table}[ht]
\centering\small
\begin{tabular}{lrrrr}
\toprule
Seed & MLP $C$ readout & Isolated $C$ & Layer-3 KV & Exact shift (\%)\\
\midrule
42 & .861 & .00021 & .909 & 83.08\\
43 & .983 & -.00021 & .638 & 27.44\\
44 & .651 & -.00001 & .812 & 61.21\\
\bottomrule
\end{tabular}
\caption{Count-sensitive readouts and interventions. The MLP column measures
the change in the decoded local contribution after an input replacement;
isolated $C$ and KV columns measure changes in the predicted mean $z$.
All three are divided by the task-level increment change $\delta v_j$. Exact shift is
the fraction of argmax predictions changing by the expected $\pm2$ under KV
replacement. All models use learning rate $10^{-3}$.}
\label{tab:heisenberg-causal}
\end{table}

Third-layer KV replacements retain $(x,y)$ predictions in 100\% of
cases, but the isolated linear $C$ component carries only a small fraction
of the predicted effect (Table~\ref{tab:heisenberg-causal}). Full donor vectors
can contain other features; their effect is not specific evidence for $C$.
Differences between seeds' key and value sensitivity do not establish distinct
high-level algorithms. These tests identify activation replacements that affect the central-coordinate
prediction, but do not isolate the count contribution as their causal feature.

Further exploratory probes estimate unsigned counts, which count $a$ and
$a^{-1}$ together rather than subtracting them. However, probes fitted to
randomly reassigned targets reach $R^2=.079$, above the prespecified maximum
of .02 allowed for that control.
We therefore do not treat weak readouts of counts over only a short part of
the preceding sequence as confirmed evidence.
A separate 2,048-word direction-injection test distinguishes net count $X$
from unsigned count $S$. At the first layer and position 64, a nominal $+2$
$X$ injection shifts the current predicted mean $x$ by .0011, .1792, and
.0012 in seeds 42--44, respectively; the same $S$ injection changes the
current predicted mean $x$ by less than $4\times10^{-6}$ in each model.
Larger injections introduce nonspecific disruption. Together, the direction interventions show that count-associated activation
changes affect central-coordinate predictions, and the patching tests show
that selected activation replacements transmit part of the expected change.
They do not yet isolate a complete implementation of the prefix-count update
or exclude other order-sensitive computations.

\section{Gradient conditions, diagnostics, and interventions}
\label{app:gradient-details}

We distinguish three questions in the gradient analysis. First, does the
average loss gradient provide a signal for separating states within a
learned quotient class? Second, under what conditions would that signal
cancel? Third, can changing the learning signal improve tracking? The
measurements, sufficient condition, and interventions below address these
questions separately.

For intuition, consider a class with two possible target states. The loss
on one example favors the first state, while the loss on a related example
favors the second. If parameter changes affect both examples' logits
identically, these opposing target-specific contributions cancel on average.
If the sensitivities differ, the model can instead receive a signal that
separates the two states, even when it currently assigns them equal
probabilities. The formal argument below makes this distinction explicit.

\paragraph{From cross-entropy to class-member contrasts.}
At one input prefix, let $z_\theta$ be the vector of logits, $p$ its softmax,
$y$ the one-hot target, and $J=\partial z_\theta/\partial\theta$ the Jacobian
with respect to all trainable parameters. The cross-entropy gradient is
$J^\top(p-y)$. Define the orthogonal projection $P$ by
$(Pv)_g=|K|^{-1}\sum_{h\in gK}v_h$, where $K=[G,G]$. It replaces each
coordinate by the average over its quotient class. Thus $P(p-y)$ is constant
within classes, whereas $(I-P)(p-y)$ contains only contrasts between their
members. Applying $J^\top$ to the latter gives
\begin{equation}
\label{eq:fiber-gradient-decomposition}
h=J^\top(I-P)(p-y)
 =\underbrace{J^\top(Py-y)}_{\ell\;\text{(Label)}}
 +\underbrace{J^\top(p-Pp)}_{m\;\text{(Model)}}.
\end{equation}
The identity follows by expanding $(I-P)(p-y)$. The Label term compares the
true label with a uniform label on its class. The Model term compares current
probabilities with their within-class averages. Their sum, Net, is a component
of the cross-entropy gradient, not the full gradient or the actual AdamW update.

\paragraph{Constructing examples and averaging.}
We use raw final weights of three recipe-A models per group for $A_4$,
$\mathrm{SL}(2,3)$, and $S_4$, trained for 74,219 updates with seeds 42--44.
For each of 32 fresh i.i.d.\ base sequences (sampling seed 189401), replace
its first input $x_1$ by $k x_1$ for every $k\in K$. At any measured position,
the resulting targets are $kq_t$. They enumerate the complete true class
because $K$ is normal. We call this set of input variants an \emph{orbit}.
The quotient target is unchanged, while every class member becomes a target
exactly once. We compute $\ell,m,h$ on each variant, average vectors over
the complete orbit and then over the 32 base sequences, obtaining
$\bar\ell,\bar m,\bar h=\bar\ell+\bar m$. Norms are taken after averaging.
We never average parameter vectors across separately trained models.

\paragraph{Reading the two panels.}
We measure positions 1--40, 60, 80, and 100 at a fixed checkpoint.
Panel (a) plots $\|\bar\ell_t\|$, $\|\bar m_t\|$, and $\|\bar h_t\|$.
Within each model, all three are divided by the same reference
$S=\max_t\max(\|\bar\ell_t\|,\|\bar m_t\|)$ over this grid.
Panel (b) plots
\[
c_t=\frac{\bar\ell_t^\top\bar m_t}
          {\|\bar\ell_t\|\,\|\bar m_t\|}.
\]
A cosine of $-1$, $0$, or $1$ means opposite, orthogonal, or aligned average
vectors. Cancellation depends on both direction and magnitude, as
\[
\|\bar h_t\|^2=\|\bar\ell_t\|^2+\|\bar m_t\|^2
 +2\|\bar\ell_t\|\|\bar m_t\|c_t.
\]
We compute all quantities within each model before plotting the three-seed
mean and minimum--maximum range. Bands describe training-run variation,
not uncertainty from the 32 sampled sequences. We apply no smoothing.
Regions I--III are spatial guides, not inferred mechanism boundaries. Region
I ends at the group mean of $\max(1.5,F_e-.5)$; region III is positions 17--100.

\begin{figure}[!tb]
\centering
\includegraphics[width=\linewidth,trim=0 10bp 0 4bp,clip]{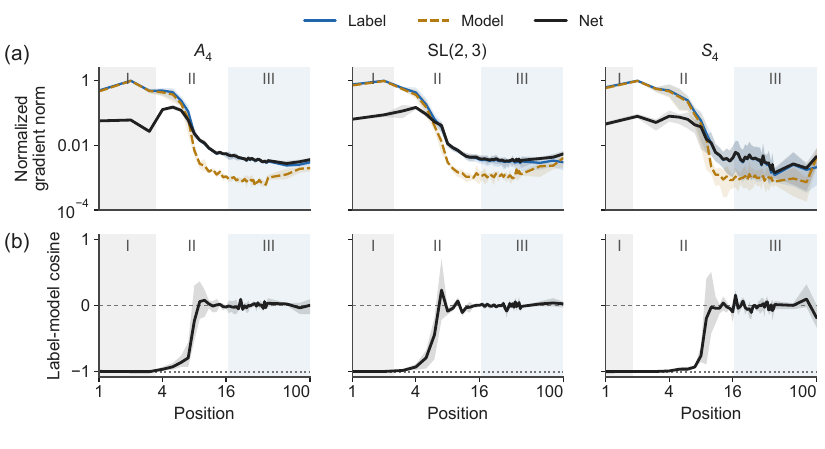}
\caption{\textbf{Learning signals within quotient classes.}
Matched $A_4$, $\mathrm{SL}(2,3)$, and $S_4$ Transformers, three seeds each.
(a) Norms of averaged Label, Model, and Net terms, divided by the common
per-model reference $S$. (b) Cosine between the averaged Label and Model
terms. Each position uses 32 complete input orbits. Lines and bands show
seed means and ranges; positions vary within a fixed checkpoint. The
definitions and averaging procedure are given in this appendix.}
\label{fig:matched-gradients}
\end{figure}

\paragraph{What the spatial profiles establish.}
The late-position net signal is consistently small relative to each model's
peak, but the early profile is not universally a near-zero region followed by
a frontier peak. Table~\ref{tab:gradient-profile-readout} reports all nine runs.
At position 1, all have perfect argmax accuracy, yet their net norms are
30--64\% of their own peak. Cross-entropy can still improve prediction
confidence after accuracy reaches 100\%. For $S_4$ seed 42, the largest net
signal occurs at position 2 with 96.19\% accuracy. Conversely, the position-40
norm is only 0.8--3.5\% of the peak in every run. These ratios use
$\max_t\|\bar h_t\|$, whereas the figure uses $S$ to compare its three terms.

\begin{table}[ht]
\centering\small
\begin{tabular}{llrrrr}
\toprule
Group & Seed & $t_*$ & Accuracy at $t_*$ & $\|\bar h_1\|/\|\bar h_{t_*}\|$ & $\|\bar h_{40}\|/\|\bar h_{t_*}\|$\\
\midrule
$A_4$ & 42 & 5 & 53.44 & 30.33 & 1.93\\
$A_4$ & 43 & 5 & 58.79 & 38.58 & 2.27\\
$A_4$ & 44 & 4 & 81.93 & 45.25 & 2.53\\
$\mathrm{SL}(2,3)$ & 42 & 4 & 50.21 & 40.95 & 2.34\\
$\mathrm{SL}(2,3)$ & 43 & 4 & 36.27 & 46.01 & 3.51\\
$\mathrm{SL}(2,3)$ & 44 & 3 & 88.59 & 37.97 & 1.13\\
$S_4$ & 42 & 2 & 96.19 & 63.56 & 1.82\\
$S_4$ & 43 & 4 & 42.22 & 33.42 & 0.81\\
$S_4$ & 44 & 6 & 12.16 & 50.23 & 1.94\\
\bottomrule
\end{tabular}

\caption{Per-model net-gradient profiles behind
Fig.~\ref{fig:matched-gradients}. $t_*$ maximizes $\|\bar h_t\|$ over the
measured grid. Ratios and accuracy are percentages. Accuracy at position 1 is
100\% in every row.}
\label{tab:gradient-profile-readout}
\end{table}

Two different effects can produce a small net gradient. When $p$ approaches
the correct one-hot label, $p-y$ becomes small and the Label and Model terms
balance. Alternatively, class-member signals can become small only after
averaging examples whose targets differ. These are distinct from cancellation
between the two already-averaged terms. In the late window, absolute
Label--Model cosines have medians .039, .026, and .057 for the three groups;
they do not show two large, oppositely directed mean terms. Small mean vectors
also make their directions sensitive to sampling. These are spatial
measurements at final checkpoints, not a proof of permanent training arrest.

\paragraph{Numerical checks.}
Each orbit-mean proxy is differentiated directly using a shared full-orbit
forward batch, with TF32 disabled. We use FP32 derivatives and FP64
accumulation, independently differentiate the three terms, and check
$\bar h=\bar\ell+\bar m$ relative to $\|\bar\ell\|+\|\bar m\|$.
The original $S_4$ seed-44 profile failed the $10^{-3}$ tolerance at five
positions; its entire 43-position profile is therefore recomputed in FP64.
The original measurements and comparison are retained. The within-orbit
ratio $R_{\mathrm{label}}=\|\mathbb E_k\ell_k\|/\mathbb E_k\|\ell_k\|$
measures cancellation across variants and differs from the pooled two-term
balance $B=\|\bar h\|/(\|\bar\ell\|+\|\bar m\|)$ and from $c_t$.

\paragraph{Cancellation within a class at the matched checkpoints.}
We additionally measure the nine checkpoints above using a dense census
over positions 1--40 (Fig.~\ref{fig:matched-cancellation}). To distinguish
cancellation across examples from cancellation between average terms, we
compute two dimensionless ratios. Within a complete input orbit, let
$\ell_k=J_k^\top(Py_k-y_k)$ be the Label contribution defined in
Eq.~\ref{eq:fiber-gradient-decomposition}. We measure
\[
R_t=\frac{\|\sum_{k\in K}\ell_k\|}{\sum_{k\in K}\|\ell_k\|}.
\]
An aligned set of contributions has $R_t=1$, while exact cancellation gives
$R_t=0$. We report the median of this ratio over eight orbits per position.
The dashed $1/\sqrt f$ levels are geometric references for $f$ equal-norm,
pairwise-orthogonal vectors, not a fitted baseline or a null distribution
for the model's gradients. At position 40, the nine ratios range from
.00796 to .03322, below these references.

For the second ratio we pool Label and Model vectors across the eight
complete orbits and compute
\[
B_t=\frac{\|\bar\ell_t+\bar m_t\|}
           {\|\bar\ell_t\|+\|\bar m_t\|}.
\]
A small $B_t$ indicates cancellation between these two averaged terms;
it does not measure cancellation within an orbit. At position 40,
$B_t=.634$--$.837$, unlike the small within-orbit $R_t$.
Accuracy in this census uses 8,192 sequences, independently of the gradient
samples. All curves use the same frozen weights, and no seed averaging or
smoothing is applied. Both gradient panels in this census use eight orbits;
the earlier norm and cosine profiles use 32 complete orbits. Their estimates
must not be interchanged merely because they use the same checkpoints.

\begin{figure}[htbp]
\centering
\includegraphics[width=\linewidth]{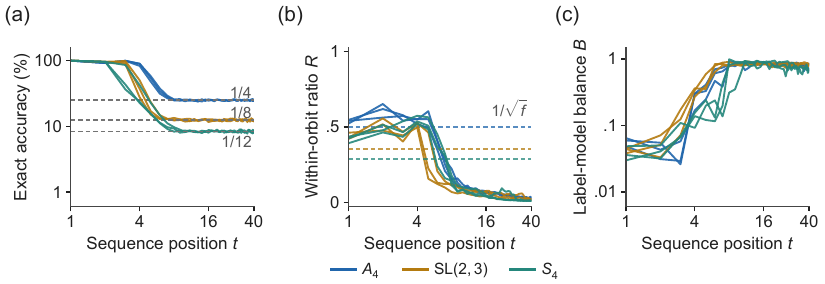}
\caption{\textbf{Cancellation across examples differs from balance between mean terms.}
The same nine $A_4$, $\mathrm{SL}(2,3)$, and $S_4$ checkpoints as
Fig.~\ref{fig:matched-gradients}, using the additional census defined above.
(a) Exact accuracy and class-size references. (b) Within-orbit Label ratio
$R$, with geometric $1/\sqrt f$ references. (c) Label--Model balance $B$
after pooling those orbits. Each line is one seed. The small late-position $R$
supports cancellation across class members, while $B$ distinguishes this
from cancellation of two large mean terms.}
\label{fig:matched-cancellation}
\end{figure}

These measurements identify a candidate obstacle to refinement. They do
not establish that the total gradient vanishes, that no information changes
during a plateau, or that cancellation causes a model to remain there.

\paragraph{Sufficient conditions.} Consider right multiplication of one input
slot by an element of $K=[G,G]$. Independent uniform full-group inputs are
invariant under this action. The target remains in the same $K$-coset, and
averaging over the complete action sends it uniformly through that coset.
Assume (R1) logits and (R2) their parameter Jacobian are invariant under the
action at the position being studied. Orbit averaging then replaces the one-hot
label $y$ by the uniform-fiber label $u$ while preserving the other terms,
giving
\begin{equation}
\mathbb E[J^\top(p-y)]=\mathbb E[J^\top(p-u)].
\label{eq:gradient-cancellation}
\end{equation}

R2 does not follow from R1. For a binary logit $z=\theta x$, with equally likely
$x\in\{-1,1\}$ and label $(x+1)/2$, predictions are identical at $\theta=0$.
Nevertheless, the population cross-entropy gradient is $-1/2$, because the
Jacobian changes sign with $x$. Output uniformity therefore does not by itself
prevent refinement. Our measurements test signatures consistent with the
condition; they do not establish invariance of the complete parameter Jacobian.

\paragraph{Directional identity.} For nonzero gradient contributions $g_i$,
write $v_i=g_i/\|g_i\|$. Their mean pairwise cosine satisfies
\begin{equation}
\label{eq:cancel}
\overline{\cos}=-\frac{1}{f-1}
+\frac{\|\sum_i v_i\|^2}{f(f-1)}.
\end{equation}
This follows by expanding the squared sum. It is an identity for normalized
vectors, distinct from the pooled label/model cosine in
Fig.~\ref{fig:matched-gradients}b. Fourteen of eighteen measurements over
fiber sizes 2--60 are within .0015 of the lower bound; all lie at or above it.
A second prospective $S_5$ checkpoint reproduces the $f=60$ signature
(-.01692 versus $-1/59$). The abelian positive control gives at least .9969.
These checks do not independently prove the population invariance assumptions.

\paragraph{Profile qualifications.} At $S_4$ seed 42's largest net signal,
position 2, exact accuracy is 96.19\%, label/model cosine is -.9975, and
$\|\bar h\|/(\|\bar\ell\|+\|\bar m\|)=.0351$. A net-signal peak can therefore
coexist with substantial cancellation. Late-position cosine estimates vary
across seeds because the mean vectors are small.

\paragraph{Cross-group comparison of the average Label signal.}
A separate study measures 44 checkpoint--input-distribution pairs across
sixteen groups. For each pair, we average the Label gradients over the full
commutator orbit at input position 1 and over 32 base sequences from that
model's input distribution. This gives $\bar\ell_t$ as above. Since the
orbit-averaged label residual $Py-y$ is zero, this quantity measures its
covariance with the logit Jacobian. It differs from the orbit cancellation
ratio $R$, which normalizes by the individual gradient norms.

We choose the measured positions from behavior before measuring gradients.
Let $F$ be the exact frontier at threshold .75. The reference position is
$t_e=\max(2,F+1)$ and the later position is
$t_f=\min(100,\max(40,t_e+20))$. A pair is eligible if $t_f-t_e\ge20$,
quotient accuracy at both positions is at least .90, exact accuracy at
$t_e$ exceeds $1/f+.05$, and exact accuracy at $t_f$ is within .10 of $1/f$.
Twenty eligible pairs had not had these gradient measurements inspected
before the protocol was fixed. Their ratios
$\|\bar\ell_{t_f}\|/\|\bar\ell_{t_e}\|$ are .0009--.0149. Thus the result
concerns positions with an acquired quotient and class-size accuracy,
not every position beyond the frontier.

A ten-pair follow-up instead uses $t_f=100$ and the maximum average Label
norm over $\{\max(2,F-2),\ldots,\min(100,F+2)\}$ as its reference. This
window must contain accuracies both above and below .75; position 100 must
be at least 20 positions beyond the window, with quotient accuracy at
least .90 and exact accuracy within .10 of $1/f$. Seven pairs meet these
conditions and show suppression; three fail them. This is a design repair
after inspecting the first results, not an independent confirmation.
The excluded $D_{15}$ run is still improving and has ratio 1.14; two
excluded runs have unlearned quotients.

For three two-coset $Q_8$ seeds, we use the same frontier-window reference
and select the first position in $\{40,50,\ldots,100\}$ satisfying the
distance, quotient-accuracy and class-size-accuracy conditions. The ratios
are .0012, .0009, and .0284, with the largest on an accelerating trajectory.
These follow-ups use the same 32-sequence orbit averaging, but their
different position-selection rules prevent pooling the ratios into one statistic.

\paragraph{Intervention protocols.} The generic-release continuation starts
from the $D_4$ seed-43 checkpoint at 200k updates. The forward computation is
unchanged; a custom backward rule rotates hidden gradients at the final
layer-normalization output. Opposite signs are assigned to sequences paired
by a central action at input position 7. Identity and pair-shared rotations
provide controls. The rotation magnitude is .1 and applies at positions
14--50 for 10k updates, followed by 5k updates with ordinary backpropagation.
All four arms use the same paired stream, batch 256, length 100, and a freshly
initialized AdamW optimizer at $5\times10^{-5}$ with zero weight decay.
Evaluation uses 2,048 sequences every 1,000 updates and 8,192 at the endpoint.

For the aligned intervention, let $v=(\theta_{200k}-\theta_{250k})/
\|\theta_{200k}-\theta_{250k}\|$, so a gradient step along $-v$ points toward
the observed future checkpoint. With $n$ token losses and active set $B$, the
custom backward rule adds the rank-one term
\[
A_i=\frac{n}{|B|}\frac{r_i^H}{\|r_i^H\|^2}v^\top,
\qquad r_i^H=(I-P)(p_i-y_i),\quad i\in B,
\]
to the parameter-to-logit Jacobian, scaled by the intervention strength.
Thus mean cross-entropy adds a controlled gradient along $v$ while leaving
the forward logits unchanged. The active positions are 14--18. This study
uses SGD at $5\times10^{-5}$ without clipping, with the same batch, evaluation,
10k intervention, and 5k washout schedule. The 30\% dose is calibrated to the
distance between the base and future checkpoints. Lower-dose, opposite-sign,
and orthogonal-direction arms are controls. The last-block comparison restricts
this added direction to the final transformer block or its parameter complement;
it does not replace the block's native hidden-state Jacobian.

\paragraph{Intervention outcomes.} The generic backward perturbation
has 17-fold selectivity in the plateau region versus 1.23 near the frontier.
All four continued arms, including the identity control, end at frontier 13
from a starting frontier of 14. Generic release thus provides no improvement
relative to the control. Alignment with a future-training direction gives
signed, dose-dependent local improvement that survives washout, but does not
cross the discrete frontier at the tested 30\% dose. Because it uses future
information, it is a diagnostic upper bound rather than a training method.
Restoring only last-block parameter support recovers 22--29\% of the effect,
below the 80\% criterion; the upstream complement is three times more efficient
per unit $L_2$ perturbation. The intervention chain rests on four matched pairs.

\paragraph{Larger-model check.} Of three $S_3$ models trained from scratch in
the Pythia-160M architecture with the Li recipe, one enters a parity-first
phase. In that phase, coset mass is at least .99 and within-fiber entropy is
1.580--1.583 bits, close to $\log_2 3$. Orbit ratios are .02--.06 and mean
pairwise cosine is -.50. These statistics use 24 sequences per cell, compared
with 256 for the behavioral check. Other seeds learn associative-like
solutions and can be confidently wrong beyond their exact frontier, with
coset mass .50 and sharply nonuniform within-fiber outputs. This limits the
quotient description to the observed regimes.

\clearpage
\section{Internal geometry of Transformer quotient predictions}
\label{app:transformer-geometry}

We ask whether the quotient classes visible in Transformer predictions also
have an identifiable internal representation. We find that their mean hidden
vectors usually occupy a small subspace whose geometry agrees with the
abelian quotient. Replacing the activation in this subspace transfers the
predicted class at the intervention position. We then test a separate
question, whether replacing one position also transfers that class to later
predictions. This separates a representation used for the current readout
from a state that is carried through subsequent updates.

\subsection{Models and construction of class means}
\label{app:transformer-geometry-protocol}

We analyze 40 saved training endpoints across the eleven groups in
Table~\ref{tab:transformer-geometry}. All use four-layer, width-256 recipe-A
Transformers with uniform i.i.d. full-group inputs. The collection includes
both original- and doubled-budget endpoints from the training cohorts in
App.~\ref{app:rows}; endpoints from the same training seed are not independent
replicates. We reevaluate the frozen models in fp32 on length-100 sequences,
without further training. All measurements below use the 51st input position
($t=51$, zero-based index 50, with no beginning-of-sequence token).
The abelian $C_8$ task serves as a full-state control.

Let $A=G/[G,G]$ be the abelianization, $K=|A|$, and $c_i\in A$ the true
quotient class of the prefix product in sequence $i$, with $n_c$ sequences
assigned to class $c$. We collect the hidden
vector $h_i\in\mathbb R^{256}$ after the final LayerNorm. For each class we
average the vectors assigned to it, and then center these $K$ means,
\begin{equation}
 \mu_c=\frac{1}{n_c}\sum_{i:c_i=c}h_i,\qquad
 \bar\mu=\frac1K\sum_{c\in A}\mu_c,\qquad
 M_{c,:}=(\mu_c-\bar\mu)^\top.
 \label{eq:transformer-class-means}
\end{equation}
We use 4,096 sequences per endpoint, or 16,384 for the two Heisenberg groups.
These are averages over ground-truth classes, not clusters fitted to the
activations. If $M=L\Sigma V^\top$ is its singular value decomposition, we
retain the $r$ columns of $V$ whose singular values satisfy
$\sigma_j/\sigma_1\ge .10$. This threshold is an analysis choice. We denote
these orthonormal directions by $U\in\mathbb R^{256\times r}$ and the
projected mean by $v_c=U^\top(\mu_c-\bar\mu)$. Centering alone guarantees
rank at most $K-1$; our question is which smaller dimension and geometry
the models use.

\subsection{Geometry follows the quotient action}
\label{app:transformer-geometry-action}

Fig.~\ref{fig:transformer-geometry} shows representative class means. The
$A_4$ model places its three quotient classes near a triangle. The $Q_8$
and $\mathrm{Dic}_3$ models both have four classes, but their geometries
support different actions. The former uses two sign directions for
$C_2\times C_2$, while the latter supports a quarter-turn for $C_4$.
The $C_8$ model uses an approximately octagonal arrangement. The
$H_3(\mathbb Z/5)$ example separates its 25 classes using two
two-dimensional planes, each of which groups them into five clusters.

\begin{figure}[t]
 \centering
 \includegraphics[width=\linewidth]{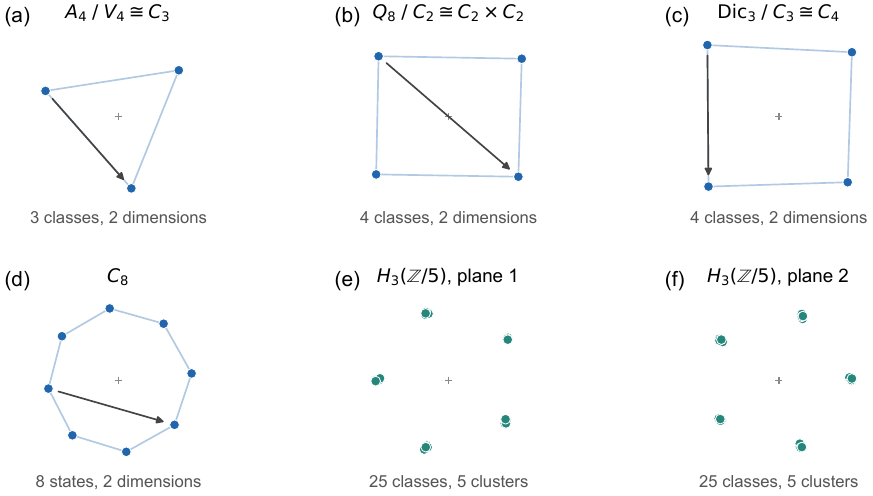}
 \caption{\textbf{Transformer class means reflect the learned abelian quotient.}
 Each dot is the mean hidden vector of one true class at $t=51$.
 Panels (a--d) show the two leading singular directions. Edges are visual
 guides; arrows connect the identity class mean to its image under one
 quotient generator, rather than tracing individual hidden-state updates.
 Panels (e,f) show two character planes of one $H_3(\mathbb Z/5)$ model.
 Several of its 25 means overlap in each plane, while their paired coordinates
 distinguish all 25 classes. Axis orientation and scale are arbitrary and
 differ between panels. Examples use seed 42 for $A_4$, doubled-budget $Q_8$
 and doubled-budget $H_3(\mathbb Z/5)$, and seed 44 for $\mathrm{Dic}_3$ and
 $C_8$. Table~\ref{tab:transformer-geometry} includes all endpoints.}
 \label{fig:transformer-geometry}
\end{figure}

We test these patterns by fitting the action on the class means. For every
nonidentity $a\in A$, we find the orthogonal matrix $R_a$ minimizing
$\sum_c\|R_av_c-v_{ca}\|^2$. Its relative fit error is
\begin{equation}
 \epsilon_a=
 \left(\frac{\sum_c\|R_av_c-v_{ca}\|^2}{\sum_c\|v_{ca}\|^2}\right)^{1/2}.
 \label{eq:transformer-action-fit}
\end{equation}
The largest error across all fitted maps and endpoints is .124.
These fits test whether the geometry accommodates the known quotient
permutations. They do not measure the model's token-by-token transformation
of individual activations.

The low dimensions have a natural algebraic interpretation. A complex
character of an abelian group is a map $\chi:A\to\{z\in\mathbb C:|z|=1\}$
with $\chi(ab)=\chi(a)\chi(b)$. Its real and imaginary parts form a plane
on which group elements act as rotations; a real-valued character gives
a sign direction. A collection is faithful when it distinguishes all
group elements. Thus two independent sign directions suffice for
$C_2\times C_2$, one rotation plane for $C_4$ or $C_8$, and two independent
rotation planes for $C_3^2$ or $C_5^2$.
For Fig.~\ref{fig:transformer-geometry}e,f, we use the two strongest
independent character pairs of $C_5^2$. For each character we form
$w_\chi=\sum_c(\mu_c-\bar\mu)\overline{\chi(c)}$ and orthonormalize its
real and imaginary parts to obtain the displayed plane. This visualization
uses the quotient labels; it is not a label-free discovery procedure.

In 38 of 40 endpoints, the retained rank and the fitted maps' trace and
determinant patterns agree with a faithful real representation of minimum
dimension. We compare multisets of traces within each element-order and
determinant-sign category, with absolute trace tolerance .30; this is a
numerical compatibility check, not a proof of an exact representation.
Determinants matter here. For $C_2^2$, the three nonidentity maps have
traces approximately $0,0,-2$ and determinant signs $-1,-1,+1$; for $C_4$
the traces are the same but all determinants are positive.
Among the 28 endpoints for which $K-1$ exceeds the minimum faithful
dimension, none uses all $K-1$ centered class directions.

We retain both exceptions in Table~\ref{tab:transformer-geometry}.
One $Q_8$ endpoint (seed 44, original budget) has rank one and only .510
quotient accuracy at this position. One doubled-budget
$H_3(\mathbb Z/5)$ endpoint (seed 44) has rank six rather than four,
consistent with an additional character pair. The observed preference
for minimum dimension is therefore common in this collection, not universal.

\begin{table}[t]
 \centering
 \caption{\textbf{Geometry and same-position transfer across all 40 Transformer
 endpoints.} $n$ counts endpoints, $r$ is the retained class-mean rank, and
 $r_{\min}$ is the minimum faithful real dimension of $A$.
 $Q$ is quotient accuracy at $t=51$. The final three columns give donor-class
 agreement after replacing the class-mean subspace (Sub.), its orthogonal
 complement (Comp.), or a same-rank random subspace (Rand.). Values are
 ranges across endpoints; each random value first averages five controls.
 No endpoint is excluded, including the low-accuracy $Q_8$ endpoint.
 The protocol and agreement metric are defined below.}
 \label{tab:transformer-geometry}
 \small
 \setlength{\tabcolsep}{3pt}
 \newcommand{\transformerGeometryRows}{%
$C_8$ & $C_8$ & 3 & 2 & 2 & 0.997--1.000 & 0.930--0.991 & 0.000--0.009 & 0.000--0.001 \\
$D_4$ & $C_2^2$ & 5 & 2 & 2 & 1.000 & 0.999--1.000 & 0.000--0.001 & 0.000--0.001 \\
$A_4$ & $C_3$ & 3 & 2 & 2 & 1.000 & 1.000 & 0.000 & 0.000 \\
$Q_8$ & $C_2^2$ & 4 & 1,2 & 2 & 0.510--1.000 & 0.482--1.000 & 0.000--0.172 & 0.000--0.166 \\
$C_7\rtimes C_3$ & $C_3$ & 3 & 2 & 2 & 1.000 & 1.000 & 0.000 & 0.000 \\
$D_6$ & $C_2^2$ & 4 & 2 & 2 & 1.000 & 1.000 & 0.000--0.001 & 0.000--0.001 \\
$S_4$ & $C_2$ & 3 & 1 & 1 & 1.000 & 1.000 & 0.000 & 0.000 \\
$\mathrm{SL}(2,3)$ & $C_3$ & 3 & 2 & 2 & 1.000 & 1.000 & 0.000 & 0.000 \\
$\mathrm{Dic}_3$ & $C_4$ & 3 & 2 & 2 & 1.000 & 0.999--1.000 & 0.000 & 0.000 \\
$H_3(\mathbb Z/3)$ & $C_3^2$ & 3 & 4 & 4 & 1.000 & 0.999--1.000 & 0.000 & 0.000 \\
$H_3(\mathbb Z/5)$ & $C_5^2$ & 6 & 4,6 & 4 & 0.968--1.000 & 0.974--1.000 & 0.000--0.001 & 0.000--0.003 \\
}

 \begin{tabular}{lcrccrrrr}
 \toprule
 $G$ & $A$ & $n$ & $r$ & $r_{\min}$ & $Q$ & Sub. & Comp. & Rand.\\
 \midrule
 \transformerGeometryRows
 \bottomrule
 \end{tabular}
\end{table}

\subsection{The class-mean subspace controls the current readout}
\label{app:transformer-geometry-patch}

We test whether the identified directions affect predictions by transferring
them between sequences. We draw a fresh pool of 2,048 sequences, split into
1,024 recipient--donor pairs, independently of the sequences used to estimate
$U$. Let $h_r$ and $h_d$ be their activations at $t=51$. We replace only the
recipient's projection onto the class-mean subspace,
\begin{equation}
 h'_r=h_r+UU^\top(h_d-h_r).
 \label{eq:transformer-subspace-patch}
\end{equation}
Donor vectors, basis vectors, and recipient replacements all refer to the
same location, after the final LayerNorm and before the linear output head.
As controls we replace the complementary projection $I-UU^\top$, the whole
vector, or a random rank-$r$ projection. For the latter we orthonormalize
independent Gaussian directions and average five draws per endpoint.
Self-replacement on eight sequences per endpoint reproduces the original
logits exactly at both intervention boundaries used here, across all
endpoints and all three replacement types
(whole vector, subspace, and complement).

We score pairs with different true quotient classes. Donor-class agreement
is the fraction for which the patched model's predicted class equals the
donor's true class; recipient-class agreement is defined analogously.
Throughout this section, we assign the model's highest-probability full
state to its quotient class, rather than taking an argmax after summing
probabilities within classes. This definition also applies to $Q$ in
Table~\ref{tab:transformer-geometry}.

Of the 40 endpoints, 39 have $Q\ge .95$. Within this set, subspace
replacement gives .9297--1.0000 donor agreement, with 35 endpoints at or
above .99. Complement replacement gives at most .0091 donor agreement,
and random replacement at most .0033. The lowest recipient agreement after
complement replacement is .9422. These controls show that
the measured directions selectively transfer the current class readout.
They do not establish that the subspace is necessary, which would require
an erasure or ablation test, or that it carries all information available
in the hidden vector.

\subsection{Does a local replacement propagate to later positions?}
\label{app:transformer-geometry-propagation}

A successful readout intervention need not install a state that the model
then updates. We therefore replace the whole hidden vector at $t=51$
after the penultimate block, where the remaining attention block can still
use it at later positions. We keep all input tokens and all other positions
unchanged. If this intervention installed the donor class $c_d(t)$ in
place of $c_r(t)$ and the model carried it through the recipient suffix,
the predicted class at $u>t$ would become
\begin{equation}
 c_{\mathrm{shift}}(u)=c_d(t)c_r(t)^{-1}c_r(u).
 \label{eq:transformer-carried-class}
\end{equation}
Here $c_r(t)^{-1}c_r(u)$ is the quotient product of the unchanged suffix.
We compare the prediction with both this shifted target and the original
recipient target $c_r(u)$ on the same different-class pairs.

At positions 52, 56, and 71, the median shifted-target agreements across
the 40 endpoints are .0090, .0006, and .0005, respectively; median
original-target agreements are .9678, .9960, and .9982. At position 71,
shifted agreement ranges from zero to .1612, with the largest value in
the low-accuracy $Q_8$ endpoint. Thus the single-position replacement
usually does not produce a sustained change to subsequent quotient
predictions. This test has a specific limit. Even at the intervention
position, donor agreement from the penultimate-block replacement ranges
from .286 to 1.000, so it does not uniformly install a donor state.
Nor does replacing one position rule out a representation distributed
across several positions. We therefore identify a causal contribution to
the current quotient readout, while leaving its sequential implementation
unresolved.

\end{document}